\documentclass{article}

\usepackage[numbers,sort&compress]{natbib}
\usepackage[utf8]{inputenc}
\usepackage[T1]{fontenc}
\usepackage{url}
\usepackage{booktabs}
\usepackage{amsfonts, amssymb, amsthm, amsmath}
\usepackage{nicefrac}
\usepackage{microtype}
\usepackage{graphicx,psfrag,subfigure}
\usepackage{animate}

\usepackage{algorithm}
\usepackage{algorithmic}

\usepackage{microtype}

\usepackage[accepted]{icml2019}

\newcommand{\X}{\ensuremath{\mathbf{X}}}
\newcommand{\I}{\ensuremath{\mathbf{I}}}
\newcommand{\0}{\ensuremath{\mathbf{0}}}
\newcommand{\HH}{\ensuremath{\mathbf{H}}}
\newcommand{\x}{\ensuremath{\mathbf{x}}}
\newcommand{\y}{\ensuremath{\mathbf{y}}}
\newcommand{\z}{\ensuremath{\mathbf{z}}}
\newcommand{\Y}{\ensuremath{\mathbf{Y}}}
\newcommand{\Z}{\ensuremath{\mathbf{Z}}}
\newcommand{\btheta}{\ensuremath{\boldsymbol{\theta}}}
\newcommand{\bxi}{\ensuremath{\boldsymbol{\xi}}}
\newcommand{\f}{\ensuremath{\mathbf{f}}}

\newcommand{\bbR}{\ensuremath{\mathbb{R}}}
\newcommand{\norm}[1]{\left\lVert#1\right\rVert}

\newcommand{\abs}[1]{\left\lvert#1\right\rvert}
\graphicspath{{figs/}}
\newtheorem{theorem}{Theorem}
\DeclareMathOperator*{\argmin}{argmin}

\icmltitlerunning{BCD Proximal Method for Simultaneous Filtering and Parameter Estimation}

\begin{document}

\twocolumn[
\icmltitle{A Block Coordinate Descent Proximal Method \\for Simultaneous Filtering and Parameter Estimation}



\icmlsetsymbol{equal}{*}

\begin{icmlauthorlist}
\icmlauthor{Ramin Raziperchikolaei}{ra,ucmcs}
\icmlauthor{Harish~S. Bhat}{uu,ucmam}
\end{icmlauthorlist}

\icmlaffiliation{ra}{Rakuten Institute of Technology, San Mateo, CA, USA}
\icmlaffiliation{ucmcs}{Department of Computer Science, University of California, Merced, USA}
\icmlaffiliation{uu}{Department of Mathematics, University of Utah, USA}
\icmlaffiliation{ucmam}{Department of Applied Mathematics, University of California, Merced, USA}

\icmlcorrespondingauthor{Ramin Raziperchikolaei}{ramin.raziperchikola@rakuten.com}
\icmlcorrespondingauthor{Harish~S. Bhat}{hbhat@ucmerced.edu}

\icmlkeywords{Machine Learning, ICML}

\vskip 0.3in
]



\printAffiliationsAndNotice{}  

\begin{abstract}
We propose and analyze a block coordinate descent proximal algorithm (BCD-prox) for simultaneous filtering and parameter estimation of ODE models.  As we show on ODE systems with up to $d=40$ dimensions, as compared to state-of-the-art methods, BCD-prox exhibits increased robustness (to noise, parameter initialization, and hyperparameters), decreased training times, and improved accuracy of both filtered states and estimated parameters.  We show how BCD-prox can be used with multistep numerical discretizations, and we establish convergence of BCD-prox under hypotheses that include real systems of interest.
\end{abstract}

\section{Introduction}
\label{s:intro}
Though ordinary differential equations (ODE) are used extensively in science and engineering, the task of learning ODE states and parameters from data still presents challenges.  This is especially true for nonlinear ODE that do not have analytical solutions.  For such problems, several published and widely used methods---including Bayesian, spline-based, and extended Kalman filter methods---work well with data with a high signal-to-noise ratio.  As the magnitude of noise increases, these methods break down, leading to unreliable estimates of states and parameters.  Problem domains such as biology commonly feature both nonlinear ODE models and highly noisy observations, motivating the present work.

Motivated by recent advances in alternating minimization \cite{ChatterjiB17, LiLR16, YiCS14}, block coordinate descent (BCD) \cite{XuYin2013, ZhangB17}, and proximal methods \cite{ParikhBoyd2014, SunLXB15}, we study a BCD proximal algorithm (BCD-prox) to solve the simultaneous filtering and parameter estimation problem.  Here filtering means recovering clean ODE states from noisy observations.  BCD-prox works by minimizing a unified objective function that directly measures how well the states and parameters satisfy the ODE system, in contrast to other methods that use separate objectives.  BCD-prox learns the states directly in the original space, instead of learning them indirectly by fitting a smoothed function to the observations.  Under hypotheses that include systems of real interest, BCD-prox is provably convergent.  In comparison with other methods, BCD-prox is more robust with respect to noise, parameter initialization, and hyperparameters.  BCD-prox is also easy to implement and runs quickly.

There have been several different approaches to the filtering and estimation problem. Nonlinear least squares methods start with an initial guess for the parameters that is iteratively updated to bring the model's predictions close to measurements \cite{Himmelbl67,Bard73,Benson79,Hosten79}. These methods diverge when the initial parameters are far from the true parameters.

Of more recent interest are spline-based methods, in which filtered, clean states are computed via (cubic) splines fit to noisy data.  As splines are differentiable, parameter estimation then reduces to a regression problem \cite{Varah82,Poyton06,Ramsay07,Cao11,Cao08}. Estimators other than splines, such as smoothing kernels and local polynomials, are also used \cite{Liang08,Dattner15,Gugushvi12}. These methods are sensitive to numerous hyperparameters (such as smoothing parameters and the numbers/positions of knots), to parameter initialization, and to the magnitude/type of noise that contaminates the data.

Bayesian approaches \cite{Girolami08,Dondelin13,Calderhe09,Gorbach17} must set hyperparameters (prior distributions, variances, kernel widths, etc.) very carefully to produce reasonable results.  Bayesian methods also feature large training times.  Another disadvantage of these methods, mentioned by \citet{Gorbach17}, is that they cannot simultaneously learn clean states and parameters. \citet{Gorbach17} uses a variational inference approach to overcome this problem, but the method is not applicable to all ODE.

BCD-prox learns parameters and states jointly, but it does not fit a smooth function to the observations.  Via this approach, BCD-prox reduces the number of hyperparameters to one.  BCD-prox avoids assumptions (i.e., spline or other smooth estimator) regarding the shape of the filtered states.  Furthermore, both the BCD and proximal components of the algorithm enable it to step slowly away from a poor initial choice of parameters.   In this way, BCD-prox remedies the problems of other methods.

Many other well-known nonlinear ODE filtering methods, including extended and ensemble Kalman filters as well as particle filters, are online methods that make Gaussian assumptions.  In contrast, BCD-prox is a distribution-free, batch method.

The rest of the paper is organized as follows. In Section \ref{s:prob} we define both the problem and the BCD-prox algorithm.  In Section \ref{s:conceptual}, we compare BCD-prox at a conceptual level against a competing method from the literature.  We discuss the convergence of BCD-prox in Section \ref{s:convergence}.  We show the advantages of BCD-prox with several experiments in Section \ref{s:exp}.  Further experiments and details are given in the supplementary material.

\section{Problem and Proposed Solution}
\label{s:prob}

Consider a dynamical system in $\bbR^d$, depending on a parameter $\btheta \in \bbR^p$, with state $\x(t)$ at time $t$:
\begin{equation}
	\label{eq:ode}
	\dot\x(t) = \frac{d\x(t)}{dt} = \f(\x(t),\btheta).
\end{equation}
At $T$ distinct times $\{t_i\}_{i=1}^T$, we have noisy observations $\y(t_i) \in \bbR^d$:
\begin{equation}
	\label{eq:noise}
	\y(t_i) = \x(t_i) + \z(t_i), \quad i=1,\dots,T
\end{equation}
where $\z(t_i) \in \bbR^d$ is the noise of the observation at time $t_i$. We represent the set of $T$ $d$-dimensional states, noises, and observations by $\X, \Z$, and $\Y \in \bbR^{d \times T}$, respectively. For concision, in what follows, we write the time $t_i$ as a subscript, i.e., $\x_{(t_i)}$ instead of $\x(t_i)$.

In this paper, we assume that the form of the vector field $\f(\cdot)$ is known. The simultaneous parameter estimation and filtering problem is to use $\Y$ to estimate $\btheta$ and $\X$. Examples of $\f(\cdot)$ and $\btheta$ can be found in Section \ref{s:exp}.

For ease of exposition, we first describe a BCD-prox algorithm based on the explicit Euler discretization of \eqref{eq:ode}.  Later, we will describe how to incorporate higher-order multistep methods into BCD-prox.  The explicit Euler method discretizes the ODE \eqref{eq:ode} for the $T$ time points as follows:
\begin{equation}
	\label{eq:ode_em}
	\x_{(t_{i+1})} - \x_{(t_i)}  = \f(\x_{(t_i)},\btheta)\Delta_{i}, \quad i=1,...,T-1
\end{equation}
where $\Delta_i = t_{i+1} - t_i$. In \eqref{eq:ode_em}, both states $\X$ and parameters $\btheta$ are unknown; we are given only the noisy observations $\Y$.  With this discretization, let us define
\begin{equation}
\label{eq:objx}
	E(\X,\btheta) = { \sum_{i=1}^{T-1}{ {\norm{\x_{(t_{i+1})} - \x_{(t_{i})}  - \f(\x_{(t_i)},\btheta)\Delta_{i}}}^2} }
\end{equation}
Note that $E$ measures the time-discretized mismatch between the left- and right-hand sides of \eqref{eq:ode}.  We refer to $E$ as fidelity, the degree to which the estimated states $\X$ and parameters $\btheta$ actually satisfy the ODE.  Let us now envision a sequence of iterates $\{ \X^{\ast(n)}, \btheta^{\ast(n)} \}_{n \geq 0}$.  For $n \geq 1$, we define the Euler BCD-prox objective function:
\begin{equation}
\label{eq:actualeulerobj}
F_n^\text{Euler}(\X, \btheta) = E(\X, \btheta) + \lambda \norm{\X - \X^{*(n-1)}}^2.
\end{equation}
\emph{We can now succinctly describe the Euler version of BCD-prox as block coordinate descent (first on $\btheta$, then on $\X$) applied to \eqref{eq:actualeulerobj}, initialized with the noisy data via $\X^{\ast(0)} = \Y$, and repeated iteratively until convergence criteria are met.}. 

The Euler method is a first-order method.  To understand this, let $\Delta_i = h$ (independent of $i$) and $t_N = N h$. Then the global error between the numerical and true solution of the ODE \eqref{eq:ode}, $\norm{ \x^\text{numerical}(t_N) - \x^\text{true}(t_N) }$, is $O(h)$.  If we seek a more accurate discretization, we can apply a multistep method.  The idea of multistep ($m$-step) methods is to use the previous $m$ states to predict the next state, yielding a method with $O(h^m)$ global error.  Let us consider the general formulation of the explicit linear $m$-step method to discretize \eqref{eq:ode}:
\begin{equation}
	\label{eq:mstep}
	\x_{(t_{i+1})} = \sum_{j=0}^{m-1}a_j \x_{(t_{i-j})} + \Delta_i \sum_{j=0}^{m-1}b_j \f(\x_{(i-j)},\btheta),
\end{equation}
where $\Delta_i$ is the time step. There are several strategies to determine the coefficients $\{a_j\}_{j=0}^{m-1}$ and $\{b_j\}_{j=0}^{m-1}$. For example, over the interval $\Delta_i$, the Adams-Bashforth method approximates $\f(\cdot)$ with a polynomial of order $m$; this leads to a method with $O(h^m)$ global error.  When $m=1$, this method reduces to the explicit Euler method considered above.  For further information on multistep methods, consult \citet{Iserles2009, Palais09}.

Note that to use $m$-step methods to predict the state at time $i$, we need its previous $m$ states. To predict the states $\{\x_i\}_{i=2}^{m}$ (the first few states), the maximum order we can use is $i-1$, because there are only $i-1$ states before the state $\x_i$. In general, to predict $\x_i$ we use a multistep method of order $\min(i-1,m)$.

When using a general $m$-step discretization method, we define our objective function as follows:
\begin{multline}
	\label{eq:mstepobj}
	E_{\text{m-step}}(\X,\btheta) = \sum_{i=1}^{T-1} \biggl\| \x_{(t_{i+1})} - \textstyle \sum_{j=0}^{k-1}a_j \x_{(t_{i-j})} \\ - \Delta_i \textstyle \sum_{j=0}^{k-1}b_j \f(\x_{(i-j)},\btheta) \biggr\| ^2 ,
\end{multline}
where $k = \min(i-1,m)$ is the order of the discretization method to predict the state $\x_i$.  We can then reformulate the BCD-prox objective as
\begin{equation}
\label{eq:actualobj}
F_n(\X, \btheta) = E_{\text{m-step}}(\X, \btheta) + \lambda \norm{\X - \X^{*(n-1)}}^2.
\end{equation}
We now regard \eqref{eq:actualeulerobj} as a special case of \eqref{eq:actualobj} for $m=1$, i.e., in the case where the $m$-step method reduces to Euler.  With these definitions, \emph{BCD-prox is block coordinate descent (first on $\btheta$, then on $\X$) applied to $F_n(\X, \btheta)$, initialized with the noisy data via $\X^{\ast(0)} = \Y$, and repeated iteratively until convergence criteria are met.}.  We detail this algorithm in Alg. \ref{alg}.
\begin{algorithm}[t]
	\caption{Pseudo-code of our proposed method} \label{alg}
	\textbf{Input:} Noisy observations $\Y=[\y_{(t_1)},\dots,\y_{(t_T)}] \in \bbR^{d \times T}$, time differences $\{\Delta_i = t_{i+1} - t_i\}_{i=1}^{T-1}$, form of $\f(\cdot)$ in Eq. \eqref{eq:ode}, hyperparameter $\lambda$, initial guess $\btheta^{*(0)}$, and order $m$ of the $m$-step method.
	\begin{algorithmic}[1]
		\vspace{1ex}
		\STATE $\X^{*(0)}=\Y$
		\STATE $n = 0$
		\REPEAT
		\STATE $n = n + 1$
		\STATE \textbullet\ Compute $\btheta^{*(n)} = \argmin_{ \btheta } F_n(\X^{*(n-1)}, \btheta)$.	
		\STATE \textbullet\ Compute $\X^{*(n)} = \argmin_{\X} F_n(\X, \btheta^{*(n)})$.
		\UNTIL{convergence}
		\STATE Compute predicted states $\hat{\X}$ by repeatedly applying Eq. \eqref{eq:mstep}, where $\btheta = \btheta^{*(n)}$ and $\x_{(t_1)} = \x^{*(n)}_{(t_1)}$.
		\STATE \textbf{return} $\btheta^{*(n)}$ and $\hat{\X}$ as the estimated parameters and predicted states.
	\end{algorithmic}
\end{algorithm}

\section{Conceptual Comparison with iPDA}
\label{s:conceptual}

Though BCD-prox may seem straightforward, we cannot find prior work that utilizes precisely this approach.  Since of the closest relatives is the successful iPDA (iterated principal differential analysis) method \cite{Poyton06,Ramsay07}, we explain iPDA and offer a conceptual comparison between iPDA and BCD-prox.  In iPDA, the parameter estimation error is defined as
\begin{equation}
\label{eq:ipdamismatch}
E_\text{cont}(\x_{(t)}, \btheta) = \int \norm{\frac{d\x_{(t)}}{dt} - \f(\x_{(t)},\btheta)}^2 dt,
\end{equation}
which we can regard as the continuous-time ($\Delta_i \to 0$) limit of either \eqref{eq:objx} or \eqref{eq:mstepobj}, our mismatch/fidelity terms.  The iPDA objective function is then 
\begin{equation}
\label{eq:ipdaobj}
J(\x_{(t)}, \btheta) = E_\text{cont}(\x_{(t)}, \btheta) + \lambda \norm{ \X - \Y}^2,
\end{equation}
the sum of the parameter estimation error with a regularization term.  Initialized with $\btheta^{(0)}$, the iPDA method  proceeds by iterating over the following two minimization steps:
\begin{enumerate}
	\item Set $\x_{(t)}^{(n)} = \argmin_{\x_{(t)}} J(\x_{(t)}, \btheta^{(n-1)})$. In this step, $\x_{(t)}$ is constrained to be a smooth spline.
	\item Set $\btheta^{(n)} = \argmin_{\btheta} J(\x_{(t)}^{(n)}, \btheta)$.  Note that the optimization only includes the parameter estimation term since the regularization term does not depend on $\btheta$.
\end{enumerate}
The main issue with the objective in \eqref{eq:ipdaobj} is the regularization term. This term determines how far the clean states are going to be from the noisy observations. If we set $\lambda$ to a large value, then $\x_{(t)}$ remains close to the data $\y_{(t)}$, potentially causing a large parameter estimation error. If we set $\lambda$ to a small value, then $\x_{(t)}$ might wander far from the observed data. It is a challenging task to set $\lambda$ to the right value for two reasons: 1) the optimal $\lambda$ depends on both the noise $\Z$ and the vector field $\mathbf{f}$, and 2) in a real problem, we do not have access to the clean states $\X$ (all we have are the noisy observations $\Y$), so we cannot find the right $\lambda$ by cross-validation. We return to this point below.

Before continuing, it is worth pointing out a crucial fact regarding all the mismatch/fidelity objectives $E$ that we have seen thus far.
\begin{theorem}
	\label{th:infsol}
	The objective functions $E$ defined in \eqref{eq:objx}, \eqref{eq:mstepobj}, and \eqref{eq:ipdamismatch} all have an infinite number of zeros, i.e., an infinite number of global minima that result in $E=0$.
\end{theorem}
\begin{proof}
	Assign arbitrary real vectors to $\btheta$ and the initial condition $\x_{(t_1)}$. Note that \eqref{eq:objx} is the special case of \eqref{eq:mstepobj} for $m=1$ so we need only discuss \eqref{eq:mstepobj}. Starting from $\x_{(t_1)}$, step forward in time via \eqref{eq:mstep}.  By computing the states $\x_{(t_2)}, \ldots, \x_{(t_T)}$ in this way, we ensure that each term in \eqref{eq:mstepobj} vanishes.
	For the continuous $E$ function \eqref{eq:ipdamismatch}, we use the existence/uniqueness theorem for ODE to posit a unique solution $\x_{(t)}$ passing through $\x_{(t_1)}$ at time $t=t_1$.  By definition of a solution of an ODE, this will ensure that \eqref{eq:ipdamismatch} vanishes.
	Because, in all cases, $E \geq 0$, we achieve a global minimum.  Because $\btheta$ and $\x_{(t_1)}$ are arbitrary, an infinite number of such minima exist.
\end{proof}
\emph{Note that BCD-prox always produces a (state,parameter) pair that results in $E=0$ for \eqref{eq:mstepobj}.}  In fact, step 8 of Alg. \ref{alg} uses the idea from the proof of Theorem \eqref{th:infsol} to generate a sequence of predicted states $\hat{\X}$ such that $E(\hat{\X}, \btheta^{*(n)}) = 0$.

Let us reconsider Step 6 in Alg. \ref{alg}:
\begin{equation}
\label{eq:xstep}
\X^{*(n)} = \argmin_{\X} F_n(\X, \btheta^{*(n)}).
\end{equation}
By definition of $F_n$ and using the notion of proximal operators \citep{ParikhBoyd2014}, we can write
$$
\X^{*(n)} = \operatorname{prox}_{(2 \lambda)^{-1} E}( \X^{*(n-1)} ),
$$
with the understanding here and in what follows that $\btheta$ is fixed at $\btheta^{*(n)}$.  In general for $\Delta_i > 0$ and arbitrary $\mathbf{f}$, $E$ defined in \eqref{eq:mstepobj} will not be convex.  In this case, we view the proximal operator above as a set-valued operator as in \citep{LiICML17}; any element of the set will do.  Conceptually, this proximal step approximates a gradient descent step:
\begin{multline}
\label{eq:proxconcept}
\X^{*(n)} = \X^{*(n-1)}  \\ - (2 \lambda)^{-1} \nabla_{\X} E(\X^{*(n-1)}, \btheta^{*(n)}) + o((2 \lambda)^{-1}).
\end{multline}
It is now clear that $\lambda$ plays the role of an inverse step size---our experiments later will confirm that there is little harm in choosing $\lambda$ too large.  With this in mind, we can now contrast BCD-prox with iPDA.  In BCD-prox, we use the data $\Y$ to initialize the algorithm; subsequently, the algorithm may take many proximal steps of the form \eqref{eq:proxconcept} to reach a desired optimum.  If the data $\Y$ is heavily contaminated with noise, it may be wise to move far away from $\Y$ as we iterate.

In contrast, iPDA's regularization term is $\lambda \| \X - \Y \|^2$.  Roughly speaking, iPDA searches for $\X$ in a neighborhood of $\Y$; the diameter of this neighborhood is inversely related to $\lambda$.  When the magnitude of the noise $\Z$ is small, searching for $\X$ in a small neighborhood of $\Y$ is reasonable.  For real data problems in which the magnitude of $\Z$ is unknown, however, choosing $\lambda$ \emph{a priori} becomes difficult.

An additional important difference between BCD-prox and iPDA has to do with convexity, which we discuss next.

\section{Convergence}
\label{s:convergence}
In practice, we implement the $\argmin$ steps in Alg.\ref{alg} using the LBFGS algorithm, implemented in Python via scipy.optimize.minimize.  Throughout this work, when using LBFGS, we use automatic differentiation to supply the optimizer with gradients of the objective function.  We stop Alg.\ref{alg} when the error $E$ changes less than $10^{-8}$ from one iteration to the next.

To see when this happens, we take another look at the optimization over the states $\X$ in \eqref{eq:xstep} at iteration $n$. This objective function $F_n$ has two parts. The optimal solution of the first part ($E$) is the predicted states $\hat{\X}^{(n)}$. The optimal solution of the proximal part is $\X^{*(n-1)}$. When we optimize this objective function to find $\X^{*(n)}$, there are three cases: 1) $\X^{*(n)} = \hat{\X}^{(n)}$, 2) $\X^{*(n)} = \X^{*(n-1)}$, and 3) $\X^{*(n)}$ is neither $\hat{\X}^{(n)}$ nor $\X^{*(n-1)}$. Our algorithm stops when we are in case 1 or 2 since further optimization over $\btheta$ and $\X$ changes nothing. In case 3, the algorithm continues, leading to further optimization steps to decrease error.

Indeed, let us note that steps 5 and 6 in Alg.~\ref{alg} together imply
\begin{equation}
	\label{eq:converge}
	 E(\X^{*(n)},\btheta^{*(n)}) \le  E(\X^{*(n-1)},\btheta^{*(n-1)}).
\end{equation}
The function $E$, bounded below by $0$, is non-increasing along the trajectory $\{ (\X^{*(n)},\btheta^{*(n)}) \}_{n \geq 1}$.  Hence $\{ E(\X^{*(n)},\btheta^{*(n)}) \}_{n \geq 1}$ must converge to some  $E^\ast \geq 0$.

Next we offer convergence theory for the Euler version of BCD-prox.  We believe this theory can also be established for the general $m$-step version of BCD-prox; however, the calculations will be lengthier.  In this subsection, we let $\x_i = \x_{(t_i)} \in \mathbb{R}^d$.  For $T$ even, set
$$
\x^+ = \{ \x_1, \x_2, \ldots, \x_{T/2} \}, \quad 
\x^- = \{ \x_{T/2+1}, \ldots, \x_{T} \}.
$$
For $T$ odd, replace $T/2$ by $(T-1)/2$ in the above definitions.  In words, $\x^+$ is the first half of the state series while $\x^-$ is the second half of the state series.  Note that $\X = (\x^+, \x^-)$.

Assume that $\mathbf{f}$ is at most linear in $\btheta$, so that
$\mathbf{f}(\x, \btheta) = \mathbf{f}_0(\x) + \mathbf{f}_1(\x) \btheta$,
with $\mathbf{f}_1 : \mathbb{R}^d \to \mathbb{R}^{d \times p}$ assumed to have full column rank for all $\x$.
  
Now initialize $\X^0 = \Y$ and proceed sequentially with the following steps for $n \geq 1$:
\begin{subequations}
\label{eqn:bcdproxnew}
\begin{align}
\label{eqn:thetaupdate}
\btheta^n &= \argmin_{\btheta} F_n(\X^{n-1}, \btheta) = \argmin_{\btheta} E(\X^{n-1}, \btheta) \\
\label{eqn:xminusupdate}
(\x^-)^n &= \argmin_{\x^-} F_n((\x^+)^{n-1}, \x^-, \btheta^n) \\
\label{eqn:xplusupdate}
(\x^+)^n &= \argmin_{\x^+} F_n(\x^+, (\x^-)^{n}, \btheta^n) \\
\X^n &= ((\x^+)^n, (\x^-)^n) 
\end{align}
\end{subequations}
We now seek to apply the results of \citet{XuYin2013}.  In order to do so, we will establish strong convexity of each of the steps in (\ref{eqn:bcdproxnew}).  We begin by noting that
$$
E(\X, \btheta) = \sum_{i=1}^{T-1} \| \x_{i+1} - \x_i - \mathbf{f}_0(\x_i) \Delta_i + \mathbf{f}_1(\x_i) \btheta \Delta_i \|^2.
$$
We compute the $p \times p$ Hessian
$$
\nabla_{\btheta} \nabla_{\btheta} E = 2 \sum_{i=1}^{T-1}  (\mathbf{f}_1(\x_i))^T \mathbf{f}_1(\x_i) \Delta_i^2.
$$
Since $\f_1$ has full column rank, it follows that $E(\X, \btheta)$ is strongly convex in $\btheta$ with $\X$ held fixed.

Next, suppose all $\Delta_i$ are zero.  Then \eqref{eq:objx} reduces to
$E(\X, \btheta; \Delta_i=0) = \sum_{i=1}^{T-1} \| \x_{i+1} - \x_i \|^2$.
This is a quadratic form written as a sum of squares; hence it is positive semidefinite.  We sharpen this to positive definiteness by examining derivatives.  First we hold $\x^-$ and $\btheta$ fixed and consider $A = \nabla_{\x^+} \nabla_{\x^+} E(\X, \btheta; \Delta_i=0)$, the Hessian with respect to $\x^+$ only.  We obtain
$$
A = \begin{bmatrix}
2I & -2I & & & \\
-2I & 4I & -2I & &\\
& -2I & 4I & \ddots & \\
& & \ddots & \ddots & -2I \\
& & & -2I & 4I 
\end{bmatrix}
$$
Here each $I$ is a $d \times d$ identity block.  The positive semidefiniteness established above implies that all eigenvalues of $A$ are nonnegative.  By an induction argument, we can show that $\det A = 2^{d T/2}$, implying that the eigenvalues of $A$ are bounded away from zero.  Hence the quadratic form $E(\X, \btheta; \Delta_i=0)$
restricted to $\mathbf{x}^+$ (with $\mathbf{x}^-$ held fixed) is strongly convex.  In an analogous way, we can show that 
$E(\X, \btheta; \Delta_i=0)$
restricted to $\mathbf{x}^-$ (with $\mathbf{x}^+$ held fixed) is strongly convex.  Both of these properties hold at $\Delta_i = 0$.  Because the eigenvalues of both restrictions are continuous functions of $\Delta_i$, there exists $\delta > 0$ such that for $\Delta_i \in (0, \delta)$, the eigenvalues remain bounded away from zero.

Then we have the following first convergence result.
\begin{theorem}
	\label{th:conv1}
Suppose all $\Delta_i \in (0, \delta)$ for the $\delta$ established above.  Suppose $\mathbf{f}$ is linear in $\btheta$ with the full-rank condition described above.  Then there exists an interval of $\lambda$ values for which the algorithm \eqref{eqn:bcdproxnew} converges to a Nash equilibrium $(\overline{\X}, \overline{\btheta})$ of the objective $E$ defined in \eqref{eq:objx}.
\end{theorem}
\begin{proof}
The result follows directly from Theorem 2.3 from \citet{XuYin2013}; we have verified all hypotheses.  In particular, when all $\Delta_i \in (0, \delta)$, $E(\X, \btheta)$ is strongly convex in $\x^+$ (with $\x^-$ and $\btheta$ held fixed) and strongly convex in $\x^-$ (with $\x^+$ and $\btheta$ held fixed).
\end{proof}

Let us further assume that $\mathbf{f}$ satisfies the Kurdyka-Lojasiewicz (KL) property described in Section 2.2 of \citet{XuYin2013}.  In particular, if each component of $\mathbf{f}$ is real analytic, the KL property will be satisfied.  Together with linearity of $\mathbf{f}$ in $\btheta$, this includes numerous vector fields of interest, including all ODE in our experimental results.  (For FitzHugh--Nagumo, a change of variables renders the system linear in the parameters.) Then we have a second convergence result.
\begin{theorem}
 \label{th:conv2}
Suppose in addition to the hypotheses of Theorem \ref{th:conv1}, $\mathbf{f}$ is smooth and satisfies the KL property.  Then assuming the algorithm defined by \eqref{eqn:bcdproxnew} begins sufficiently close to a global minimizer, it will converge to a global minimizer of $E$ defined in \eqref{eq:objx}.
\end{theorem}
\begin{proof}
The result follows directly from Corollary 2.7 and Theorem 2.8 of \citet{XuYin2013}; we have  verified all hypotheses.\end{proof}

\section{Experiments}
\label{s:exp}

We briefly explain the datasets (models) that we used in our experiments here. In the supplementary material, we detail the ODEs and true parameter values for 1) Lotka--Volterra with two-dimensional states and four unknown parameters. 2) FitzHugh--Nagumo with two-dimensional states and four unknown parameters. 3) R\" ossler attractor with three-dimensional states and three unknown parameters. 4) Lorenz-96 with $40$ nonlinear equations and one unknown parameter, the largest ODE we found in the literature.

We create the clean states using a Runge-Kutta method of order $5$. In all of our experiments, unless otherwise stated, we use the three-step Adams-Bashforth method to discretize the ODE. Also, unless otherwise stated, we added Gaussian noise with mean $0$ and variance $\sigma^2$ to each of the clean states to create the noisy observations.

\paragraph{Advantages of our approach.} Before detailing our experimental results, let us give an overview of our findings.  BCD-prox is robust with respect to its only hyperparameter $\lambda$. We will show below that for a broad range of $\lambda$ values, BCD-prox works well. We fix it to $\lambda=1$ in our later experiments. As explained before, previous methods have a large number of hyperparameters, which are difficult to set.

BCD-prox can be trained quickly. On a standard laptop, it takes around $20$ seconds for BCD-prox to learn the parameters and states jointly on ODE problems with $400$ states. The spline-based methods take a few minutes and Bayesian methods take a few hours to converge on the same problem.

Because BCD-prox, unlike Bayesian methods, does not make assumptions about the type of the noise or distribution of the states, it performs well under different noise and state distributions.  In particular, as the magnitude of noise in the observations increases, BCD-prox clearly outperforms the extended Kalman filter.

As our experiments confirm, both spline-based and Bayesian approaches are very sensitive to the initialization of the ODE parameters. If we initialize them far away from the true values, they do not converge. BCD-prox is much more robust. This robustness stems from simultaneously learning states and parameters. Even if the estimated and true parameters differ at some iteration, they can converge later, as the estimated states converge to the clean states.

\vspace{-2ex}
\paragraph{Evaluation metrics.} Let $\btheta$ and $\X$ denote the true parameters and the clean states, respectively. Let $\btheta^{*}$ and $\hat{\X}$ denote the estimated parameters and the predicted states. We report the Frobenius norm of $\X-\hat{\X}$ as the prediction error. We also consider $\abs{\theta_l - \theta^{*}_l}$ as the $l$th parameter error. To compute predicted states, we first take $\btheta^{*}$ as the parameter and $\x_{(t_1)}^*$ as the initial state; we then repeatedly apply either Euler \eqref{eq:ode_em} or multistep \eqref{eq:mstep} numerical integration.

\begin{figure}[!t] 
  \centering
  \begin{tabular}{@{}c@{\hspace{.5ex}}c@{\hspace{1ex}}c@{}}
  & R\" ossler & Lorenz-96 \\[-.8ex]
   \rotatebox{90}{\hspace{3ex} \small{pred.\ error}} &       
  \includegraphics*[width=0.425\linewidth]{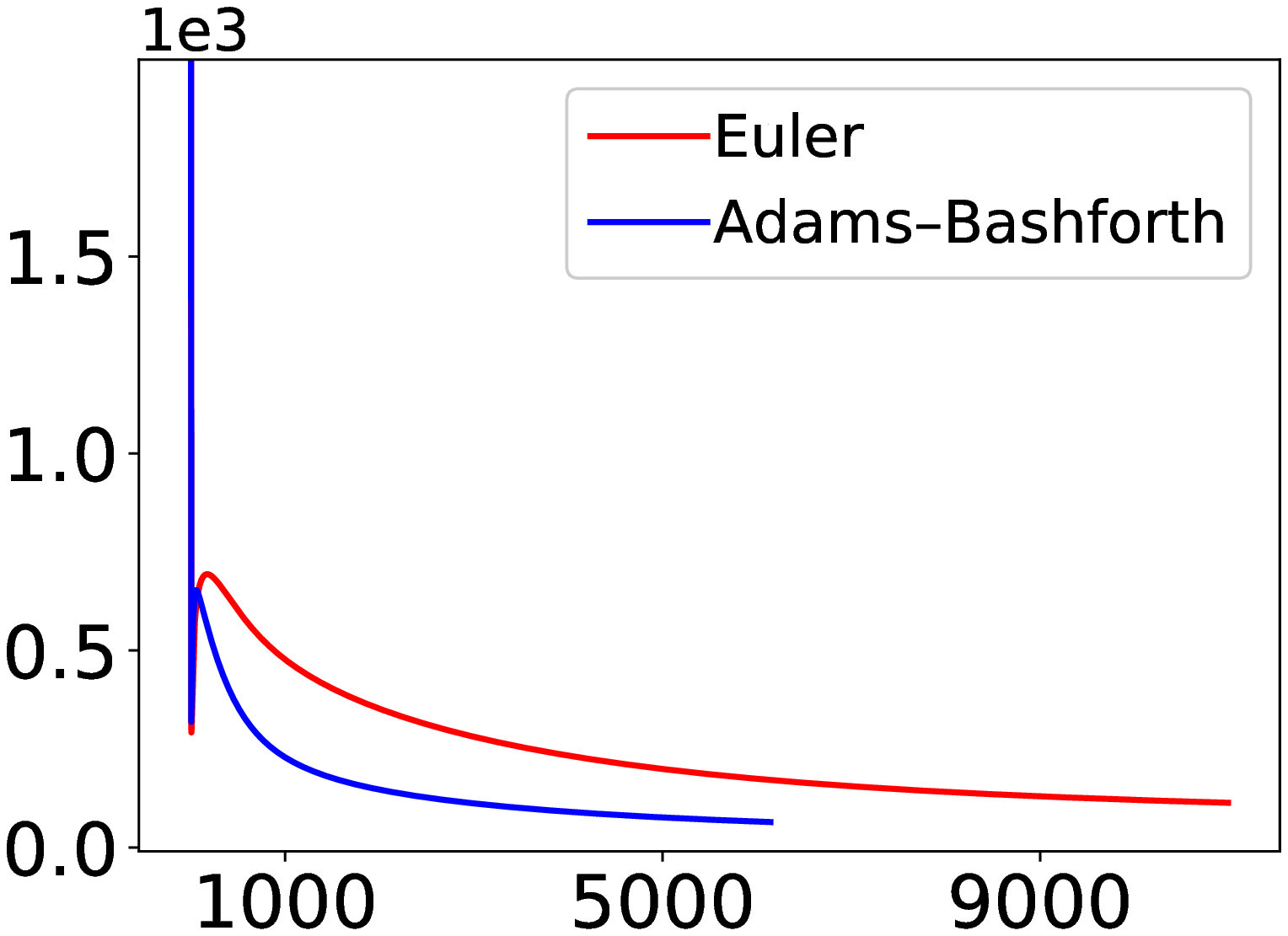}&
    \includegraphics*[width=0.43\linewidth]{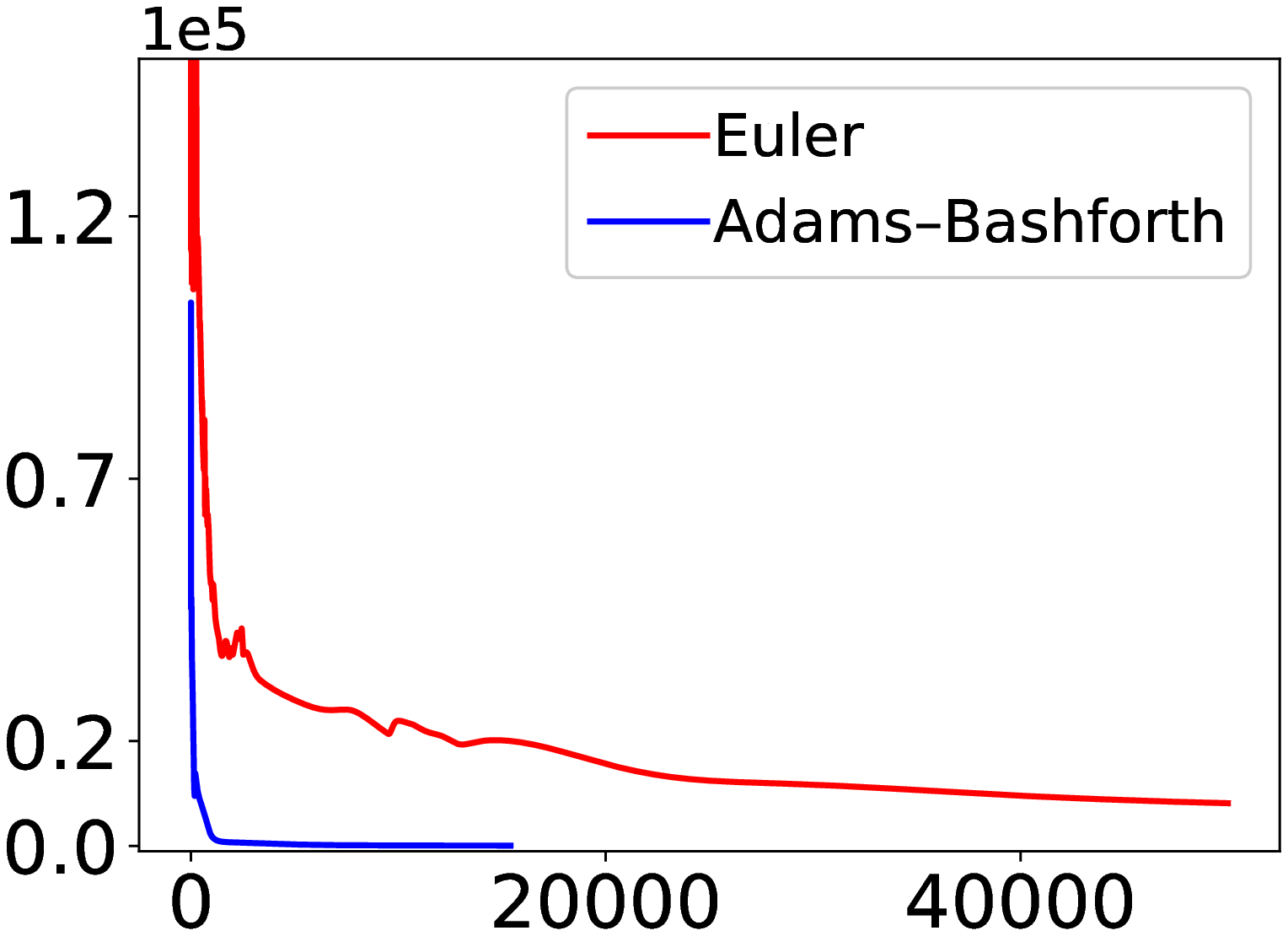}\\[-1ex]
    & \small{iteration} & \small{iteration}
  \end{tabular}
    \vspace{-2ex}
  \caption{Prediction error at different iterations of our algorithm with different discretization methods. The noise variance of observations is $\sigma^2=1$. Our learning strategy decreases the error significantly.}
  \vspace{-3ex}
  \label{f:objerr}
\end{figure}

\vspace{-1ex}
\paragraph{Optimization of objective \eqref{eq:objx} leads to better estimation.} At each iteration $n$ of our optimization, we compute the predicted states $\hat{\X}^{(n)}$ and report the prediction error. 
In Fig.~\ref{f:objerr}, we consider two kinds of discretization: 1) one-step Euler method, and 2) three-step Adams-Bashforth method. Note that as we increase the order, we expect to see more accurate results.

The variance of the noisy observations is $\sigma^2=1$. The supplementary material contains the results for the FitzHugh--Nagumo model and also for the case of $\sigma^2=0.5$. Fig.~\ref{f:objerr} shows that at the first iteration the error is significant in all models. The error is $\sim 10^3$ for FitzHugh--Nagumo and R\" ossler, and $\sim 1.5 \times 10^5$ for Lorenz-96 model. 

After several iterations of our algorithm, the error decreases significantly, no matter what kind of discretization we use. Three-step Adams-Bashforth performs better than Euler in general: it converges faster and achieves a smaller final error. This is especially clear for the Lorenz-96 model: the final error is near zero for three-step Adams-Bashforth, but near $10^4$ for Euler.

The last point about Fig.~\ref{f:objerr} is that, as expected, the prediction error increases at times; the error does not decrease monotonically. This mainly happens at the first few iterations. The main reason for this behavior is that our objective function in \eqref{eq:objx} is different from the prediction error. We cannot directly optimize the prediction error because we do not have access to the clean states. Still, the fact that our algorithm eventually brings the prediction error close to zero suggests that minimizing the objective in \eqref{eq:objx} has the same effect as minimizing the prediction error.

\begin{figure}[!t]
  \centering
  \begin{tabular}{@{}l@{}c@{}c@{}}
    & Lotka--Volterra & FitzHugh--Nagumo \\[-.1ex]
    \rotatebox{90}{\hspace{1ex} \small{pred.\ error} } &   
	\includegraphics*[width=0.46\linewidth,height=.278\linewidth]{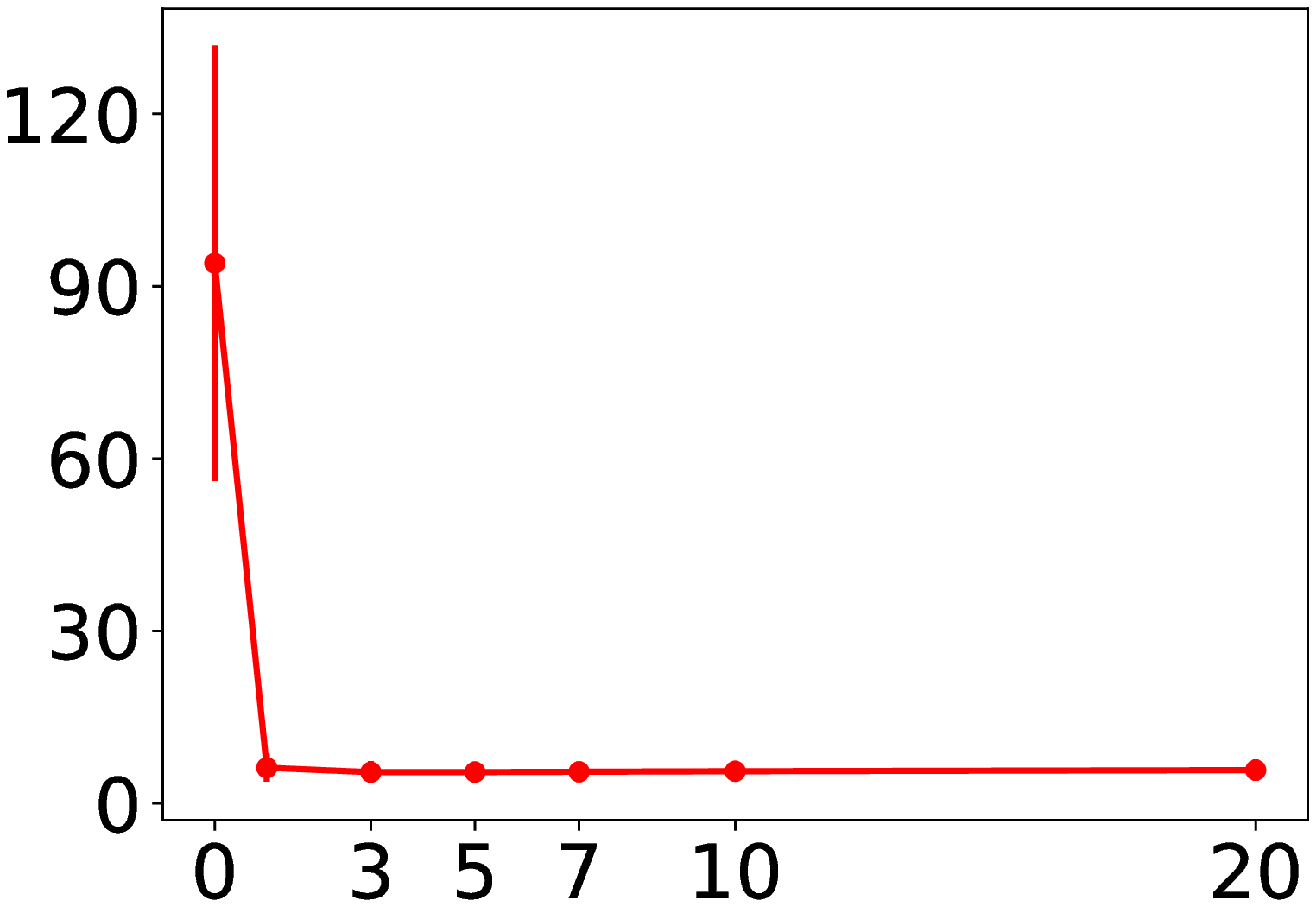}&
    \includegraphics*[width=0.46\linewidth,height=.278\linewidth]{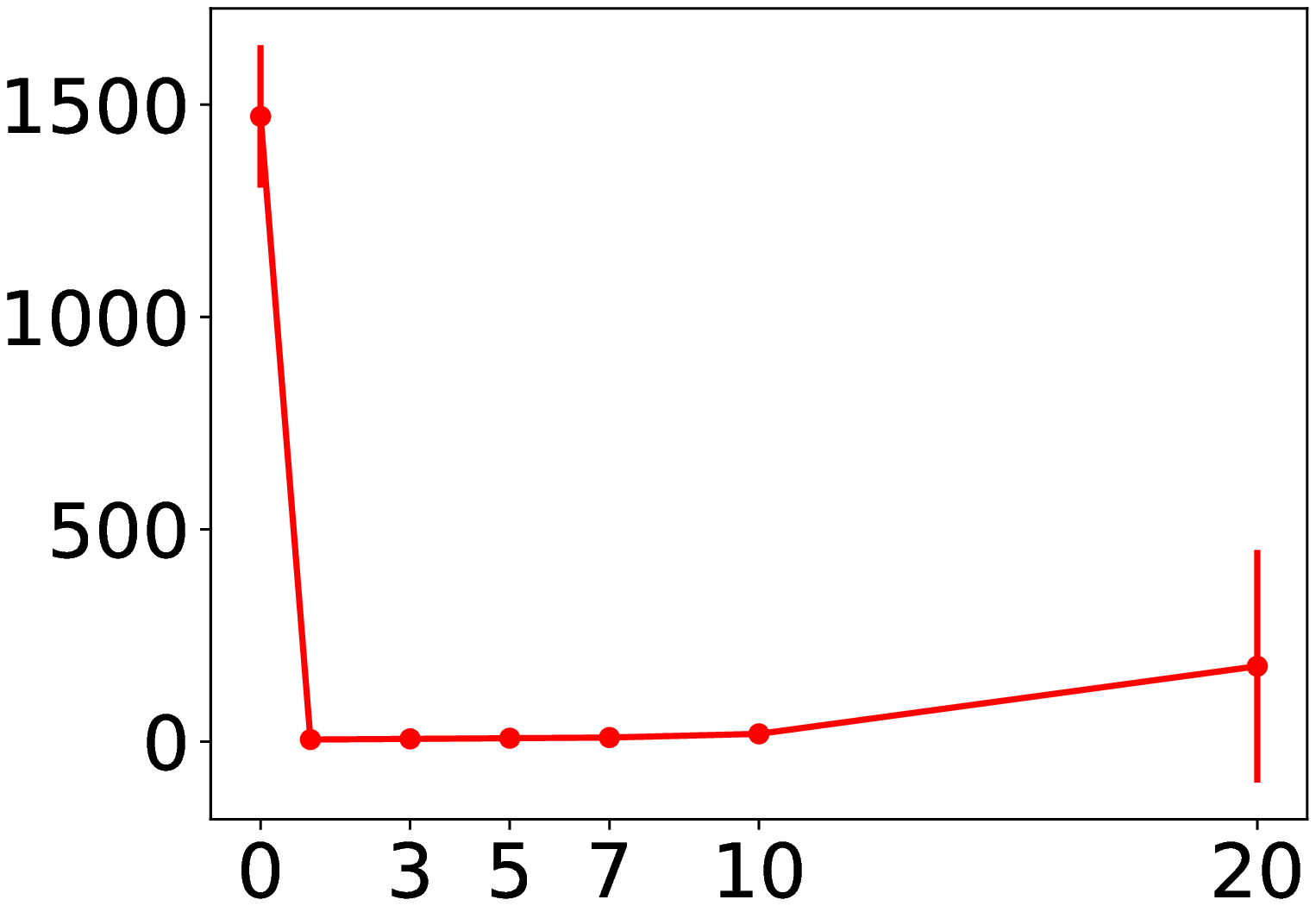}\\[-1ex]
     \rotatebox{90}{\hspace{1ex} \small{param.\ value} } &   
	\includegraphics*[width=0.46\linewidth,height=.278\linewidth]{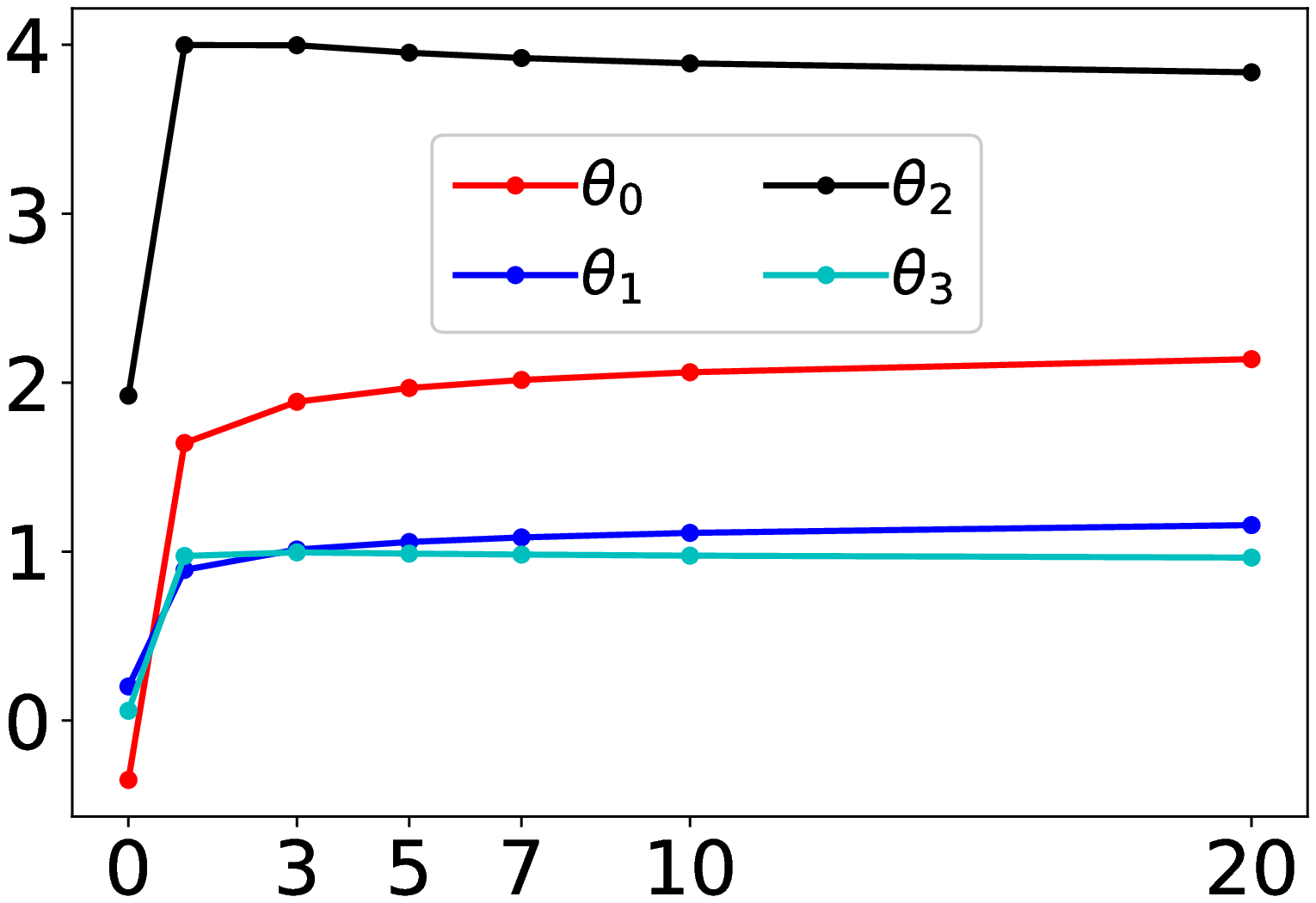}&
    \includegraphics[width=0.46\linewidth,height=.278\linewidth]{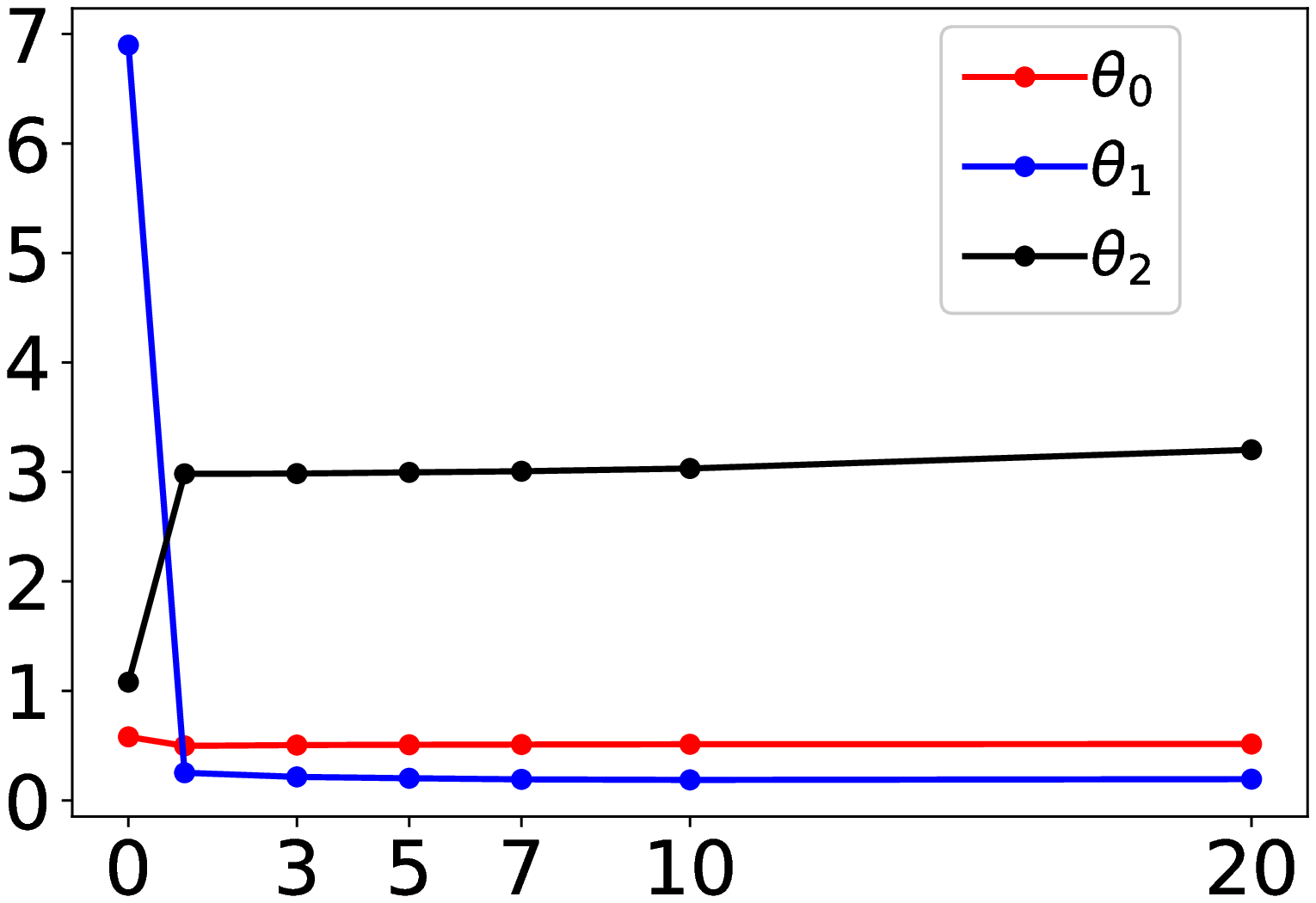}\\[-1.2ex]
   & \small{penalty weight $\lambda$} & \small{penalty weight $\lambda$}
  \end{tabular}
  \vspace{-2ex}
  \caption{Robustness to the hyperparameter $\lambda$. The true parameters are $\theta_0=.5, \theta_1=.3,$ and $\theta_2=3$ in the FitzHugh--Nagumo and $\theta_0=2, \theta_1=1, \theta_2=4$, and $\theta_3=1$ in the Lotka--Volterra. For each $\lambda$, we report the mean error and parameter value in $10$ experiments.}
  \vspace{-3ex}
  \label{f:robust}
\end{figure}

\begin{figure*}[!h]
  \centering
  \psfrag{i}[t][t]{it}
  \psfrag{idx}[t][t]{index of the hash fcn}
  \begin{tabular}{@{}c@{}c@{\hspace{.2ex}}c@{\hspace{.2ex}}c@{\hspace{.2ex}}c@{}}
  	& $\sigma_{\btheta}^2=1$ & $\sigma_{\btheta}^2=5$ & $\sigma_{\btheta}^2=10$ & $\sigma_{\btheta}^2=20$\\[-.1ex]
	  \hspace{3ex}\rotatebox{90}{\hspace{2ex}\raisebox{3ex}[0pt][0pt]{\makebox[0pt][c]		{\hspace{-30ex}\makebox[21.5em][c]{\dotfill \small{R\" ossler attractor}\dotfill}}}\small{state error}} &              
    \includegraphics*[width=0.235\linewidth,height=0.12\linewidth]{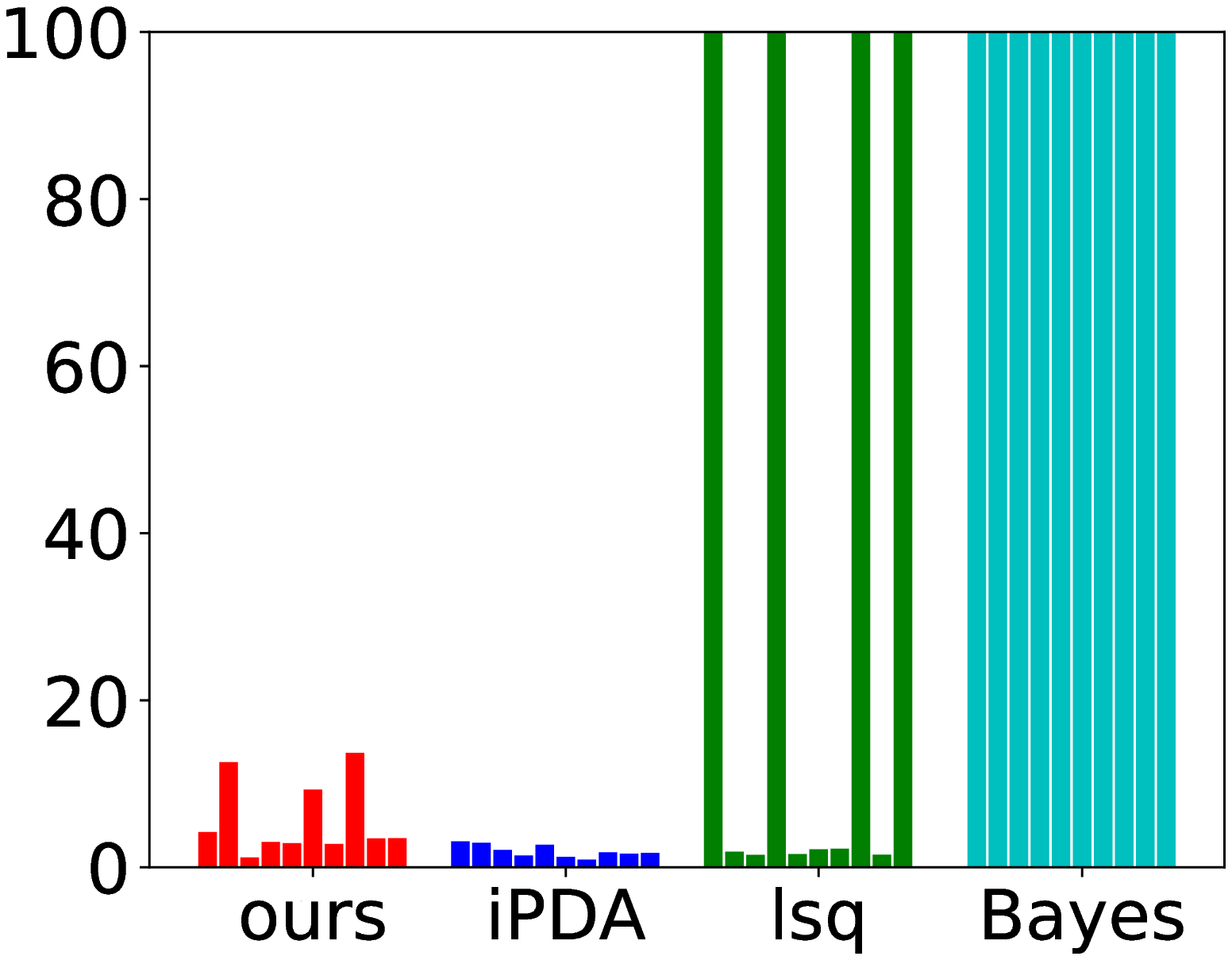}&
    \includegraphics*[width=0.235\linewidth,height=0.12\linewidth]{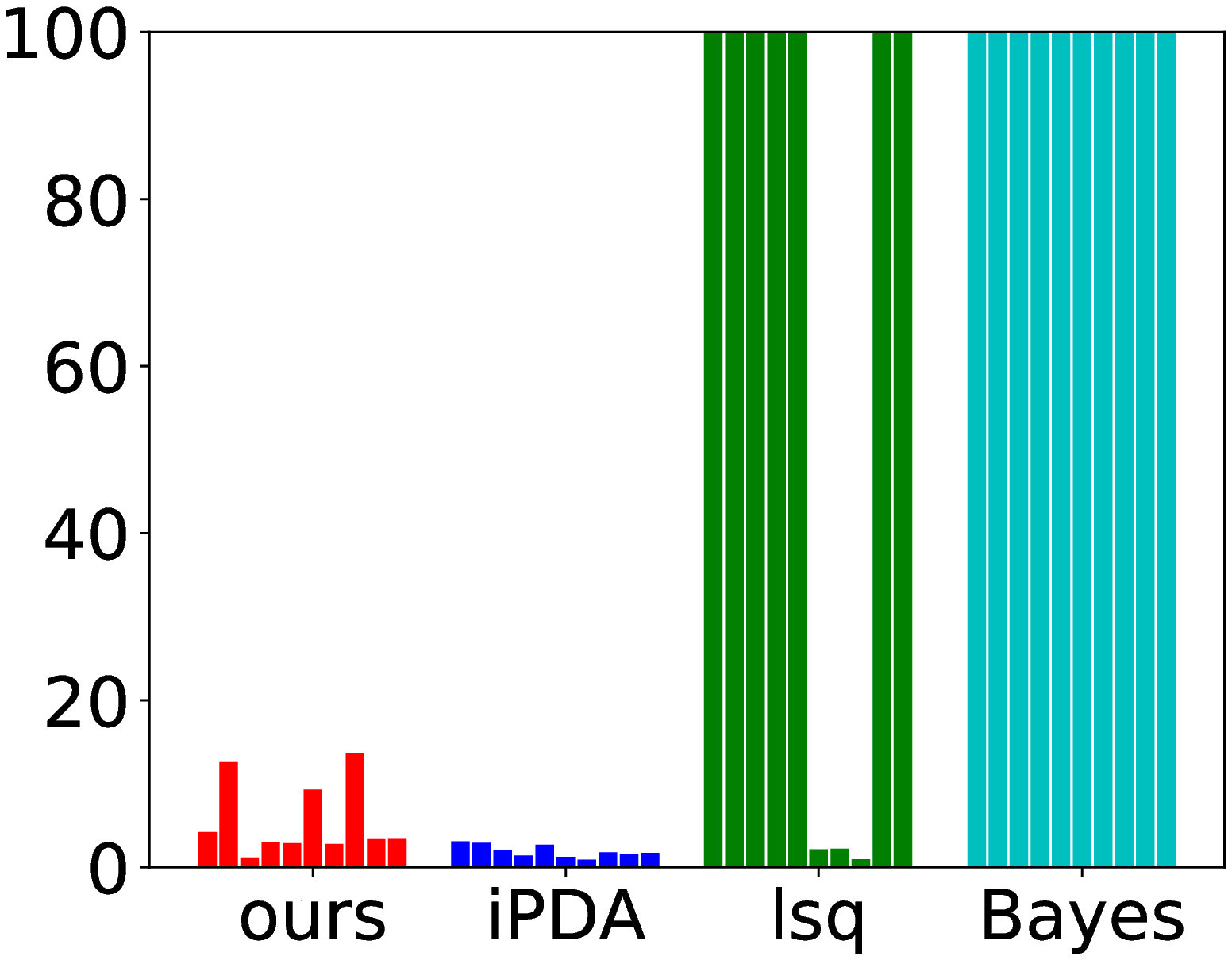}&
    \includegraphics*[width=0.235\linewidth,height=0.12\linewidth]{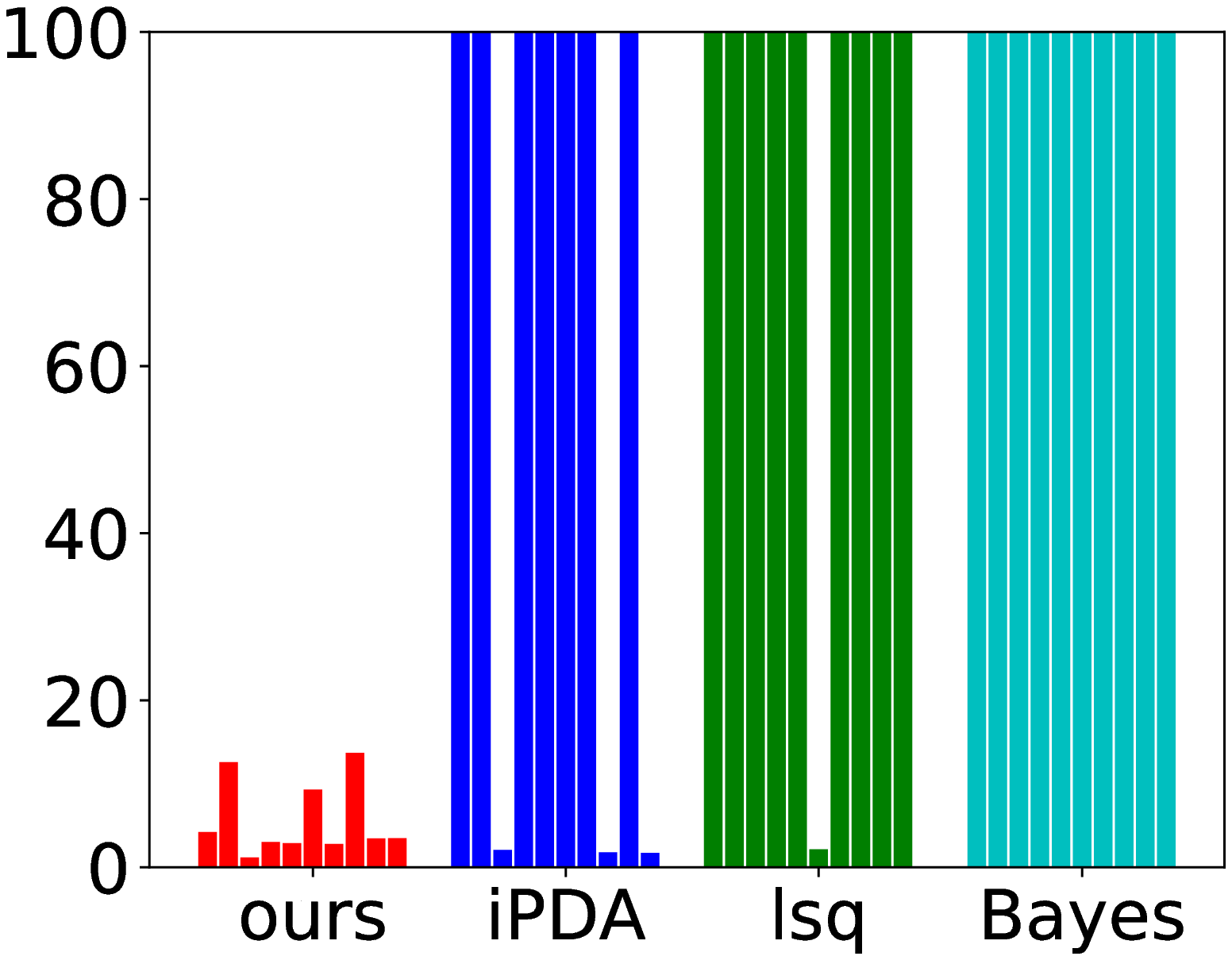}&
    \includegraphics*[width=0.235\linewidth,height=0.12\linewidth]{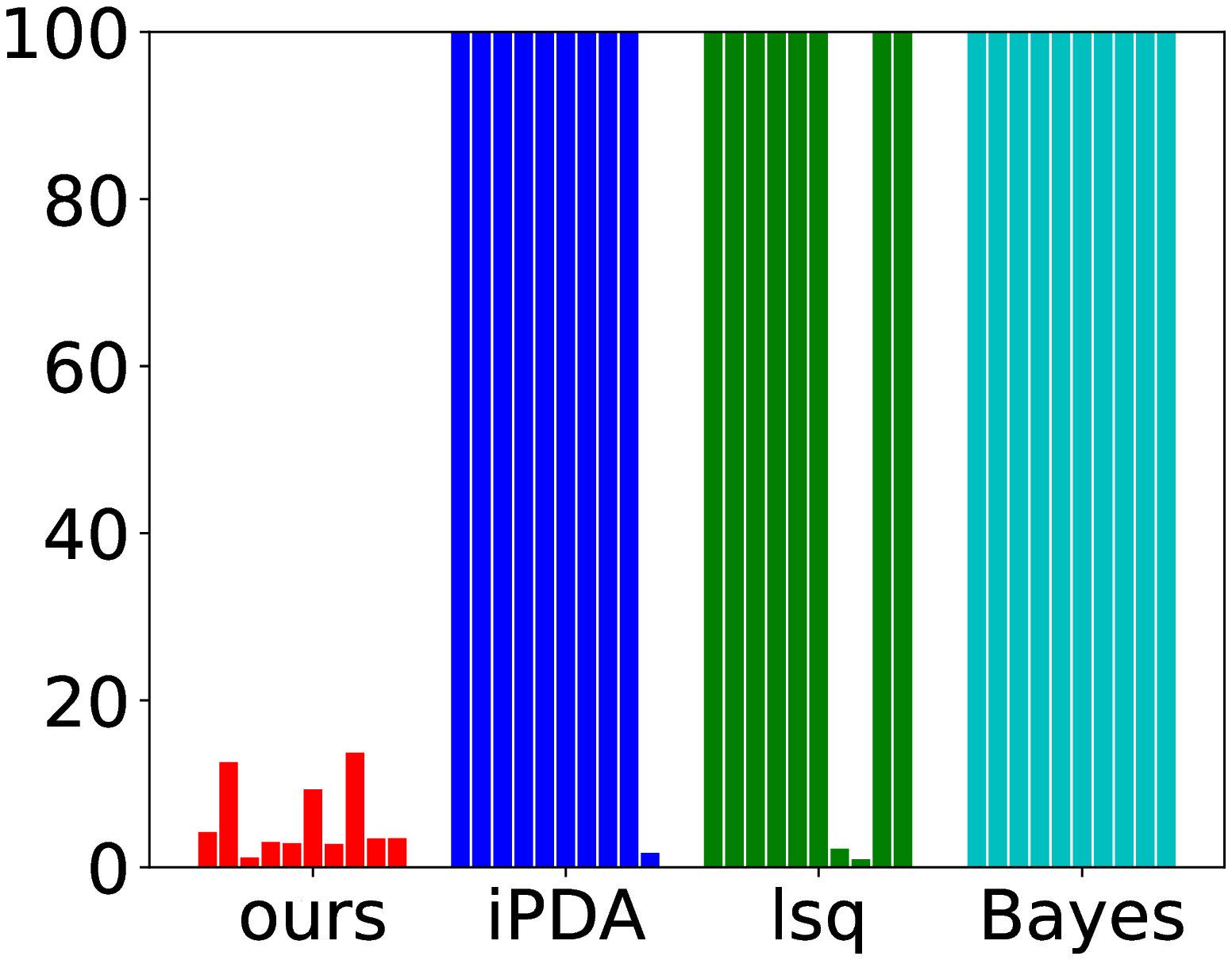}\\
     \hspace{3ex}\rotatebox{90}{\hspace{02.5ex} \small{$\theta_0$ error} } &       
    \hspace{1ex}\includegraphics*[width=0.225\linewidth,height=0.12\linewidth]{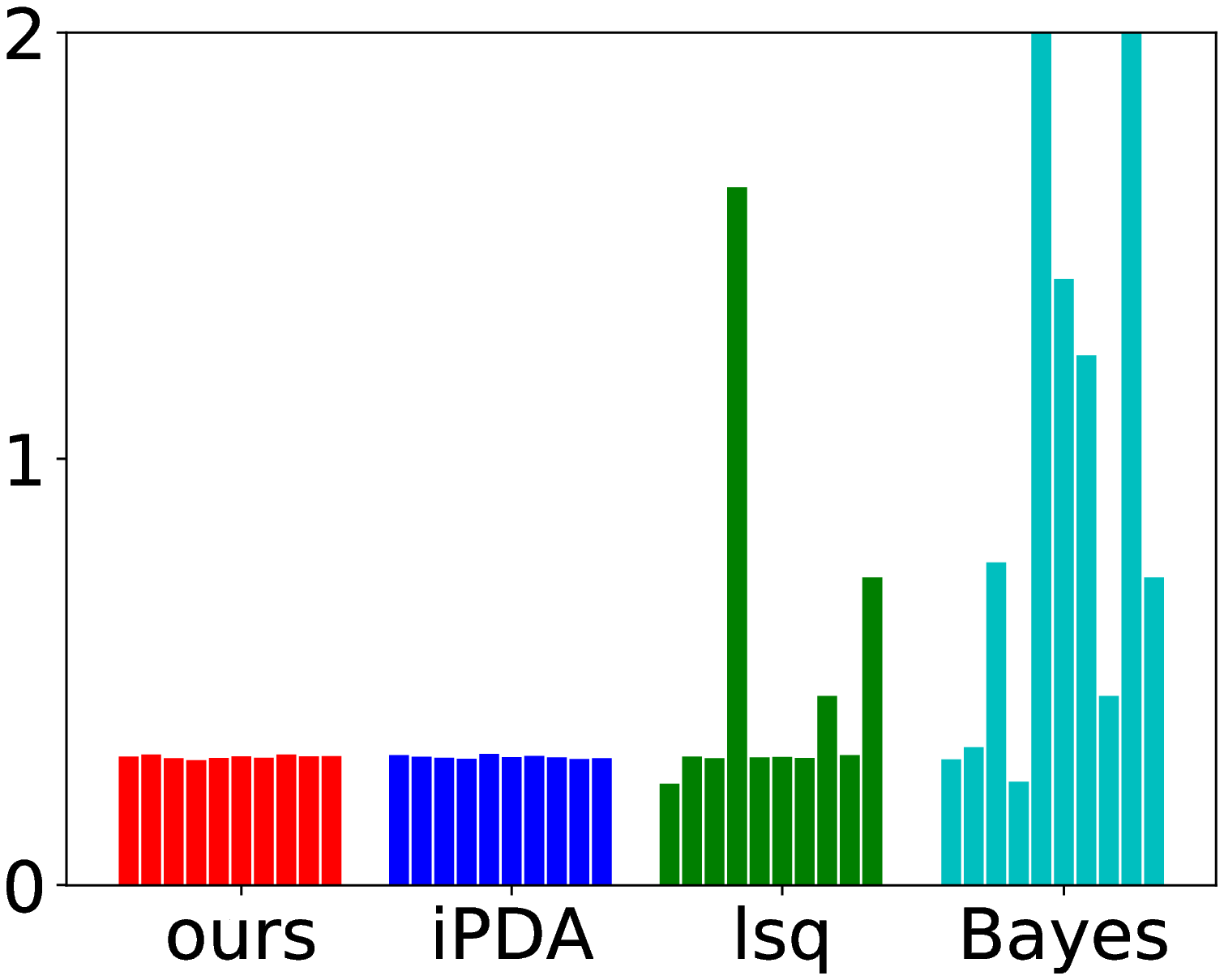}&
    \hspace{1ex}\includegraphics*[width=0.225\linewidth,height=0.12\linewidth]{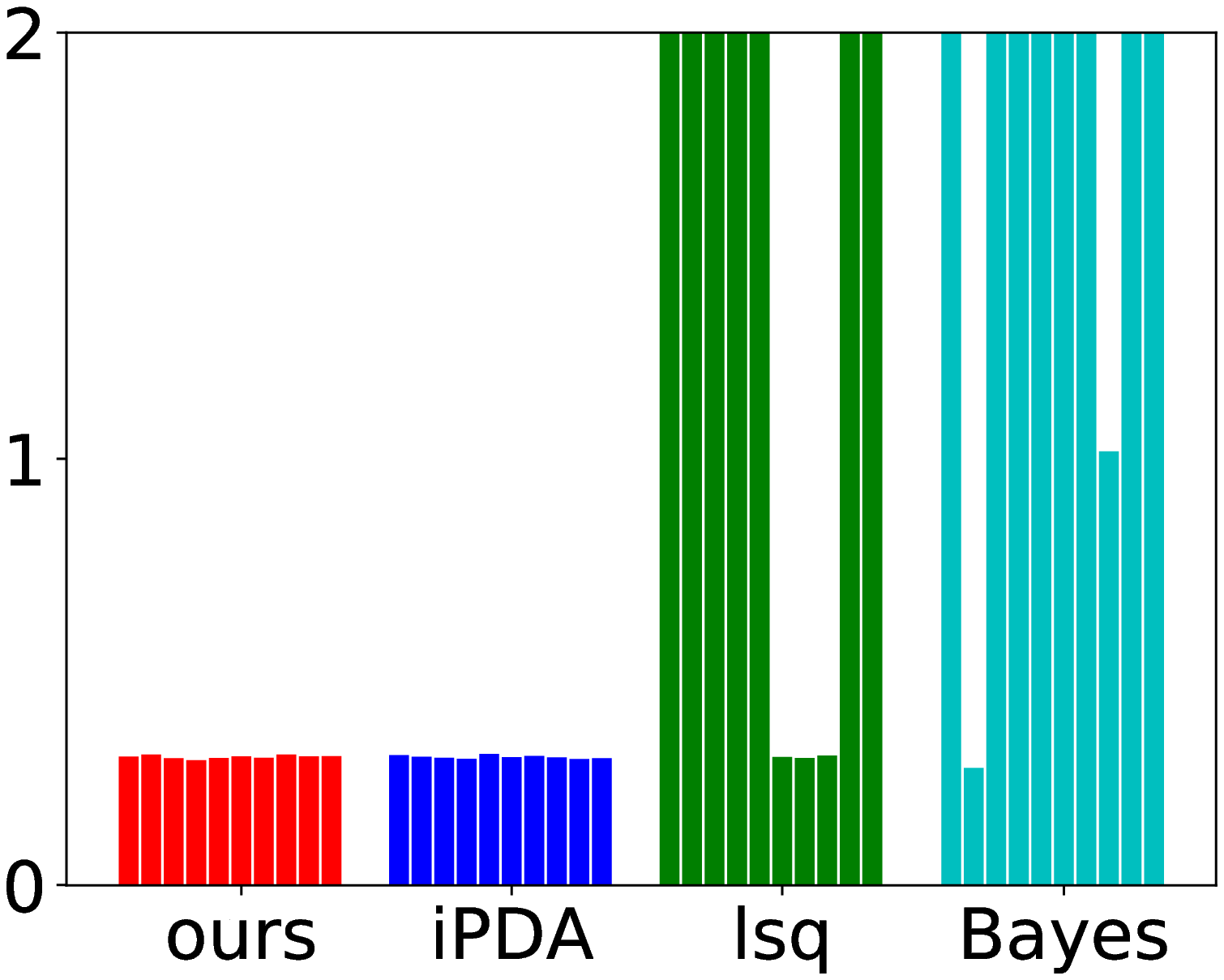}&
    \hspace{1ex}\includegraphics*[width=0.225\linewidth,height=0.12\linewidth]{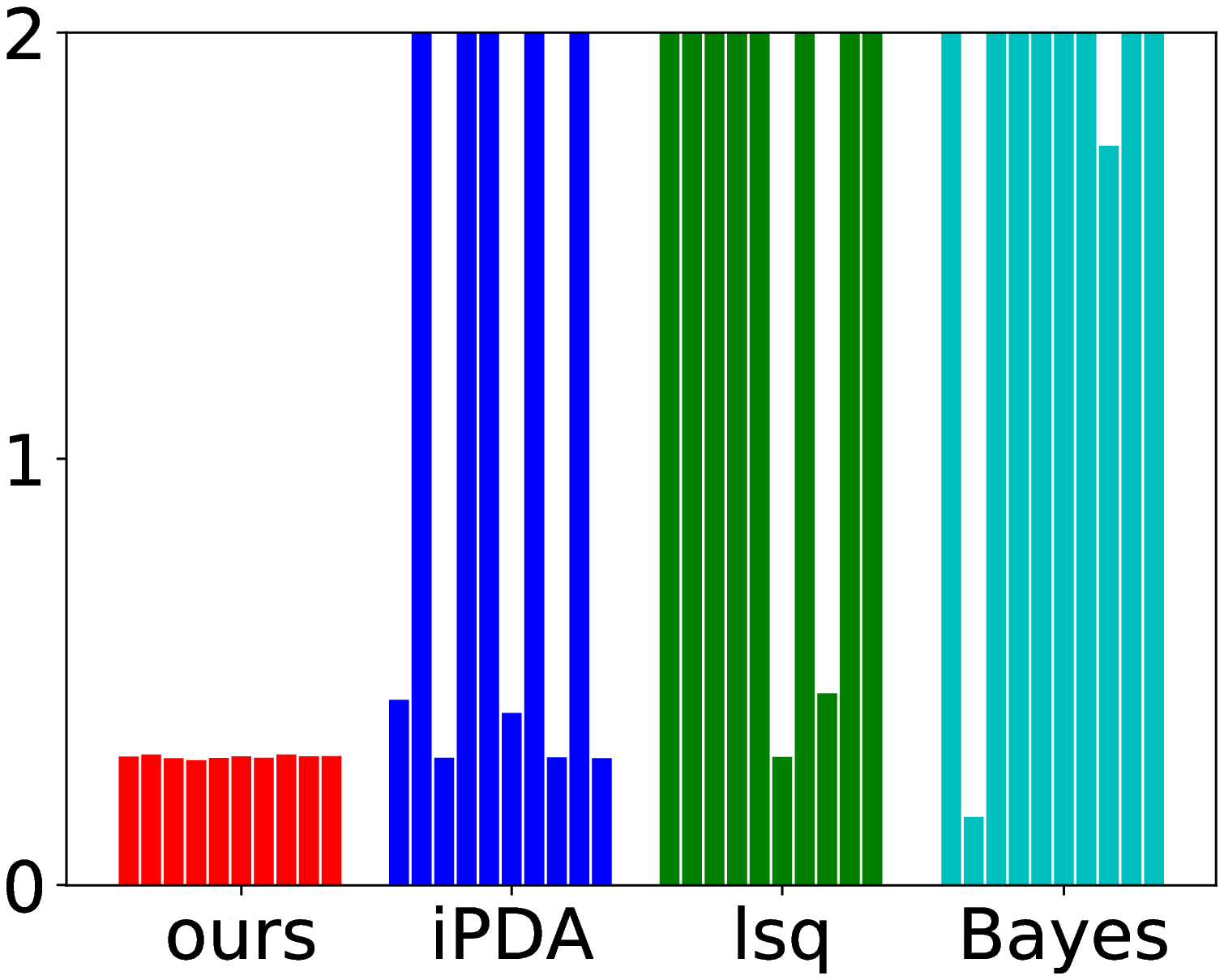}&
    \hspace{1ex}\includegraphics*[width=0.225\linewidth,height=0.12\linewidth]{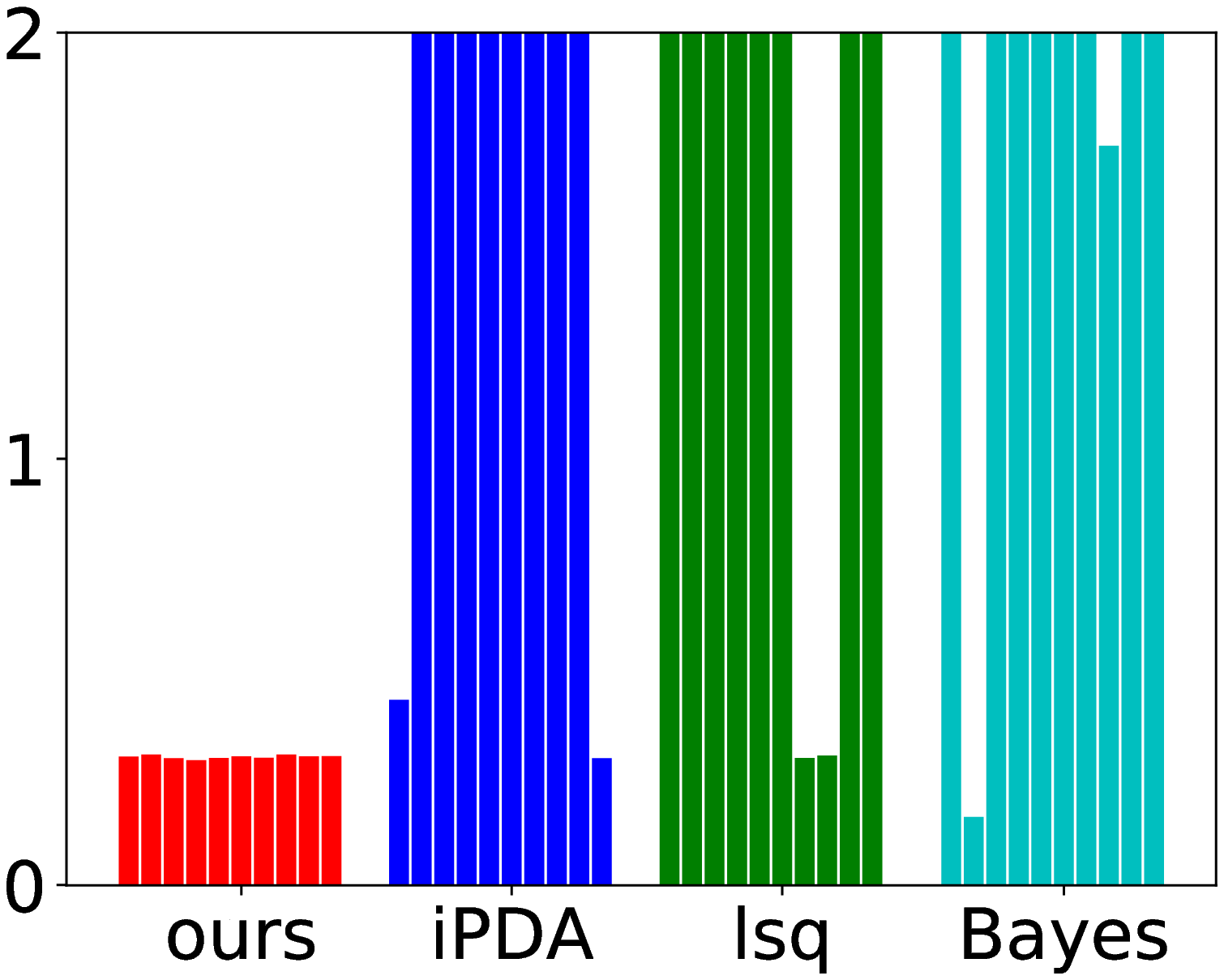}\\
     \hspace{3ex}\rotatebox{90}{\hspace{2.5ex} \small{$\theta_1$ error} } &       
    \hspace{1ex}\includegraphics*[width=0.225\linewidth,height=0.12\linewidth]{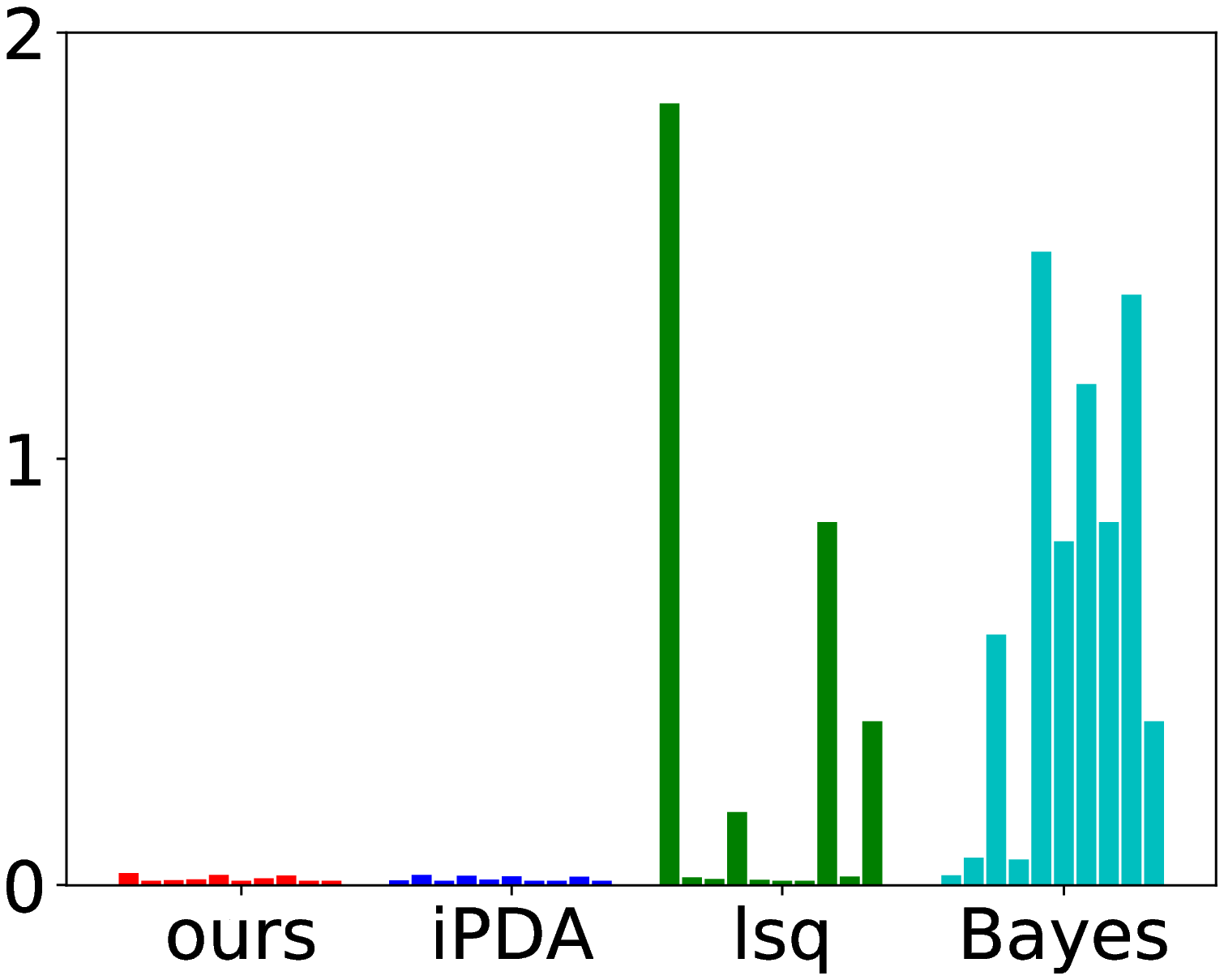}&
    \hspace{1ex}\includegraphics*[width=0.225\linewidth,height=0.12\linewidth]{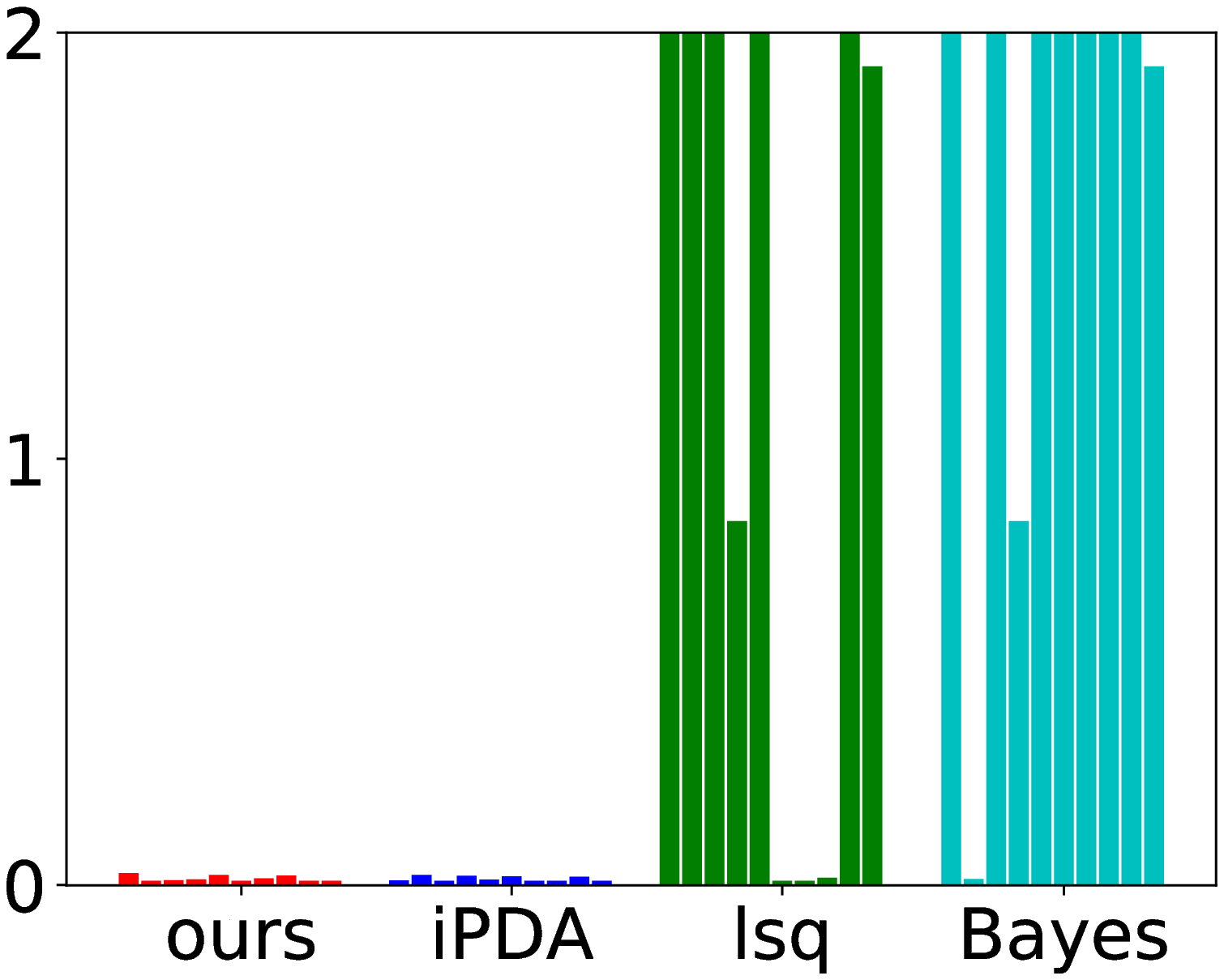}&
    \hspace{1ex}\includegraphics*[width=0.225\linewidth,height=0.12\linewidth]{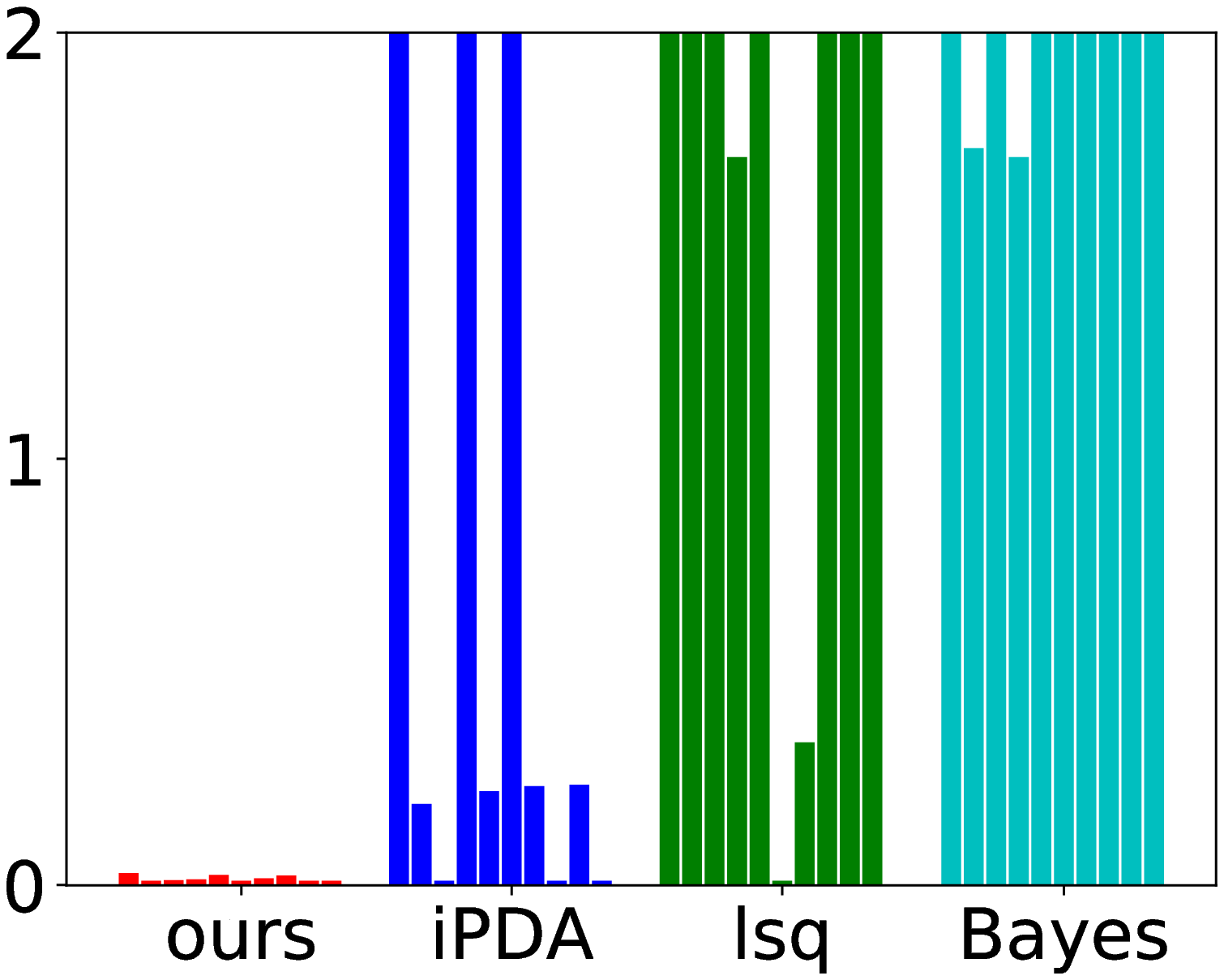}&
    \hspace{1ex}\includegraphics*[width=0.225\linewidth,height=0.12\linewidth]{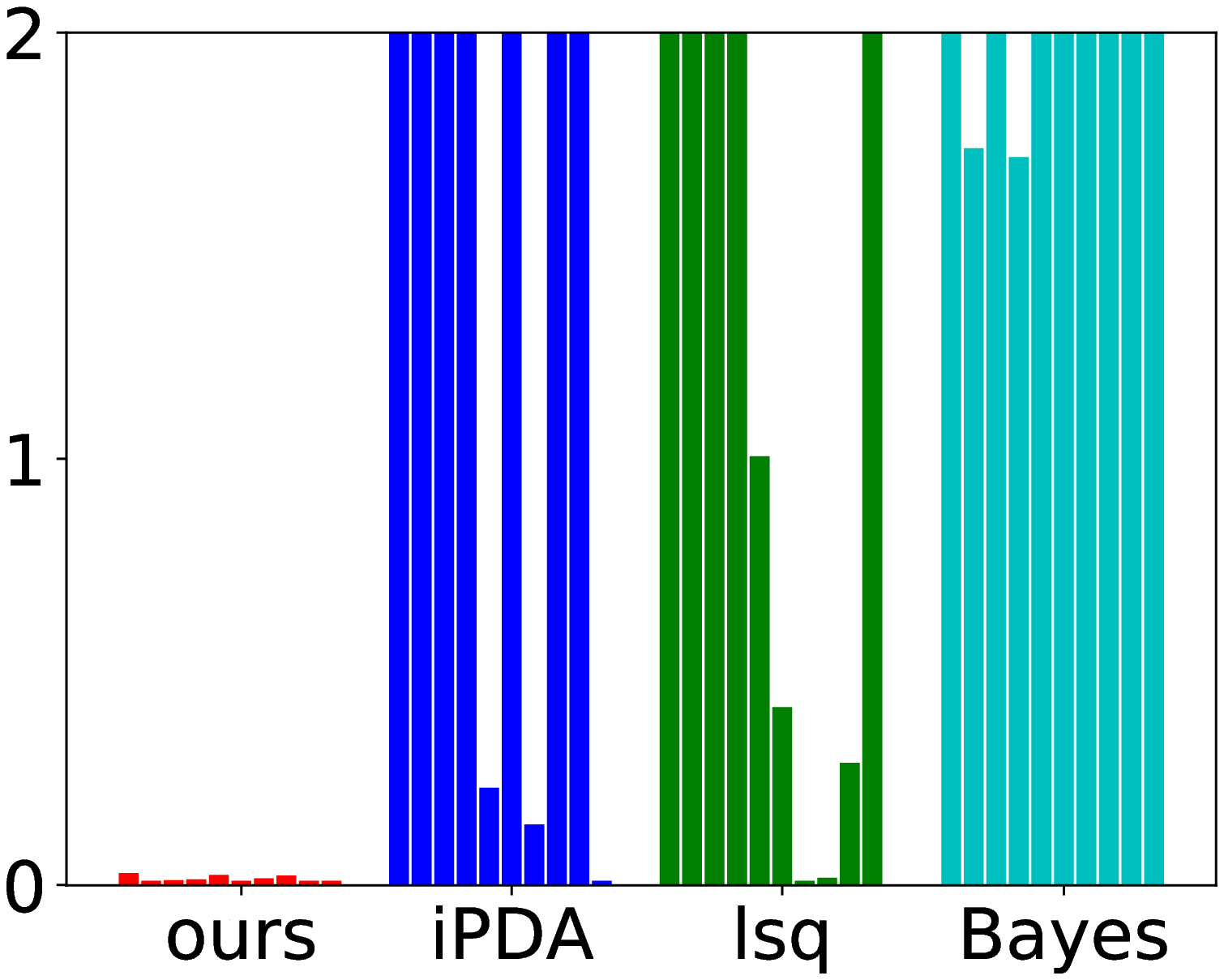}\\
     \hspace{3ex}\rotatebox{90}{\hspace{2.5ex} \small{$\theta_2$ error} } &       
    \hspace{1ex}\includegraphics*[width=0.225\linewidth,height=0.12\linewidth]{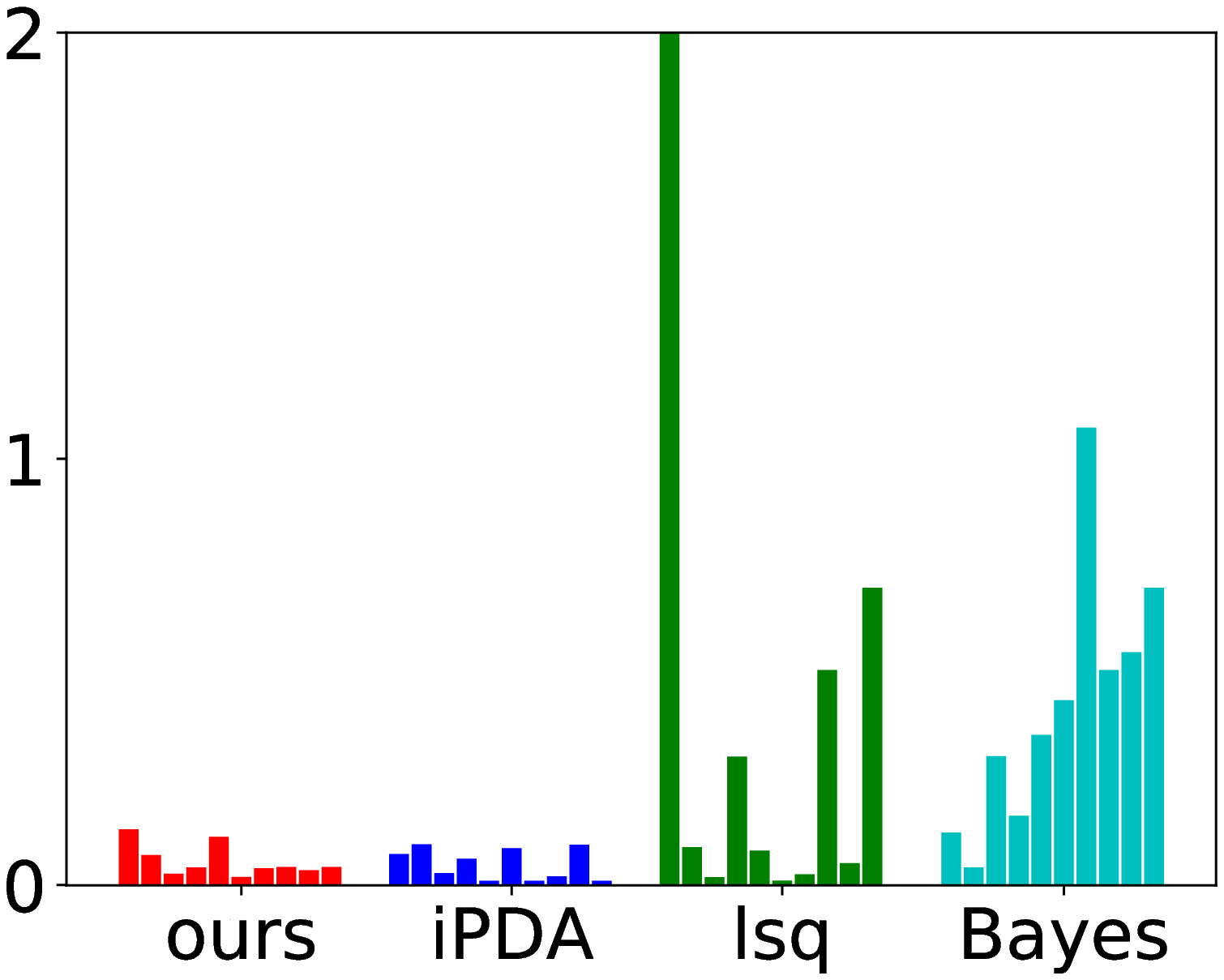}&
    \hspace{1ex}\includegraphics*[width=0.225\linewidth,height=0.12\linewidth]{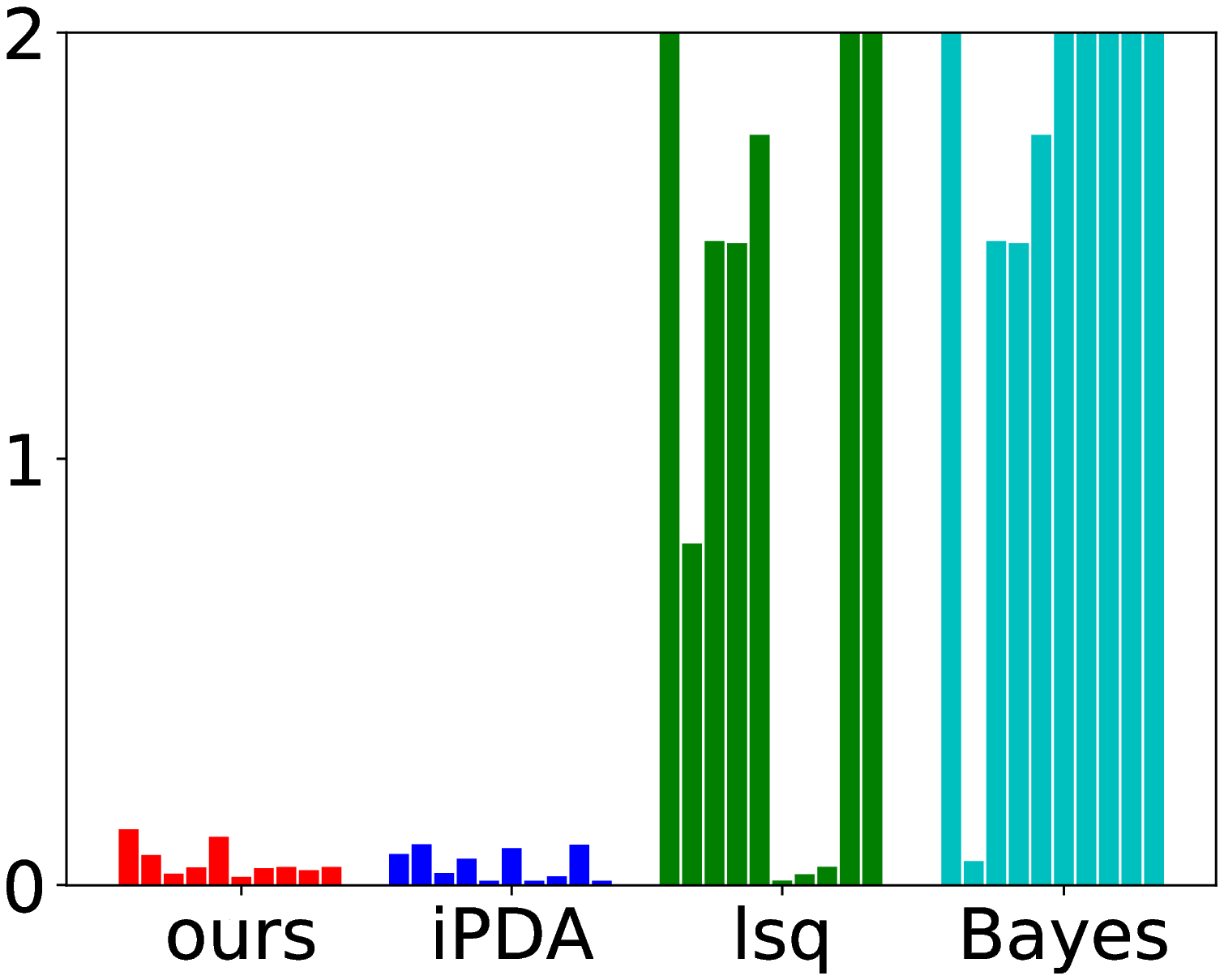}&
    \hspace{1ex}\includegraphics*[width=0.225\linewidth,height=0.12\linewidth]{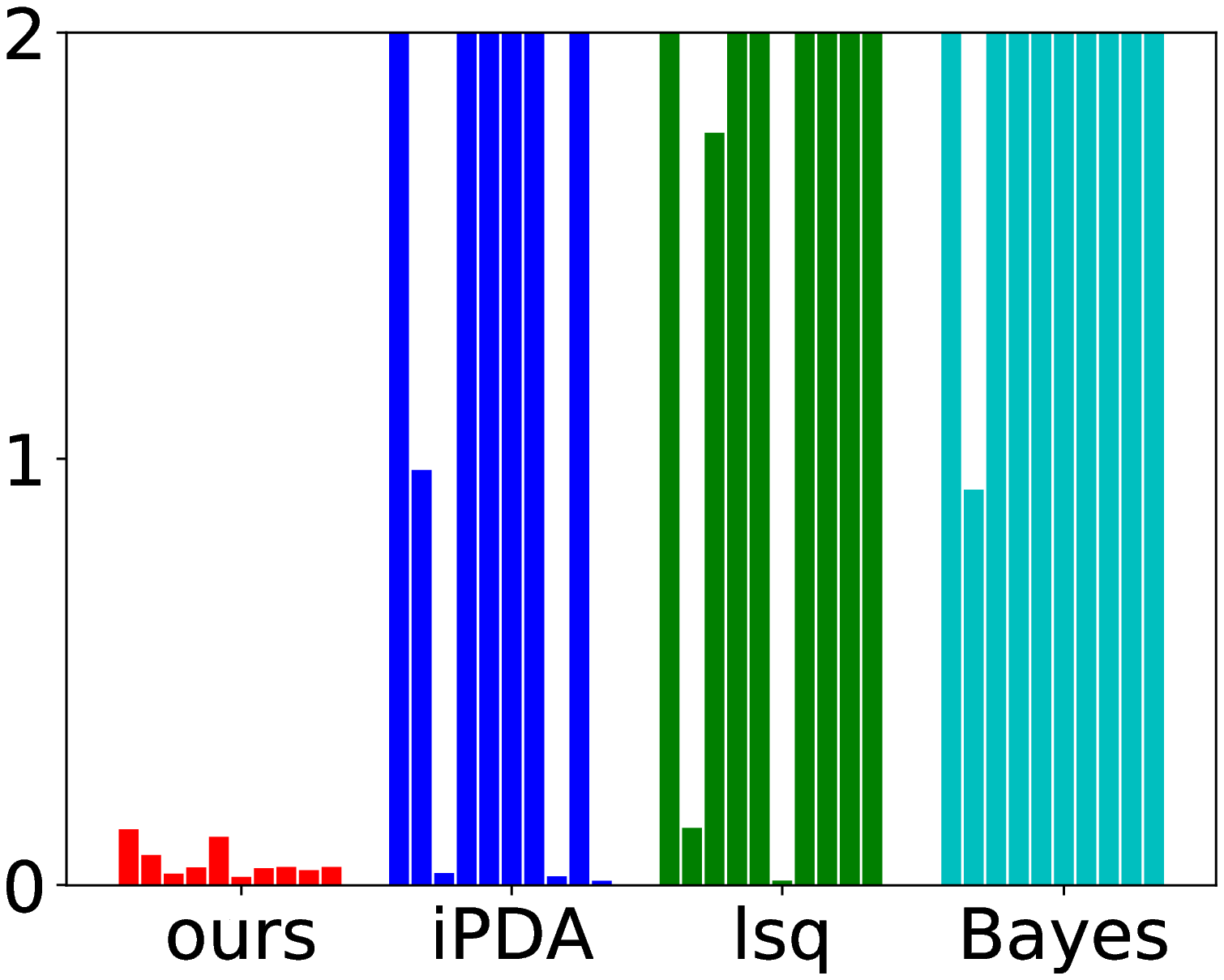}&
    \hspace{1ex}\includegraphics*[width=0.225\linewidth,height=0.12\linewidth]{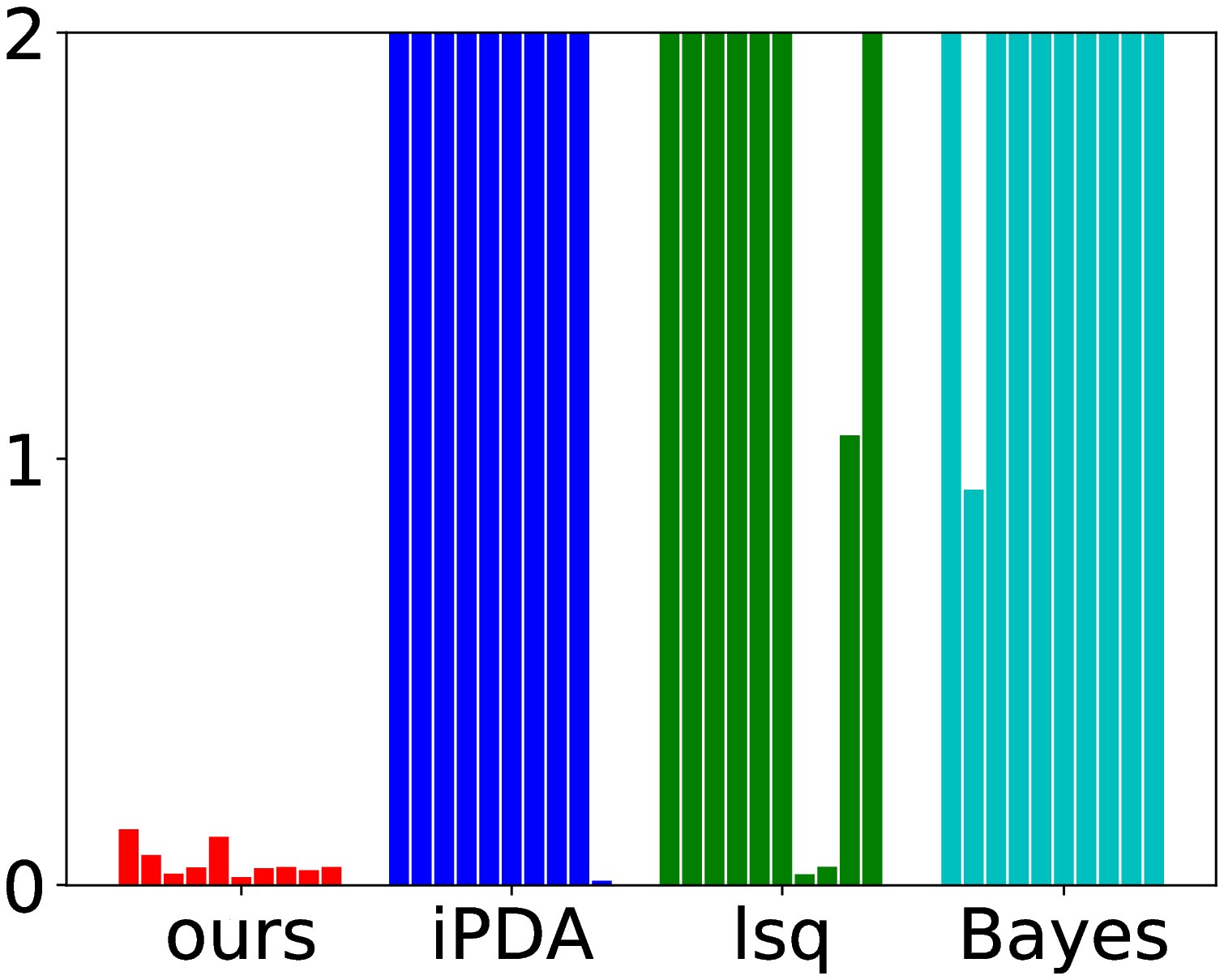}\\[-1ex]
  \end{tabular}
  \vspace{-2ex}
  \caption{Comparison with other methods. We create initializations by adding Gaussian noise of variance $\sigma_{\btheta}^2$ to the true parameters. We create $10$ sets of observations and initializations per each $\sigma_{\btheta}^2$ and report the errors. Each error bar corresponds to the error in one of the experiments. BCD-prox performs significantly better.}
  \vspace{-3ex}
  \label{f:comp}
\end{figure*}

\paragraph{Robustness to the hyperparameter $\lambda$.}The only hyperparameter in our algorithm is $\lambda$. In Fig.~\ref{f:robust}, we set $\lambda$ in turn to a set of values from $0$ to $20$, run our algorithm, and report the results after convergence. In both models, we generate observations with the variance $\sigma^2=0.5$. Because of randomness included in creating noisy observations, we create $10$ sets of observations, run our algorithm once for each of them, and report the mean in Fig.~\ref{f:robust}. We also show the standard deviation in prediction errors, but not in parameter values (to avoid clutter).

In Fig.~\ref{f:robust} we report the prediction error and the estimated parameters for each value of $\lambda$. The true values for the FitzHugh--Nagumo are $\theta_0=.5, \theta_1=.3,$ and $\theta_2=3$. For the Lotka--Volterra model, the true values are $\theta_0=2, \theta_1=1, \theta_2=4$, and $\theta_3=1$.
 
We see in Fig.~\ref{f:robust} that for $\lambda>0$, BCD-prox correctly finds the parameters and brings the error close to zero. Also, in the range of $\lambda=1$ to $20$, the errors and the estimated parameters remain almost the same. We have found that increasing $\lambda$ to $1\,000$ does not change the estimated parameters. The only disadvantage of increasing $\lambda$ to a large value is that training time increases---as explained above, increasing $\lambda$ is analogous to decreasing the step size in a gradient descent method.  Large $\lambda$ implies that states can change very little from one iteration to another, forcing the algorithm to run longer for convergence. The algorithm, as explained in detail before, does not work well when $\lambda = 0$; in this case, the algorithm stops after a single iteration, with the predicted states far from the clean states.

\vspace{-2ex}
\paragraph{Comparison with other methods (robustness to initialization).} As the first experiment, we compare BCD-prox with three other methods, each of them from a different category. Among the iPDA (spline-based) methods, we use a MATLAB code available online \cite{Ramsay07}, denoted by ``iPDA'' in our experiments. Among the Bayesian approaches, we use an R code available online \cite{Dondelin13}, denoted ``Bayes'' in our experiments. We also implement a method that uses the iterative least square approach, denoted ``lsq'' in our experiments. This method considers the parameters and the initial state as the unknown variables. To implement lsq, we use the Python LMFIT package \cite{Newville14}. The variance of the noisy observations is $\sigma^2=0.5$.

All methods including ours need an initial guess for the unknown parameters. We add Gaussian noise with mean $0$ and variance $\sigma_{\theta}^2$ to the true parameter and use the result to initialize the methods. Fig.~\ref{f:comp} shows the results for the R\" ossler model and the supplementary material contains the results on the FitzHugh--Nagumo model. We change the variance from $\sigma_{\theta}^2=1$ to $20$. Since there is randomness in both initialization and observation, we repeat the experiment $10$ times. Note that the comparisons are fair, with the same observations and initializations used across all methods. 

In Fig.~\ref{f:comp}, each of the bars corresponds to the prediction or parameter error for one of the methods in one of the experiments. Hence there are $10$ error bars for each of the methods in each plot. We set $\lambda=1$ in BCD-prox for all the experiments. For the other methods, we chose the best hyperparameters that we could determine after careful experimentation.

The first point in Fig.~\ref{f:comp} is that BCD-prox is robust with respect to the initialization, while the other methods are not. The total number of experiments per method is $80$ ($40$ for the FitzHugh--Nagumo and $40$ for R\" ossler). The prediction error of BCD-prox exceeds $100$ in $4$ experiments. The prediction error of iPDA (the second best method after ours) exceeds $100$ in $39$ experiments (nearly half the experiments). For lsq and Bayes, the errors are substantially worse.

\begin{figure*}[!t]
  \centering
  \begin{tabular}{@{}l@{}c@{\hspace{1ex}}c@{\hspace{0ex}}c@{\hspace{0ex}}c@{}c@{}c@{}}
    	 & & \small{prediction error} & \small{$\theta_0$ error} & \small{$\theta_1$ error} & \small{$\theta_2$ error} & \small{$\theta_3$ error} \\[-.5ex]
     \rotatebox{90}{\hspace{5ex}  \small{$T=20$ }} &
    \rotatebox{90}{\hspace{5ex}  \small{$\sigma^2=1$} } &
	\includegraphics*[width=0.19\linewidth]{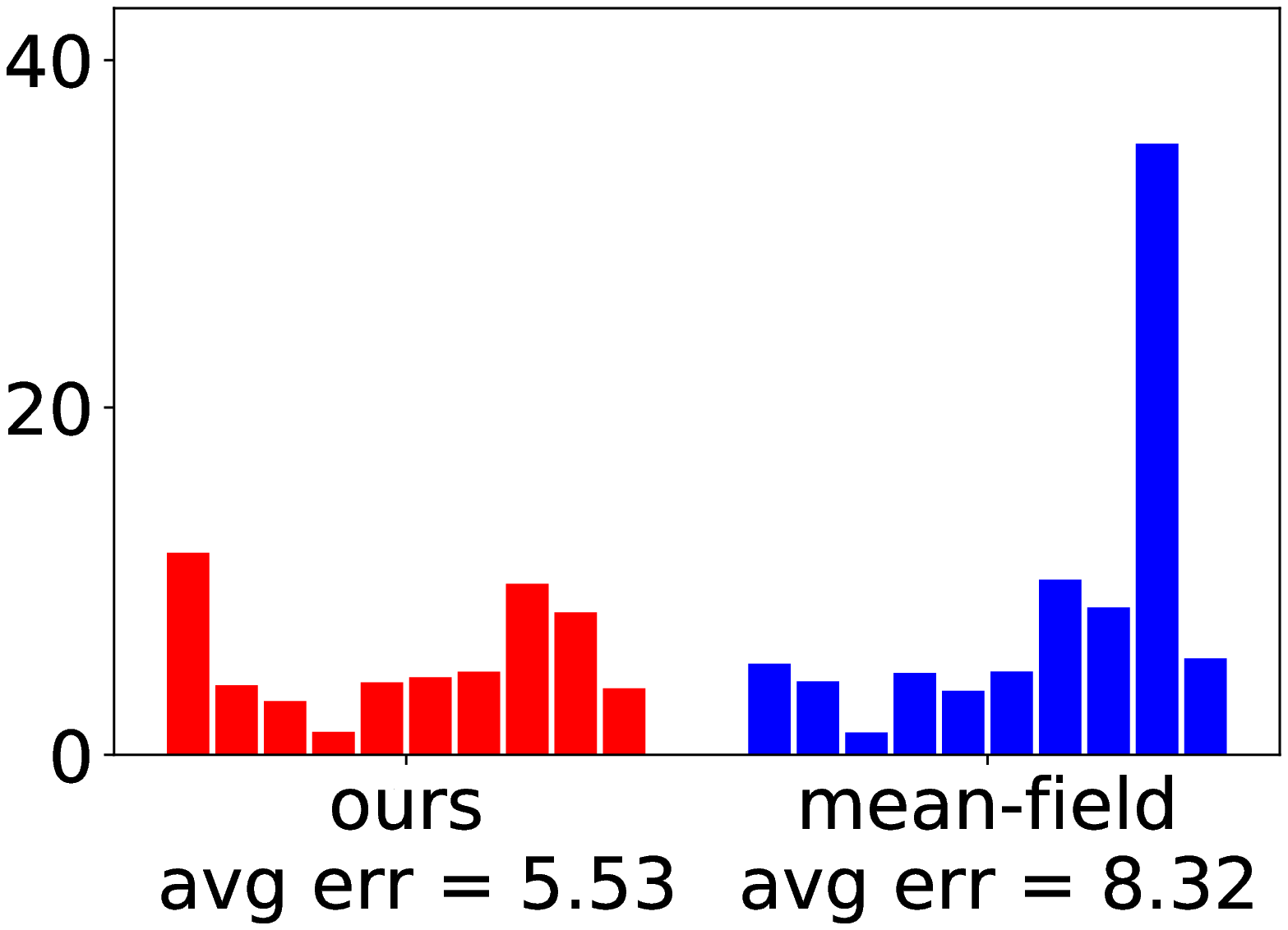}&
	 \includegraphics*[width=0.19\linewidth]{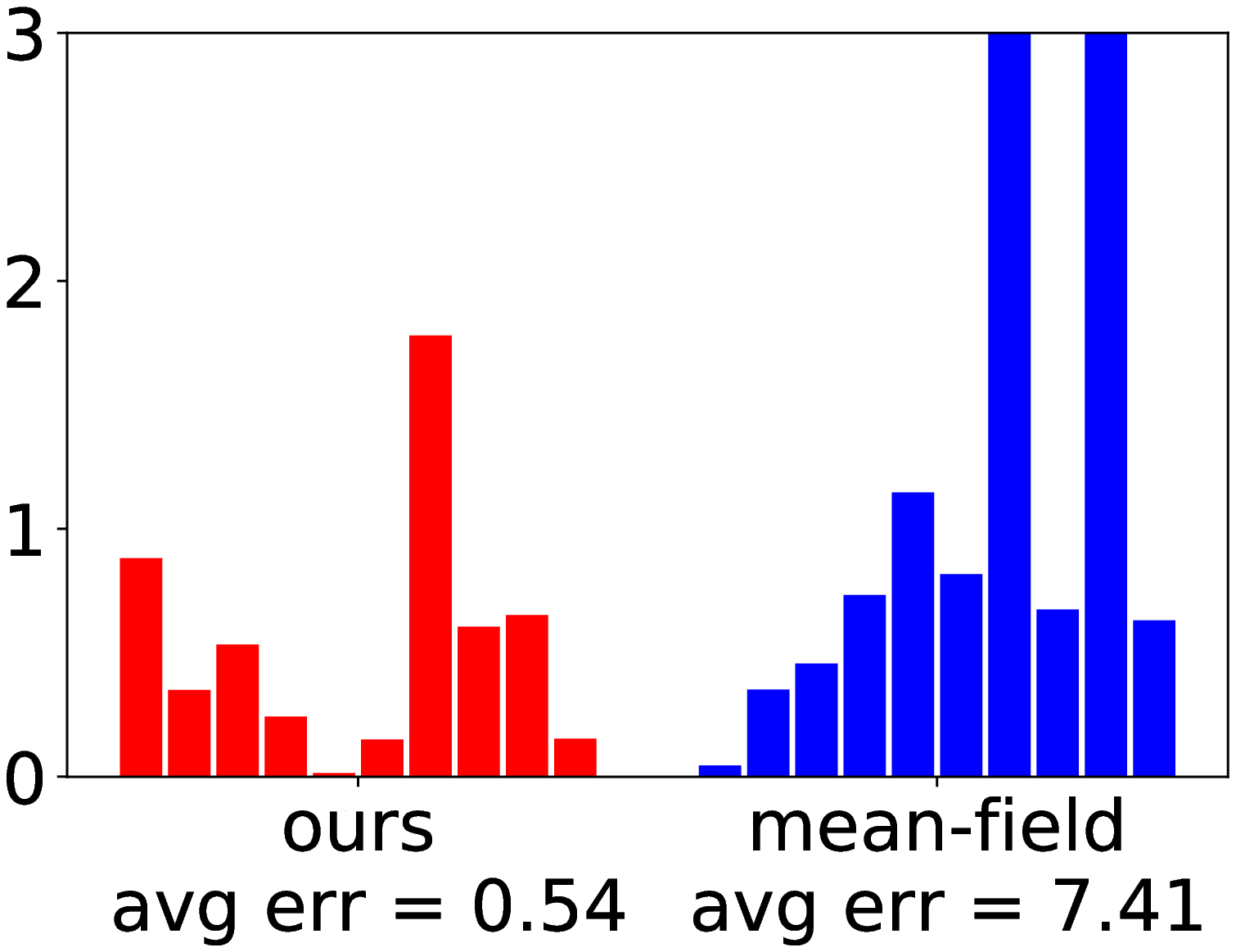}&
	 \includegraphics*[width=0.19\linewidth]{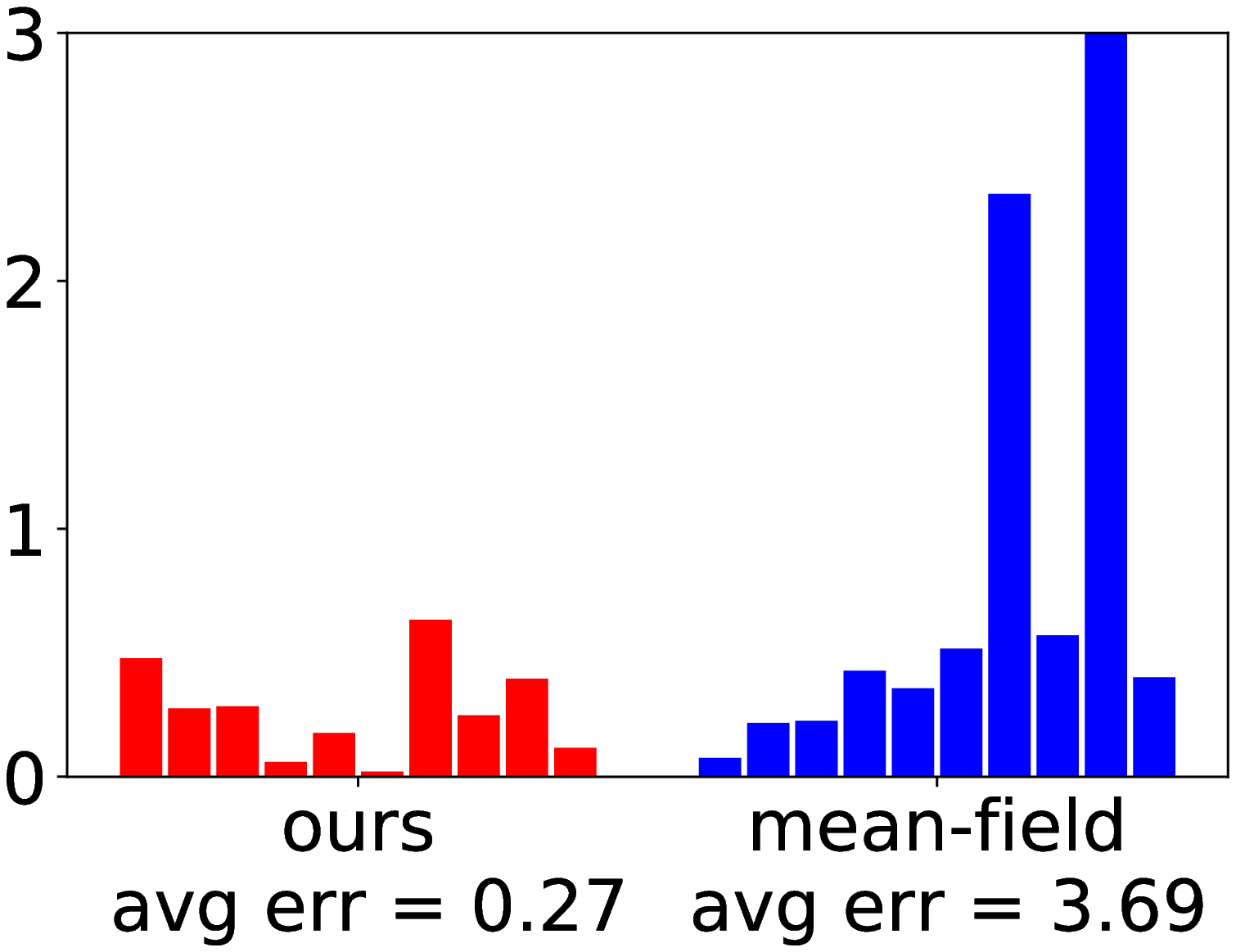}&
	 \includegraphics*[width=0.19\linewidth]{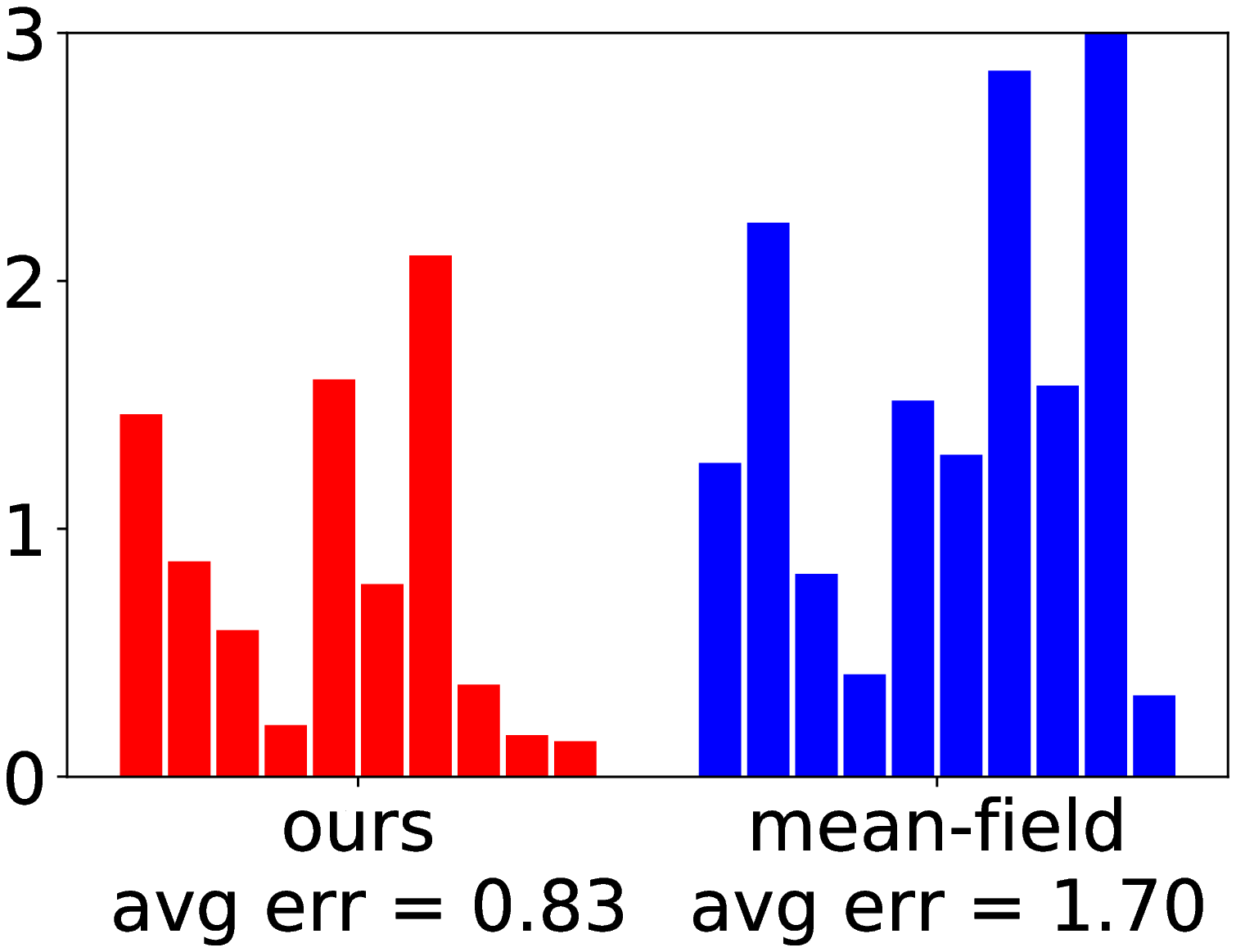}&
	 \includegraphics*[width=0.19\linewidth]{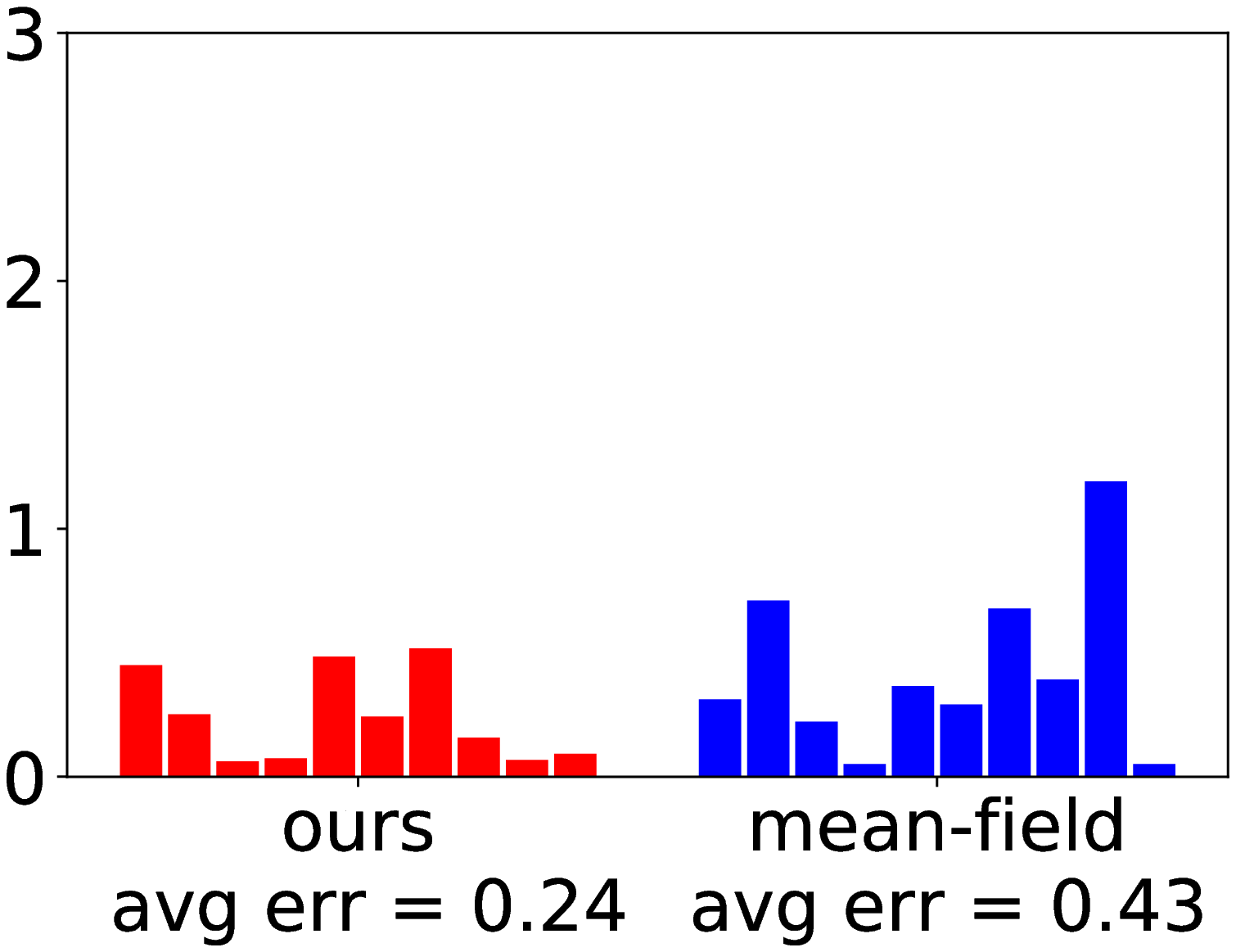}\\
	\rotatebox{90}{\hspace{5ex}  \small{$T=10^4$} } &
    \rotatebox{90}{\hspace{5ex}  \small{$\sigma^2=0.1$} } &
	\includegraphics*[width=0.19\linewidth]{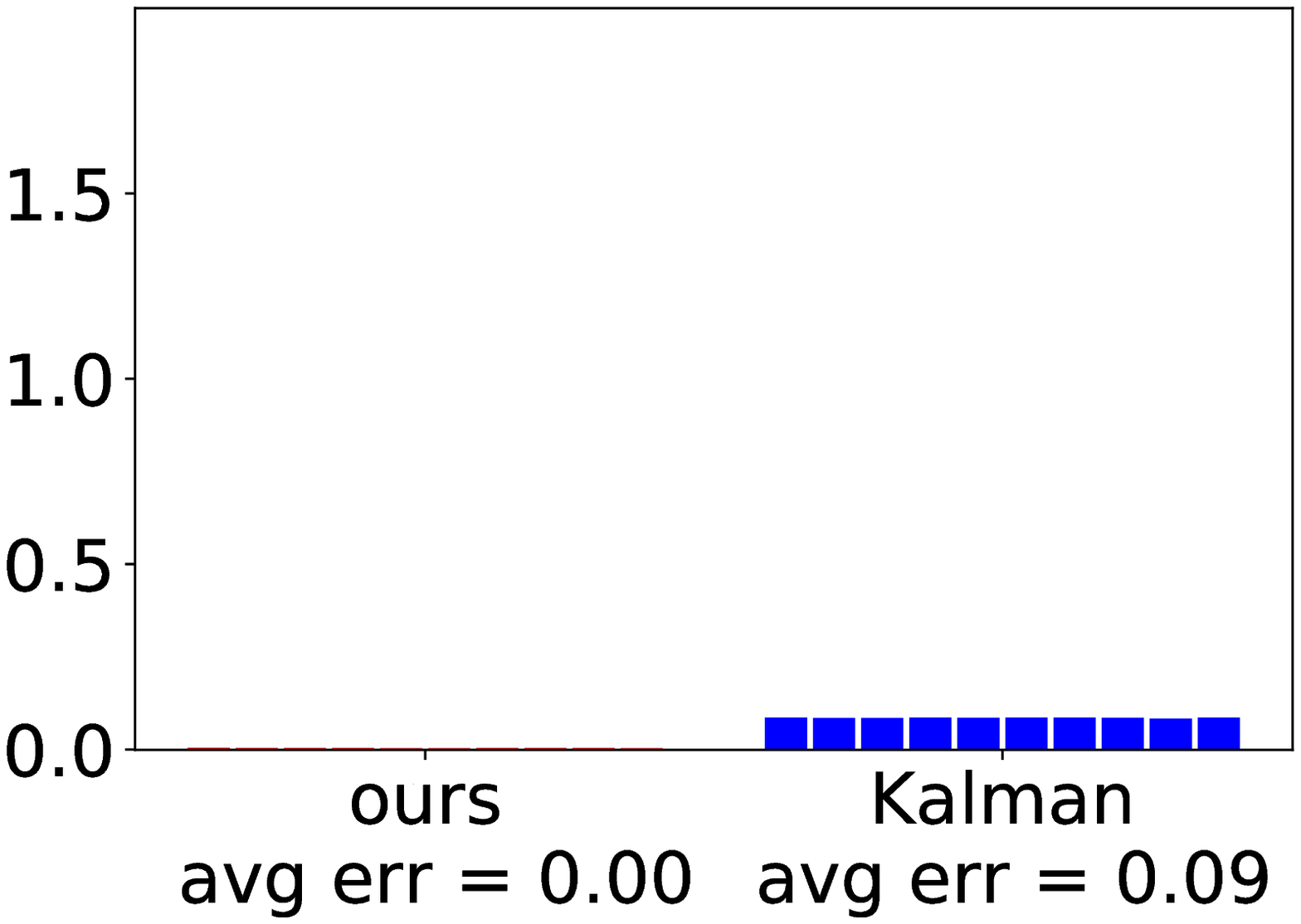}&
	 \includegraphics*[width=0.19\linewidth]{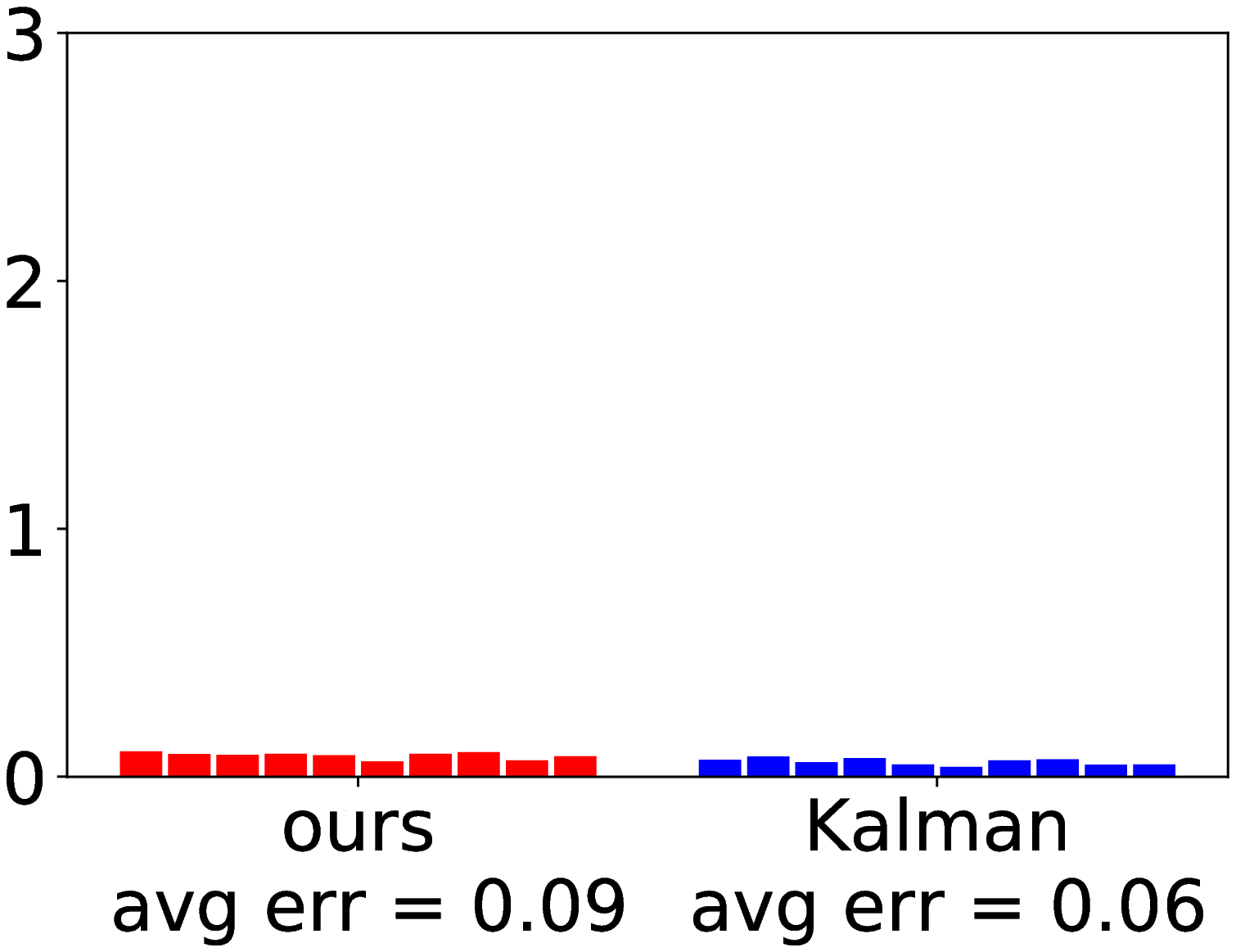}&
	 \includegraphics*[width=0.19\linewidth]{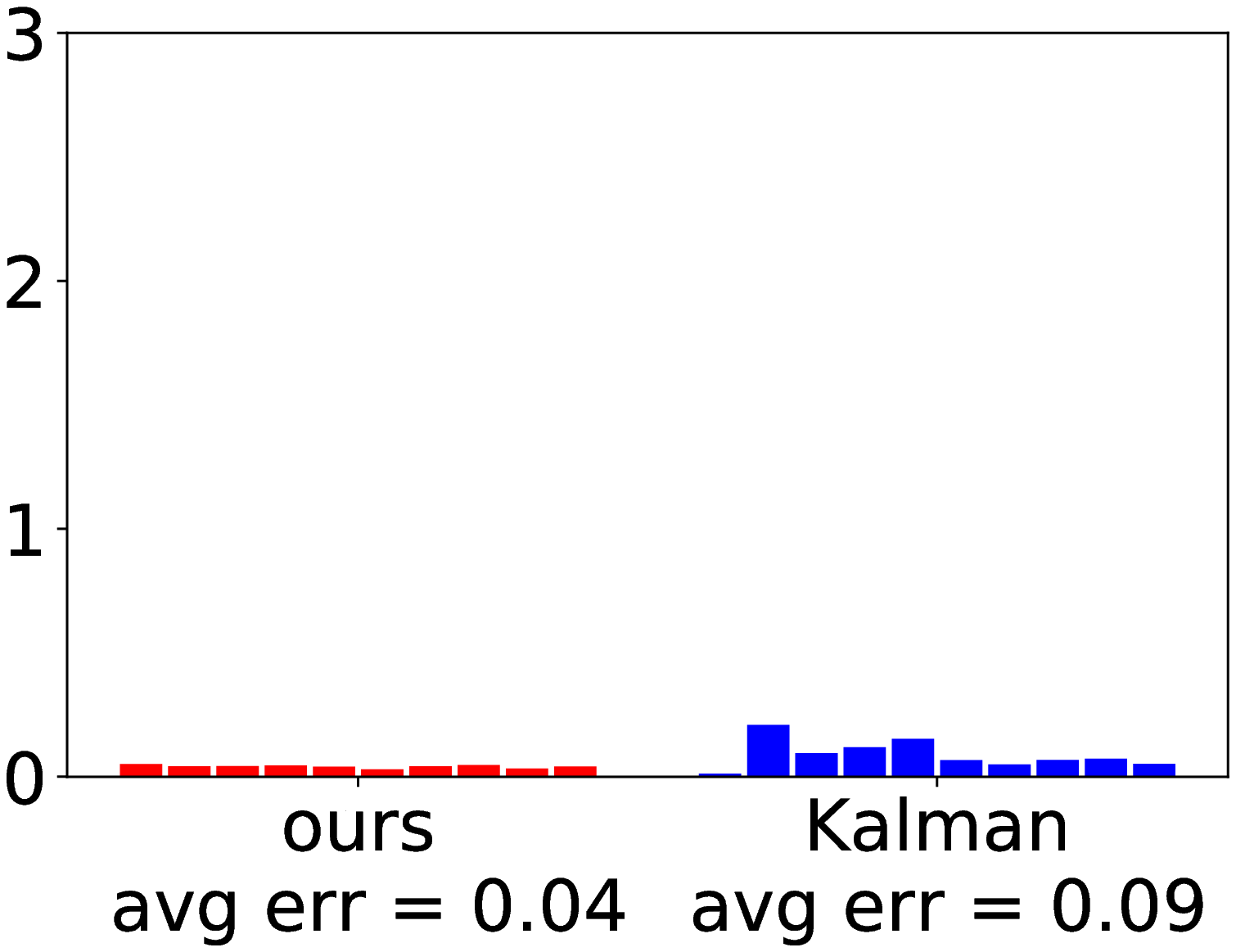}&
	 \includegraphics*[width=0.19\linewidth]{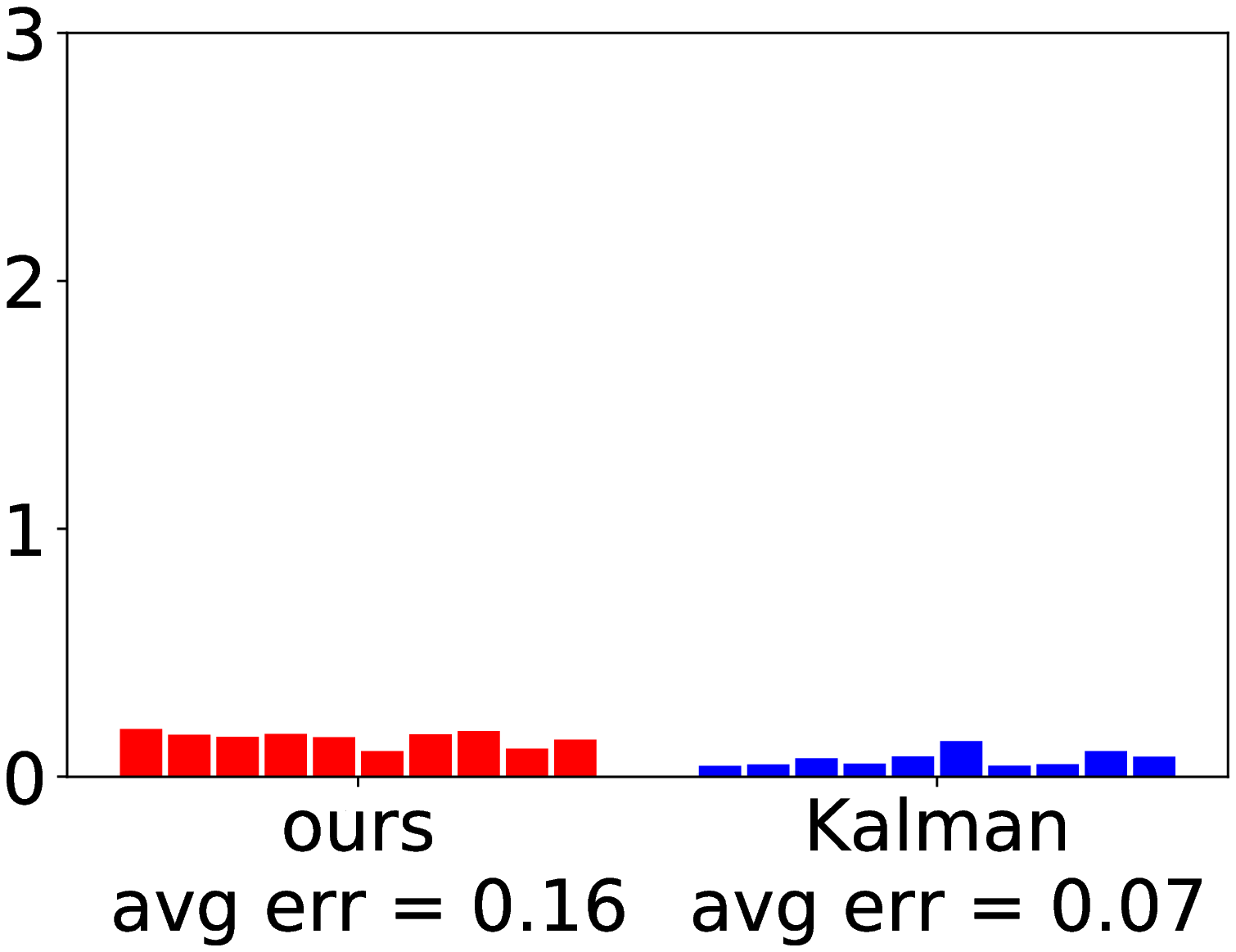}&
	 \includegraphics*[width=0.19\linewidth]{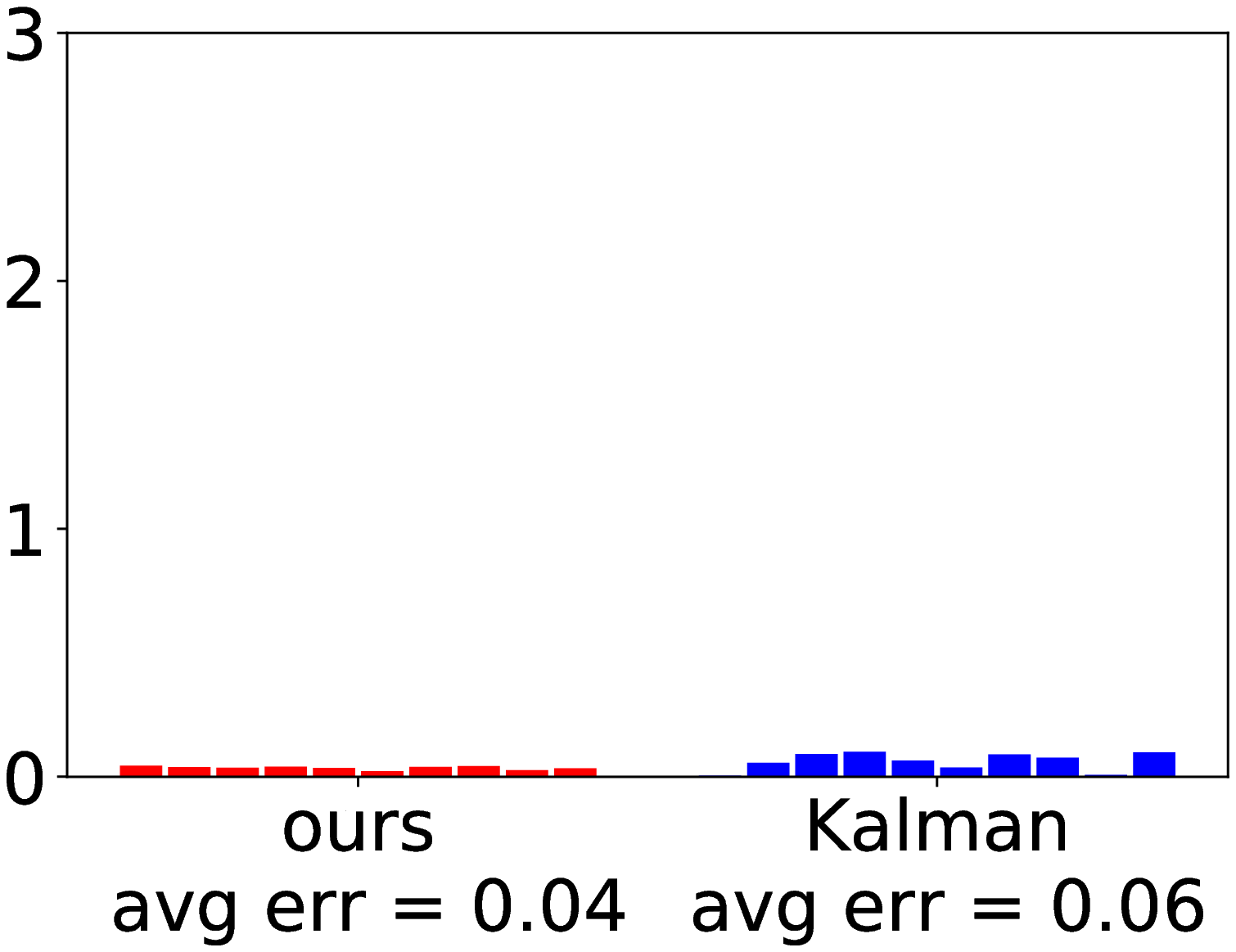}\\
	\rotatebox{90}{\hspace{5ex}  \small{$T=10^4$} } &
    \rotatebox{90}{\hspace{5ex}  \small{$\sigma^2=1.5$} } &
	\includegraphics*[width=0.19\linewidth]{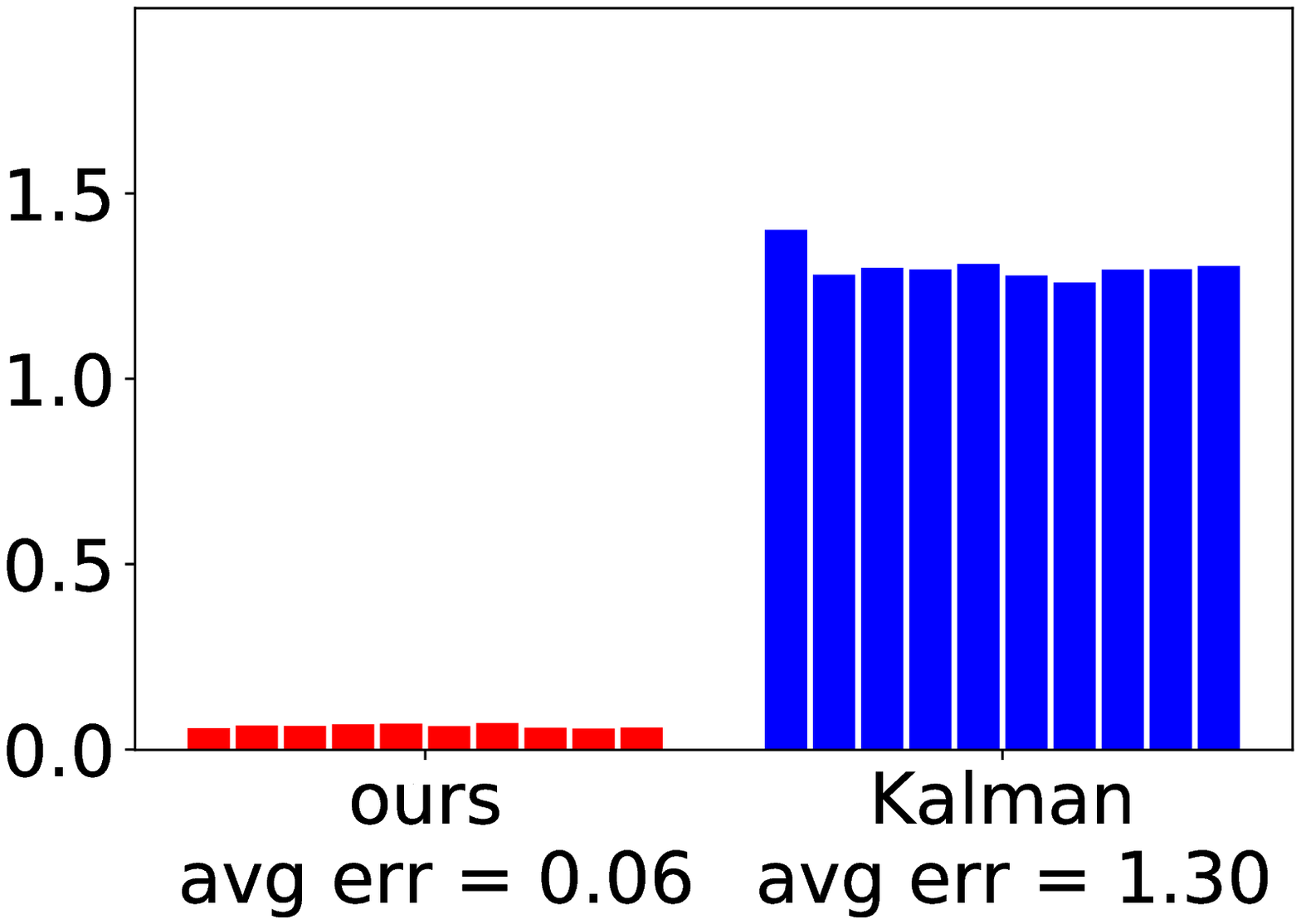}&
	 \includegraphics*[width=0.19\linewidth]{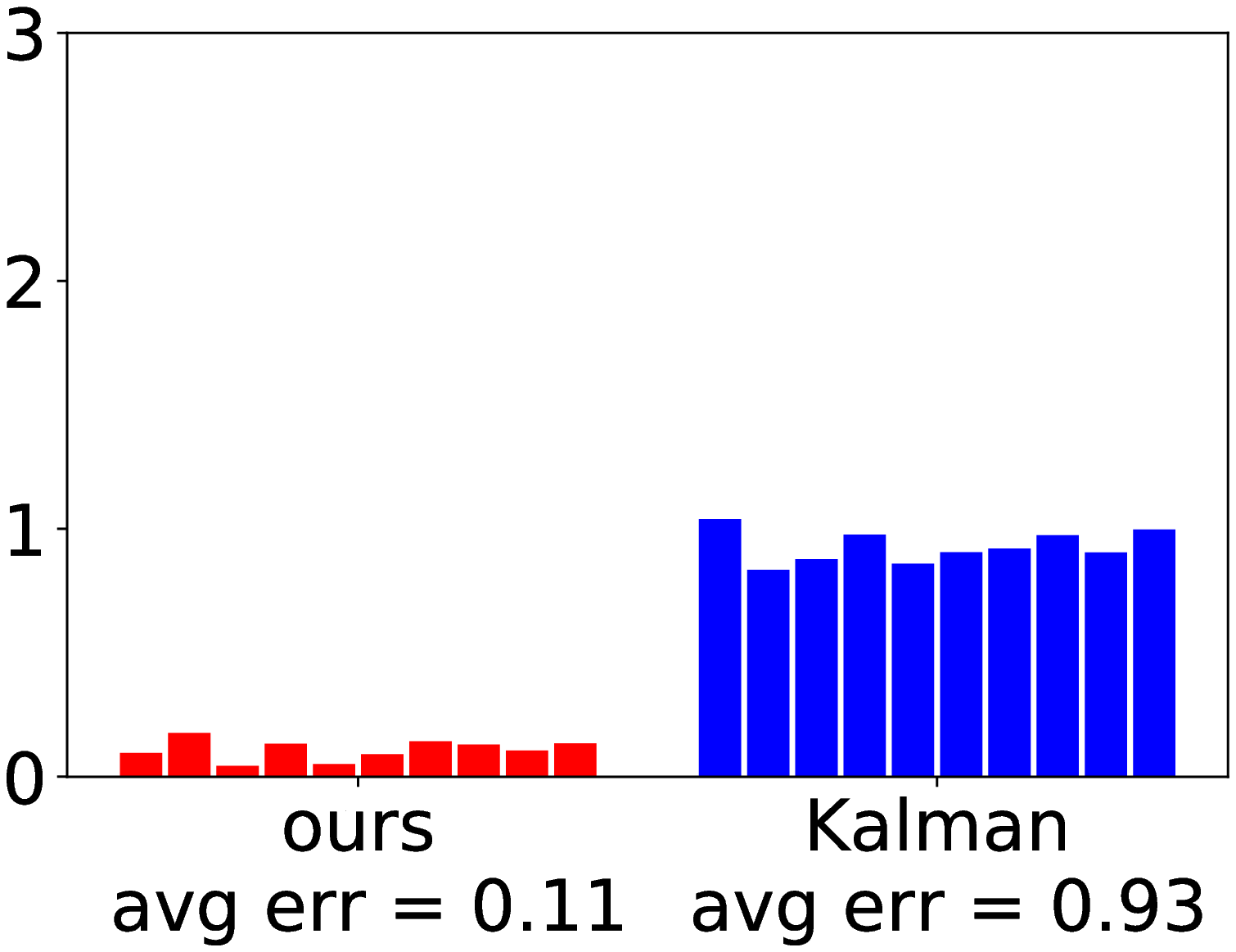}&
	 \includegraphics*[width=0.19\linewidth]{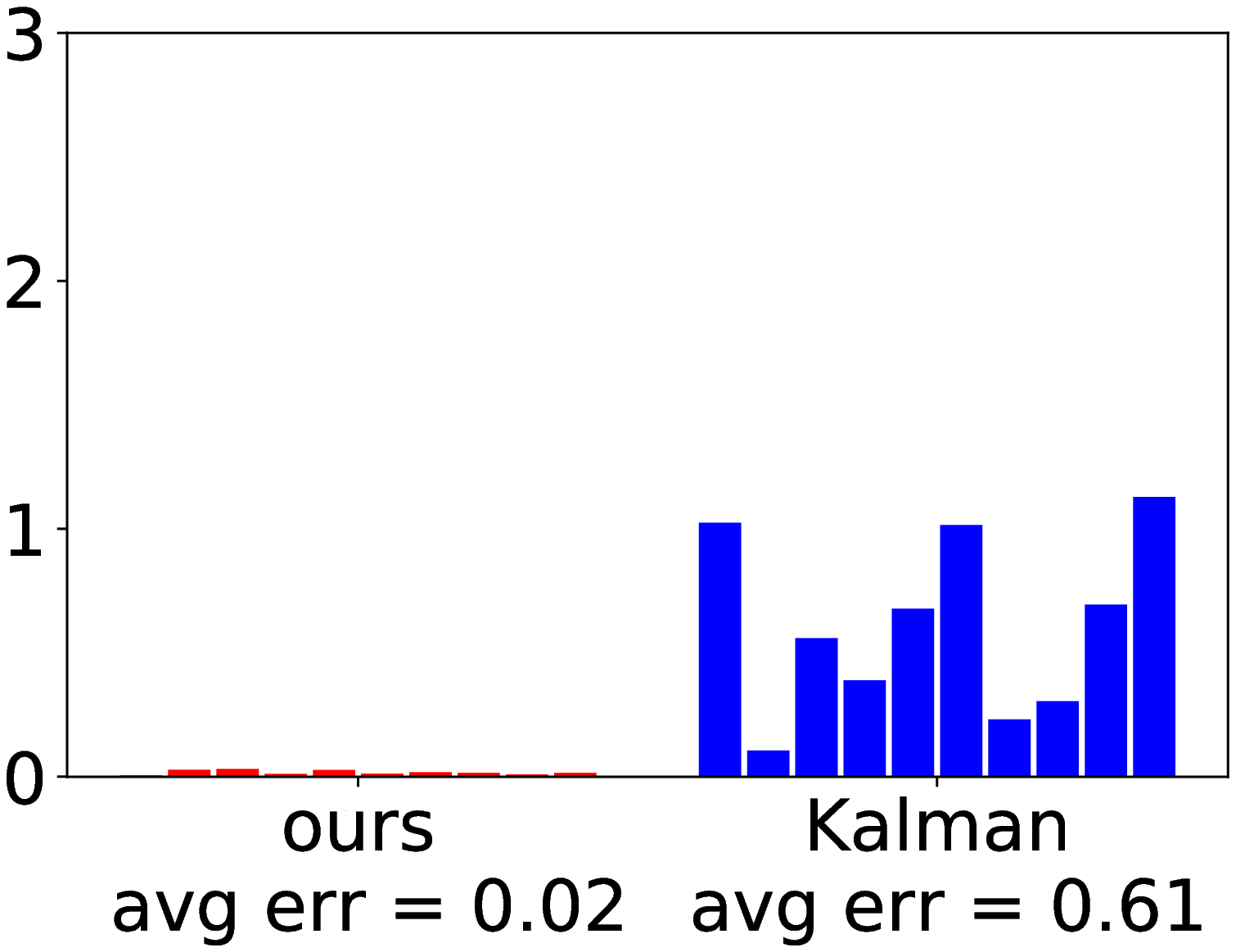}&
	 \includegraphics*[width=0.19\linewidth]{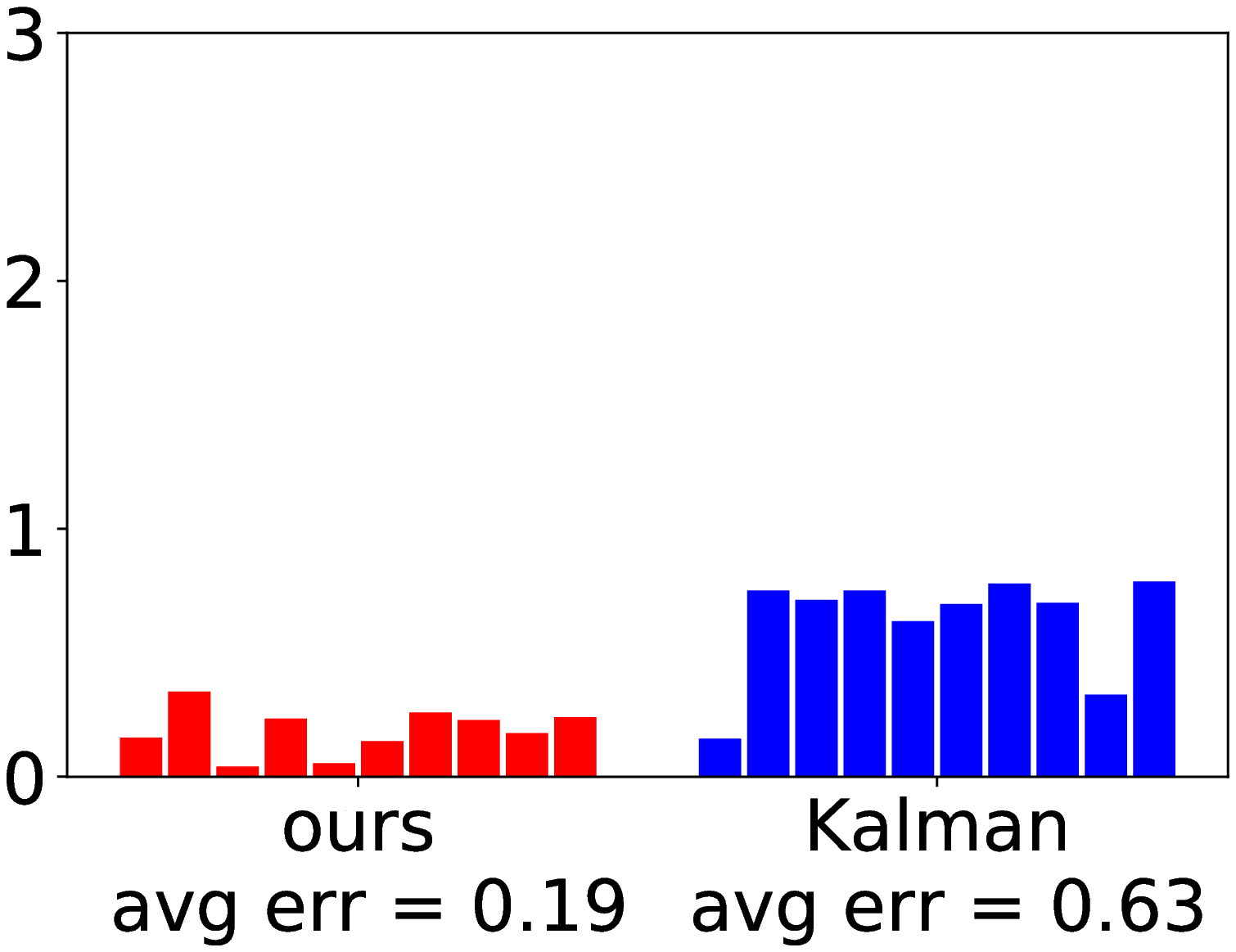}&
	 \includegraphics*[width=0.19\linewidth]{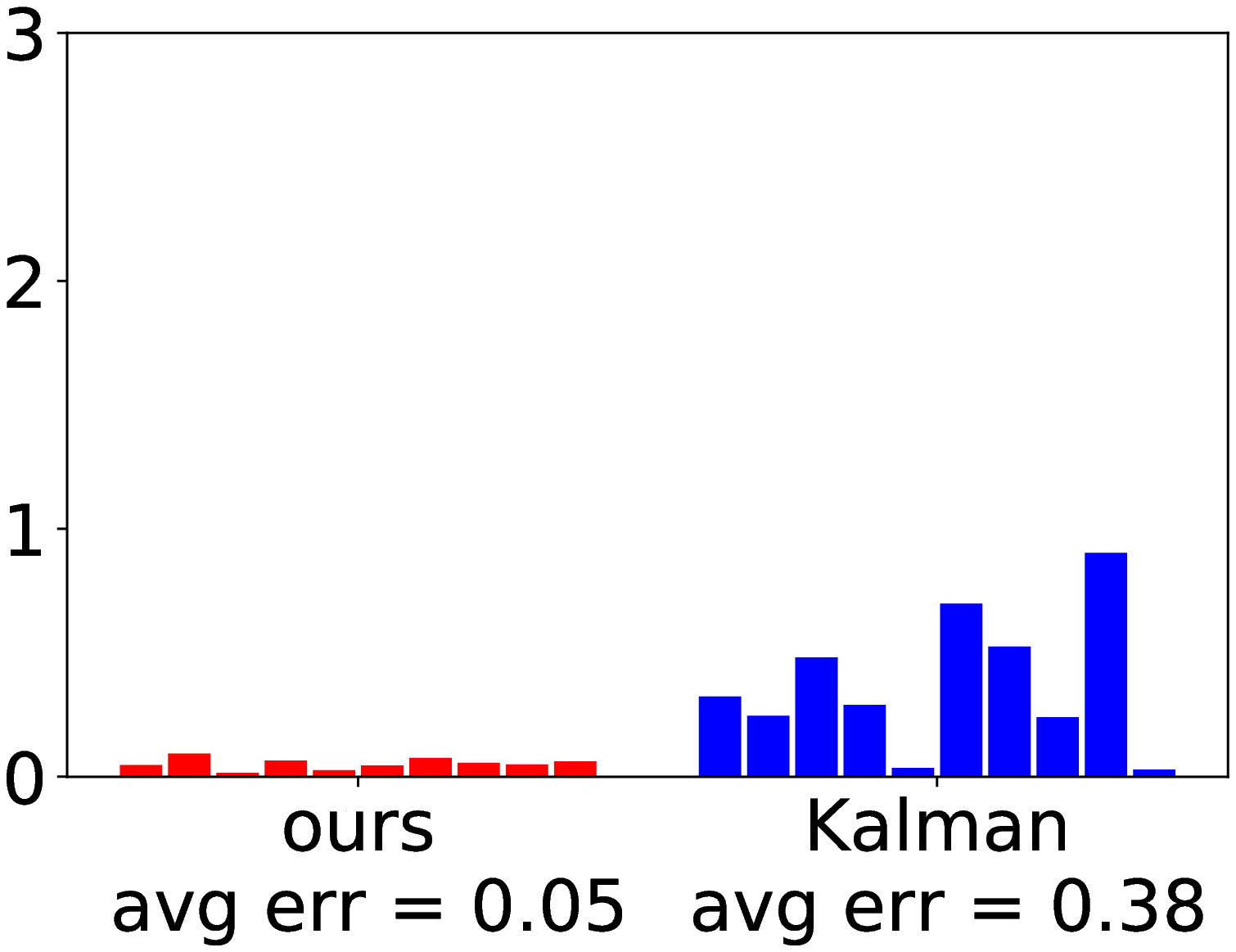}
  \end{tabular}
    \vspace{-2ex}
  \caption{We generate $10$ sets of observations for the Lotka--Volterra model and report the error for each of the experiments. The average error has been reported below each plot. \emph{First row:} comparison with the mean-field method of \citet{Gorbach17}, where the noisy observations have the variance $\sigma^2=1$. \emph{Second and third rows:} comparison with the extended Kalman filter (EKF) where the noisy observations have the variance $\sigma^2=0.1$ and $1.5$, respectively.The number of observation is $T=10\,000$.}
  \vspace{-3ex}
\label{f:comp1}
\end{figure*}

Fig.~\ref{f:comp} shows that almost all the methods work well when the initialization is close to the true parameters (small noise). In reality, we do not know what the real parameters are; it is reasonable to say that the last column of Fig.~\ref{f:comp} (initialization with the largest noise) determines which method performs better in real-world applications. BCD-prox outperforms other methods in both prediction and parameter error.

In our second experiment, we compare BCD-prox with the mean-field variational Bayes method of \citet{Gorbach17} on the Lotka--Volterra model. The mean-field method is only applicable to  differential equations with a specific form---see Eq.\ $(10)$ in \citet{Gorbach17}. While we cannot apply the mean-field method to the FitzHugh--Nagumo and R\" ossler models, we can apply it to the Lotka--Volterra model. In Fig.~\ref{f:comp1}, we compare the methods by prediction and parameter errors.

Fig.~\ref{f:comp1} (first row) shows the results for $\sigma^2=1$ (results for other variances are in supplementary material). Similar to our previous experiments, we generate $10$ sets of noisy observations and each bar corresponds to the error for one of the methods in one of the experiments.

Fig.~\ref{f:comp1} shows that the average error of BCD-prox is less than that of the mean-field method in almost all cases. The average parameter error of the mean-field method for $\theta_0$ and $\theta_1$ becomes around $3$ and $8$, respectively, but the average error of BCD-prox for both parameters remains less than $1$. The results in the supplementary material show that as we increase the noise in the observations, the error of both methods increases. Still, BCD-prox is more robust to observational noise than the mean-field method.

\vspace{-2ex}
\paragraph{Comparison with extended Kalman filter (EKF).} We follow \citet{Sitz02} to apply  EKF to our problem.  We use an open-source Python code \citep{Labbe14} in our implementation. For details, see the supplementary material.

In the second and third rows of Fig.~\ref{f:comp1}, we compare BCD-prox with EKF on the Lotka--Volterra model. We compare the methods in different settings by changing the amount of noise and the number of samples. The noise variances are $\sigma^2=0.1$ and $\sigma^2=1.5$ and the number of samples are $T = 20$ (time range $[0,2]$) and $T = 10\,000$ (time range $[0,1\,000]$). The results for $T=20$ can be found in the supplementary material.

In Fig.~\ref{f:comp1} we report the average estimation error instead of the prediction error. Estimation error is defined as the difference between the clean states and the estimated states $\X^*$. We report the estimation error because the prediction error of EKF goes to infinity. To see why this happens, note that to obtain reasonable predictions we need good estimations of the parameters and the initial state. Since EKF is an online method, it never updates the initial state. Given that the initial state is noisy, no matter how well parameters are estimated, the prediction error becomes very large. BCD-prox updates the initial state, yielding small prediction error.

We found that the only setting in which EKF performs comparably to BCD-prox is the case of $T=10\,000$ and $\sigma^2=0.1$. In other words, EKF works fine when we have long time series with low noise.  In more realistic settings, BCD-prox significantly outperforms EKF.  A key difference between the two methods is that EKF is an online method while ours is a batch method, iterating over the entire data set repeatedly.  Consequently, BCD-prox updates parameters based on information in all the states, leading to more robust updates than is possible with EKF, which updates parameters based on a single observation.

We also see that the error of both methods becomes smaller as we increase the number of samples $T$. This is expected because increasing $T$ is equivalent to giving more information about the model to the methods. The average estimation error of BCD-prox becomes almost $0$ for large $T$.

\paragraph{Conclusion.} BCD-prox addresses issues of previous approaches to simultaneous parameter estimation and filtering, achieving fast training and robustness to noise, initialization, and hyperparameter tuning. We have shown how to use BCD-prox with multistep ODE integration methods.  Additional features of BCD-prox include its connection to BCD and proximal methods, its unified objective function, and a convergence theory resulting from blockwise  convexity.  In ongoing/future work, we seek to extend BCD-prox to estimate the vector field $\mathbf{f}$ from noisy observations.

\section*{Acknowledgements}
H. S. Bhat was partially supported by NSF award DMS-1723272.  Both authors acknowledge use of the MERCED computational cluster, funded by NSF award ACI-1429783.

\begin{center}
\Large \textbf{Supplementary Material}
\end{center}

\printAffiliationsAndNotice{} 

%

%

\begin{abstract}
This supplementary material contains the following: 1) The equations and ground truth parameters for the ODEs that we used in the experiments, 2) Extensions of our experiments with different types and magnitudes of the noise, and 3) An animation that shows how our method works and how the estimated and predicted states move closer to each other at each iteration.
\end{abstract}

\section{ODEs in our experiments}
We have used four benchmark datasets in our experiments. We only gave a brief explanation of each of them in the paper. Here, we introduce them in detail.

\paragraph{Lotka--Volterra model.} This model is used to study the interaction between predator (variable $x_0$) and prey (variable $x_1$) in biology \citep{Lotka32}. The model contains two nonlinear equations as follows:

\begin{equation*}
	\label{eq:lot-vol}
	\frac{dx_0}{dt} = \theta_0 x_0 - \theta_1 x_0 x_1 \qquad \frac{dx_1}{dt} = \theta_2 x_0 x_1 - \theta_3 x_1.
\end{equation*}

The state is two-dimensional and there are four unknown parameters. We use the same settings as in \citet{Dondelin13}. We set the parameters to $\theta_0=2$, $\theta_1=1$, $\theta_2=4$ and $\theta_3=1$. With initial condition $\x_{(1)}=[5,3]$, we generate clean states in the time range of $[0,2]$ with a spacing of $\Delta t = 0.1$.

\paragraph{FitzHugh--Nagumo model.} This model describes spike generation in squid giant axons \citep{FitzHugh61,Nagumo62}. It has two nonlinear equations:
\begin{align*}
	\frac{dx_0}{dt} &= \theta_2 \left(x_0 - \frac{(x_0)^3}{3} + x_1 \right)\\
	 \frac{dx_1}{dt} &= - \frac{1}{\theta_2} (x_0 - \theta_0 + \theta_1 x_1),	
\end{align*}
where $x_0$ is the voltage across an axon and $x_1$ is the outward current. The states are two-dimensional and there are three unknown parameters. We use the same settings as in \citet{Ramsay07}. We set the parameters as $\theta_0=0.5, \theta_1=0.2,$ and $\theta_3=3$. With initial condition $\x_{(1)}=[-1,1]$, we generate clean states in the time range of $[0,20]$ with a spacing of $\Delta t = 0.05$.

\paragraph{R\" ossler attractor.} This three-dimensional nonlinear system has a chaotic attractor \citep{Rossler76}:
\begin{align*}
	\frac{dx_0}{dt} &= -x_1 - x_2 \\
	\frac{dx_1}{dt} &= x_0 + \theta_0 x_1 \\
	\frac{dx_2}{dt} &= \theta_1 + x_2(x_0 - \theta_2).
\end{align*}
	The states are three-dimensional and there are three unknown parameters. We use the same settings as in \citet{Ramsay07}. We set the parameters as $\theta_0=0.2, \theta_1=0.2,$ and $\theta_3=3$. With the initial condition $\x_{(1)}=[1.13,-1.74,0.02]$, we generate clean states in the time range of $[0,20]$, with $\Delta t = 0.05$.

\paragraph{Lorenz-96 model.} The goal of this model is to study weather predictability \citep{Lorenz98}. For,  $k=0,\ldots,d-1$, the $k$th differential equation has the following form:
\begin{equation*}
\label{eq:lorenz}
	\frac{dx_k}{dt} = (x_{k+1} - x_{k-2})(x_{k-1}) - x_k + \theta_0,
\end{equation*}
The model has one parameter $\theta_0$ and $d$ states, where $d$ can be set by the user. This gives us the opportunity to test our method on larger ODEs. Note that to make this ODE meaningful, we have $x_{-1} = x_{d-1}$, $x_{-2} = x_{d-2}$, and $x_d = x_0$. As suggested by \citet{Lorenz98}, we set $d=40$ and $\theta_0=8$. The clean states are generated in the time range $[0,4]$ with a spacing of $\Delta_i=0.01$. The initial state is generated randomly from a Gaussian distribution with mean $0$ and variance $1$.

\section{Experimental results}

\paragraph{Optimization of our objective function leads to better estimation.} In Fig.\ 2 of our main paper, we reported the prediction error at each iteration of our algorithm for the R\" ossler and the Lorenz-96 models. Here, in Fig.~\ref{f:objerrSUPP}, we add the FitzHugh--Nagumo model and show the results for noisy observations with $\sigma^2=0.5$ and $\sigma^2=1$.

In all settings, our method decreases the error significantly for both Euler and three-step Adams-Bashforth methods. The three-step method performs better than the Euler method, specifically in the Lorenz-96 model.

\paragraph{Different types and amounts of noise in the observations.} Our method does not assume anything about the type of noise. In reality, the noise could be from any distribution. In Fig.~\ref{f:noise}, we investigate the effect of the type of noise on the outcome of our algorithm. The red (blue) curves correspond to the case when we add Gaussian (Laplacian) noise to the observations. We set the mean to $0$, change the variance of the noise, and report the prediction and parameter errors. Note that for each noise variance, we repeat the experiment $10$ times and report the mean and standard deviation of the error.

In general, increasing the noise variance increases the error. We can see this in almost all plots. In both models, the error does not change much by changing the variance from $0$ to $0.5$. We can also see that the method performs almost as well for observations corrupted by Laplacian noise as in the Gaussian noise case.  Note that the Laplacian noise has a heavier-than-Gaussian tail.

\paragraph{Comparison with other methods (robustness to initialization).} In Fig.\ 4 of the paper, we compared our method with three other methods in different categories on the R\" ossler model. Fig~\ref{f:compSUPP} shows the comparison on the FitzHugh--Nagumo model. In both models, our method is robust with respect to the initialization and outperforms other methods significantly.

\paragraph{Comparison with the mean-field method \citep{Gorbach17}.}In Fig.\ 5 of the paper, we compared our method with the mean-field method of \citet{Gorbach17} on the Lotka--Volterra model, with noise variance $\sigma^2=1$. Fig.~\ref{f:compSUPP1} compares the methods for $\sigma^2=.5, 1,$ and $1.5$. Our method is more robust with respect to noise and performs better.

\paragraph{Comparison with the extended Kalman filter (EKF).} 
As we mentioned in the main paper, we follow \citet{Sitz02} in applying the Kalman filter to our problem of estimating the parameters and states. Here, we provide more information regarding our implementation. 

We first need to write an equation that recursively finds the state $\x_{(t_{i+1})}$ in terms of $\x_{(t_{i})}$. As suggested by \citet{Sitz02}, this can be achieved by discretizing the ODE using the Euler discretization:
\begin{equation}
	\x_{(t_{i+1})} = \x_{(t_i)}  + \f(\x_{(t_i)},\btheta)\Delta_{i}.
\end{equation}
Let us define $\btheta_{(t_i)}$ as the parameter estimated at time $t_i$ by the Kalman filter. We define a joint state variable $\bxi_{(t_i)}$, which merges the states $\x_{(t_i)}$ and the parameters $\btheta_{(t_i)}$ as follows:
\begin{equation}
	\bxi_{(t_i)} = \begin{pmatrix}
	\x_{(t_i)} \\
	\btheta_{(t_i)}
	\end{pmatrix}, \quad \bxi_{(t_i)} \in \bbR^{d+p}.
\end{equation}
The process model to predict the next state variable can be written as:
\begin{equation}
	\bxi_{(t_{i+1})} = \begin{pmatrix}
	\x_{(t_{i+1})} \\
	\btheta_{(t_{i+1})}
	\end{pmatrix} = \begin{pmatrix}
	\x_{(t_i)}  + \f(\x_{(t_i)},\btheta)\Delta_{i}\\
	\btheta_{(t_{i})}
	\end{pmatrix}.
\end{equation}
We define the observation model as follows:
\begin{equation}
	\y_{(t_{i})} = \HH \bxi_{(t_{i})}, \quad \HH=
	\begin{pmatrix}
	\I & \0
	\end{pmatrix}_{d \times (d+p)},
\end{equation}
where $\HH$ is a $d \times (d+p)$ matrix, $\I$ is a $d \times d$ identity matrix, and $\0$ is a $d \times p$ matrix where all elements are $0$.

In most cases, the function $\f(\cdot)$ is nonlinear, which makes the process model nonlinear. For this reason, we use the extended Kalman filter (EKF), which linearizes the model.

We use an open-source Python code \cite{Labbe14} to implement  EKF. We set the state covariance (noise covariance) to a diagonal matrix with elements equal to $1\,000$ ($0.1$). We set the process covariance using the function Q\_discrete\_white\_noise() provided in \cite{Labbe14}, where the variance is set to $1$. Note that these parameters must be carefully tuned to obtain reasonable results; 
changing the state or noise covariance yields significantly worse results.

In Fig.\ 5 of the main paper, we compared our method with  EKF on the Lotka--Volterra model. In that experiment, we set the number of samples to $T = 10\,000$ (time range $[0,2]$). Here, we show the results for both $T = 20$ (time range $[0,2]$) and $T = 10\,000$ (time range $[0,1\,000]$).

As we can see in Fig.~\ref{f:kalman}, the only setting in which  EKF performs comparably to our method is the case of $T=10\,000$ and $\sigma^2=0.1$. In more realistic settings, our method significantly outperforms  EKF. 

\paragraph{Animation to show how our method works.} We consider the FitzHugh--Nagumo model, with settings as explained at the beginning of this section, except that we consider the first $10$ seconds instead of $20$. We add Gaussian noise with variance $0.5$ to the clean states to create the noisy observations. In Fig. 10 (see PDF file on second author's web site), we show how our algorithm works in the first $250$ iterations. Acrobat Reader is required to play the animation. In this animation, $\X$ denotes clean states (green circles), $\X^*$ denotes estimated states, and $\hat{\X}$ denotes predicted states. Note that initially, $\X^*$ is the same as the noisy observations. Fig. 10 (see PDF file on second author's web site) shows the two dimensions separately. At the top of each figure, we show the estimated parameters at each iteration. Note that the true parameters are $\theta_0=0.5, \theta_1=0.2$, and $\theta_2=3$. As explained before, the estimated and predicted states move closer to each other at each iteration. This helps the estimated parameters converge to the true parameters.

\begin{figure*}[!t]
  \centering
  \begin{tabular}{@{}c@{\hspace{1ex}}c@{\hspace{.5ex}}c@{\hspace{1ex}}c@{}}
  		& & \large{$\sigma^2=0.5$} & \large{$\sigma^2$=1}\\
		\rotatebox{90}{\hspace{1ex} \large{FitzHugh--Nagumo}} &       
      \rotatebox{90}{\hspace{7ex} \large{pred.\ error}} &       
    \includegraphics*[width=0.4\linewidth,height=.26\linewidth]{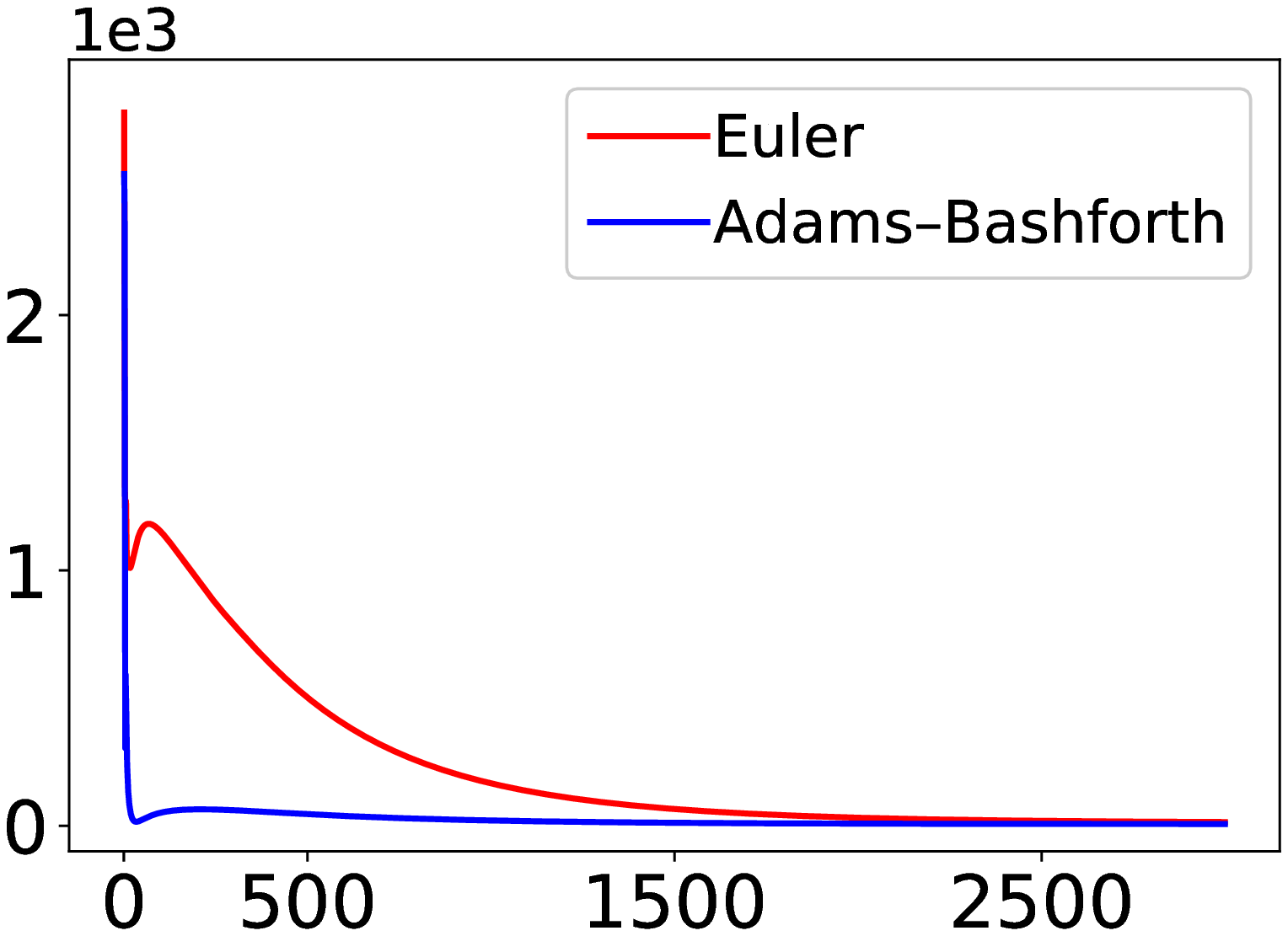}&
    \includegraphics*[width=0.435\linewidth,height=.26\linewidth]{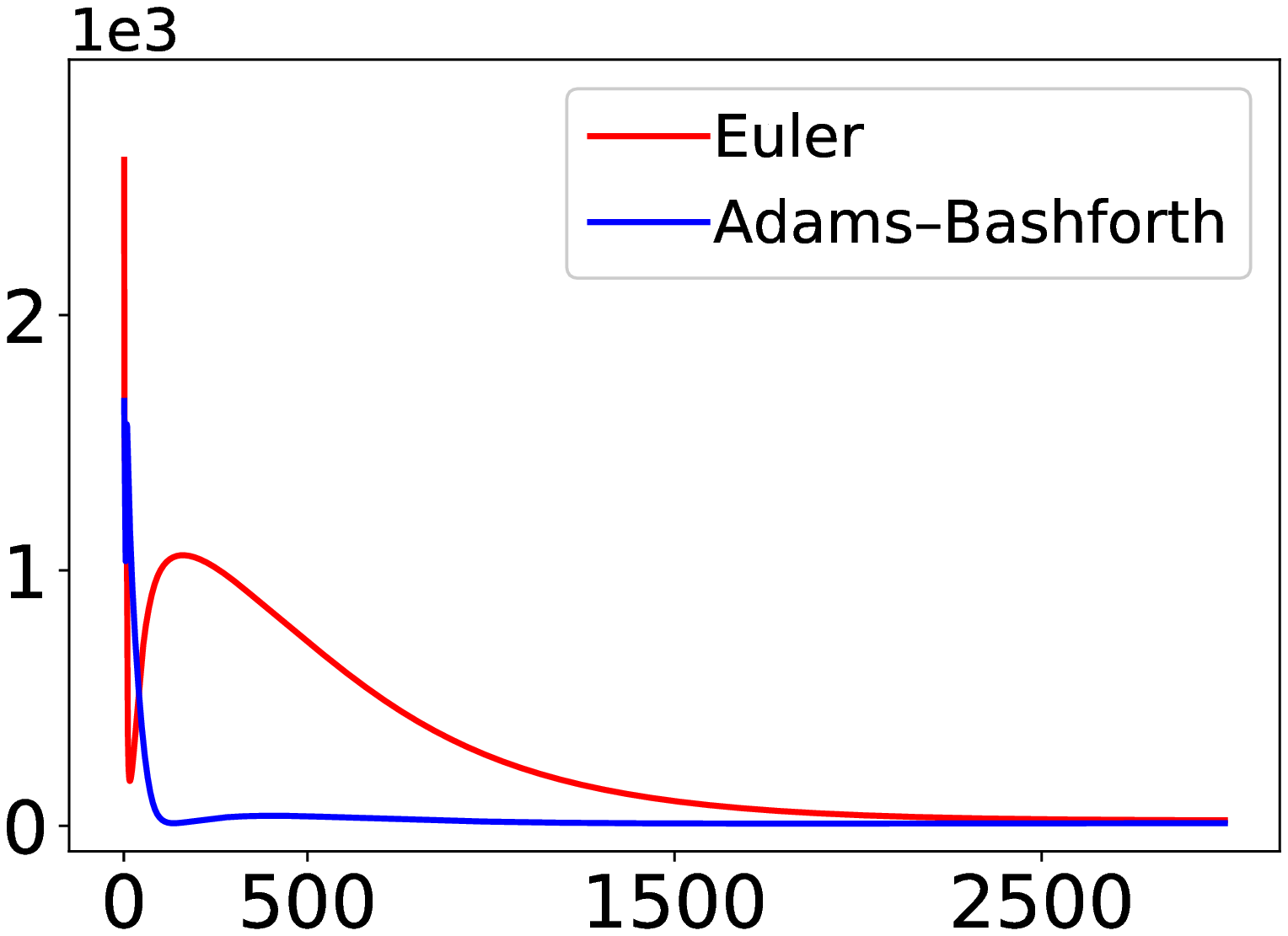}\\[.5ex]
    \rotatebox{90}{\hspace{3ex} \large{R\" ossler attractor}} &
     \rotatebox{90}{\hspace{7ex} \large{pred.\ error}} &       
    \includegraphics*[width=0.425\linewidth,height=.26\linewidth]{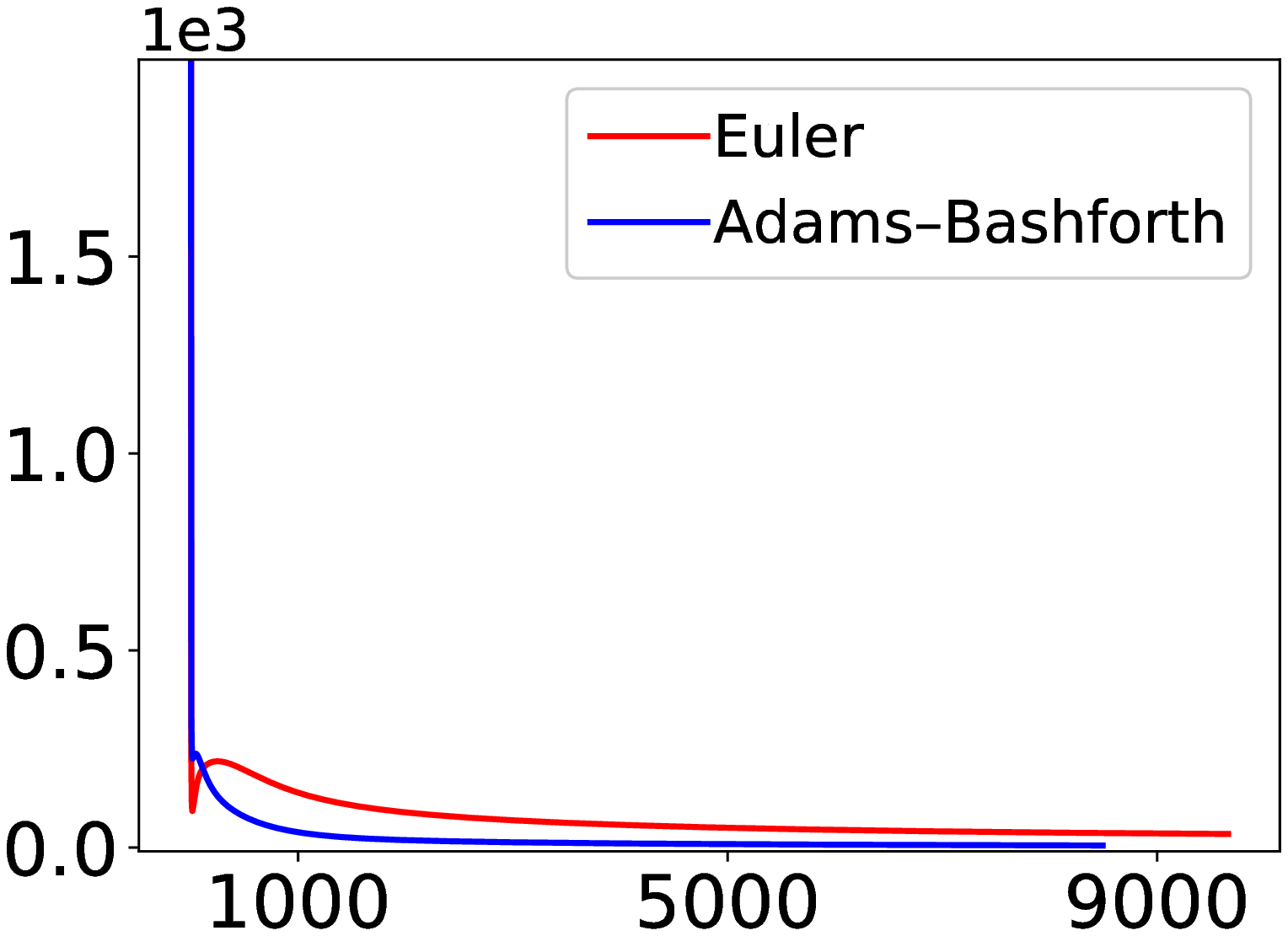}&
    \includegraphics*[width=0.435\linewidth,height=.26\linewidth]{objerr/rosslerSig1.eps}\\[.5ex]
    \rotatebox{90}{\hspace{8ex} \large{Lorenz-96}} &
     \rotatebox{90}{\hspace{7ex} \large{pred.\ error}} &       
    \includegraphics*[width=0.43\linewidth,height=0.29\linewidth]{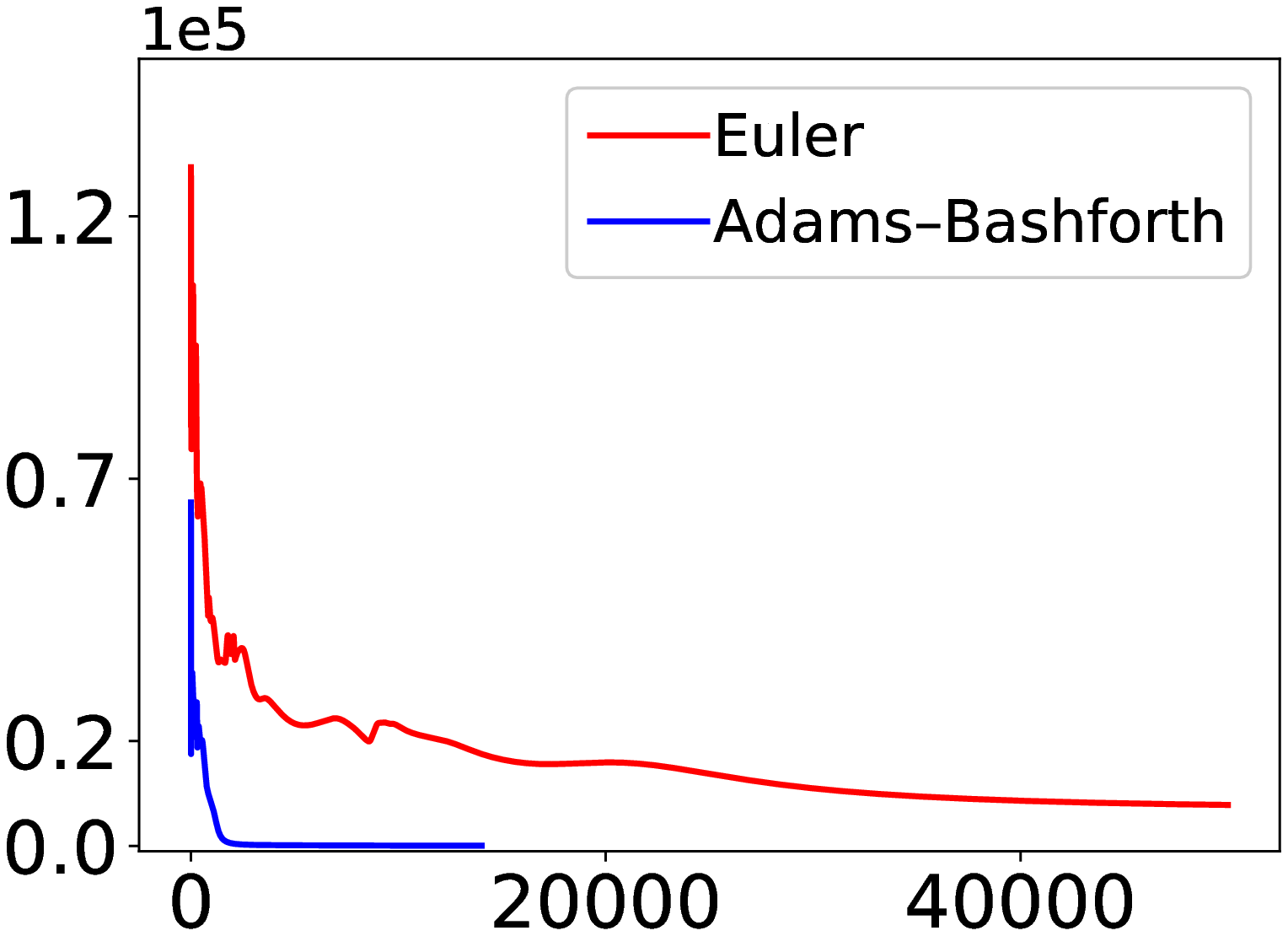}&
\includegraphics*[width=0.43\linewidth,height=0.29\linewidth]{objerr/lorenz96Sig1.eps}\\[-1ex]
    & & \large{iteration} & \large{iteration} \\[-1ex]
  \end{tabular}
  \caption{Similar to Fig.\ 2 of the main paper. We added FitzHugh--Nagumo and noisy observations with $\sigma^2=0.5$. Our learning strategy decreases the error in both cases and in all models.}
  
  \label{f:objerrSUPP}
\end{figure*}

\begin{figure*}[!t]
  \centering
  \begin{tabular}{@{}c@{\hspace{1ex}}c@{\hspace{4ex}}c@{\hspace{1ex}}c@{}}
    &\large{FitzHugh--Nagumo} & &  \large{R\" ossler} \\
    \rotatebox{90}{\hspace{6ex}  \large{pred.\ error} } &       
    \includegraphics*[width=0.41\linewidth]{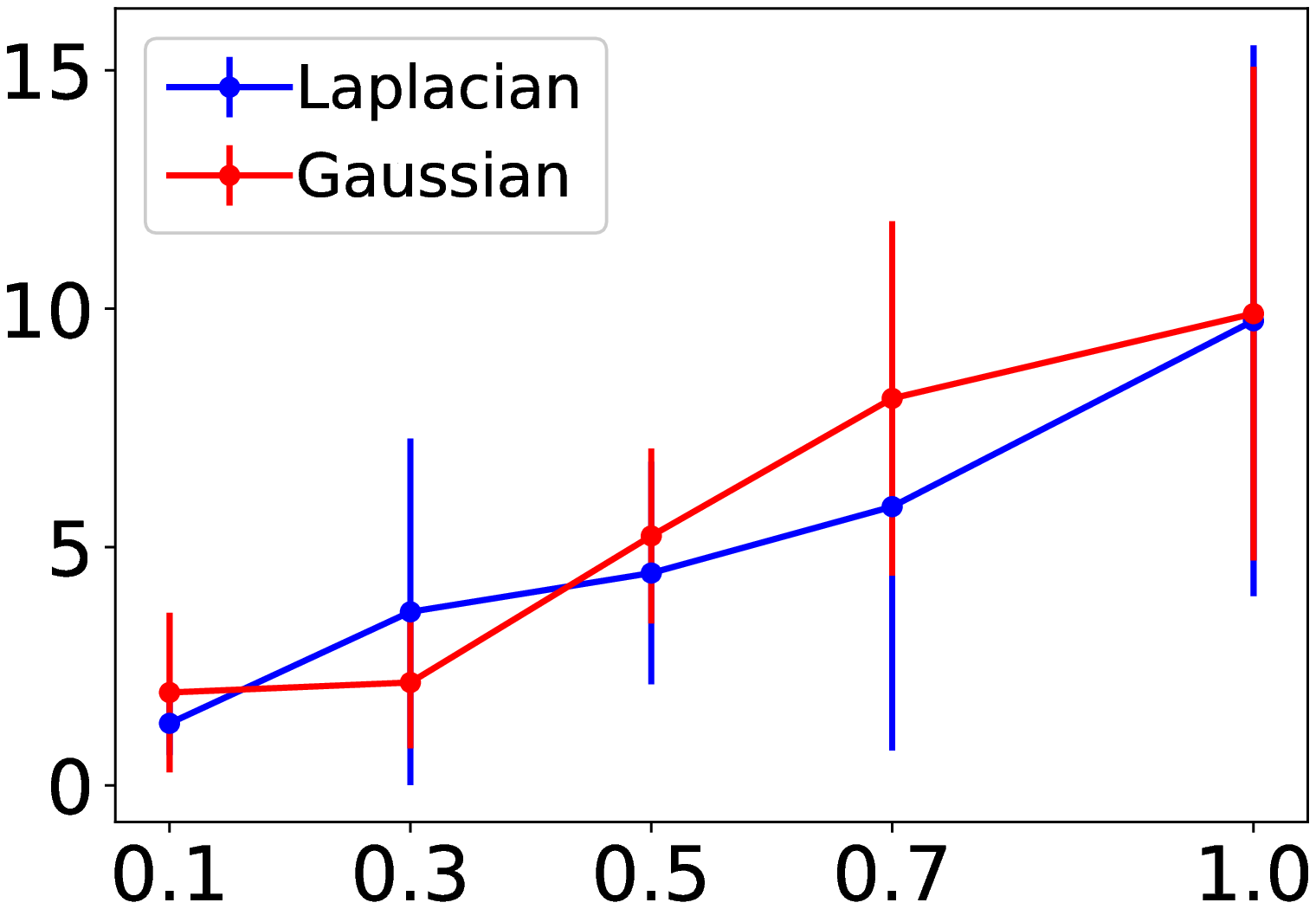}&
    \rotatebox{90}{\hspace{6ex} \large{pred.\ error} } &       
    \includegraphics*[width=0.41\linewidth]{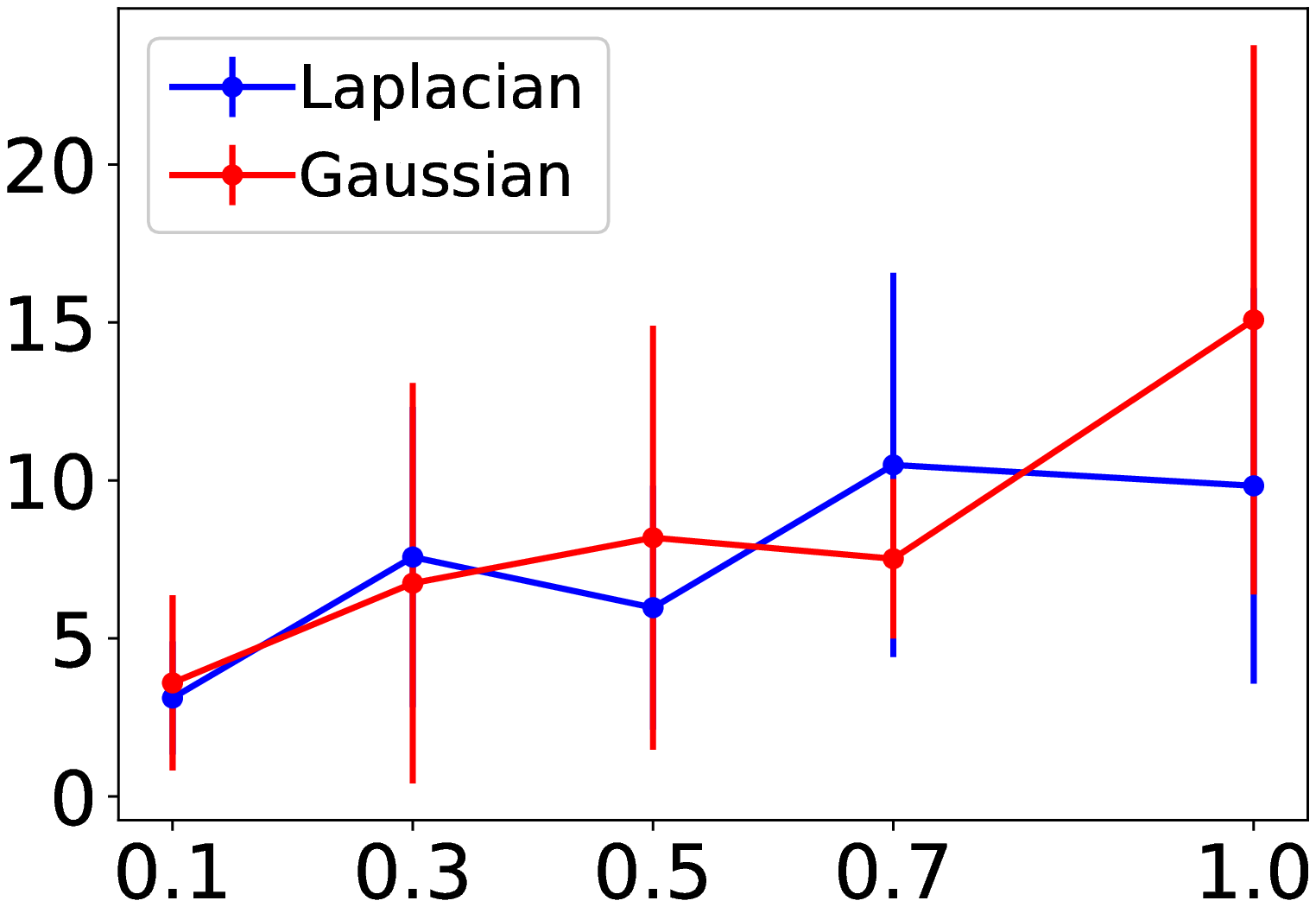}\\
    \rotatebox{90}{\hspace{7ex}  \large{$\theta_0$ error} } &  
    \includegraphics*[width=0.43\linewidth]{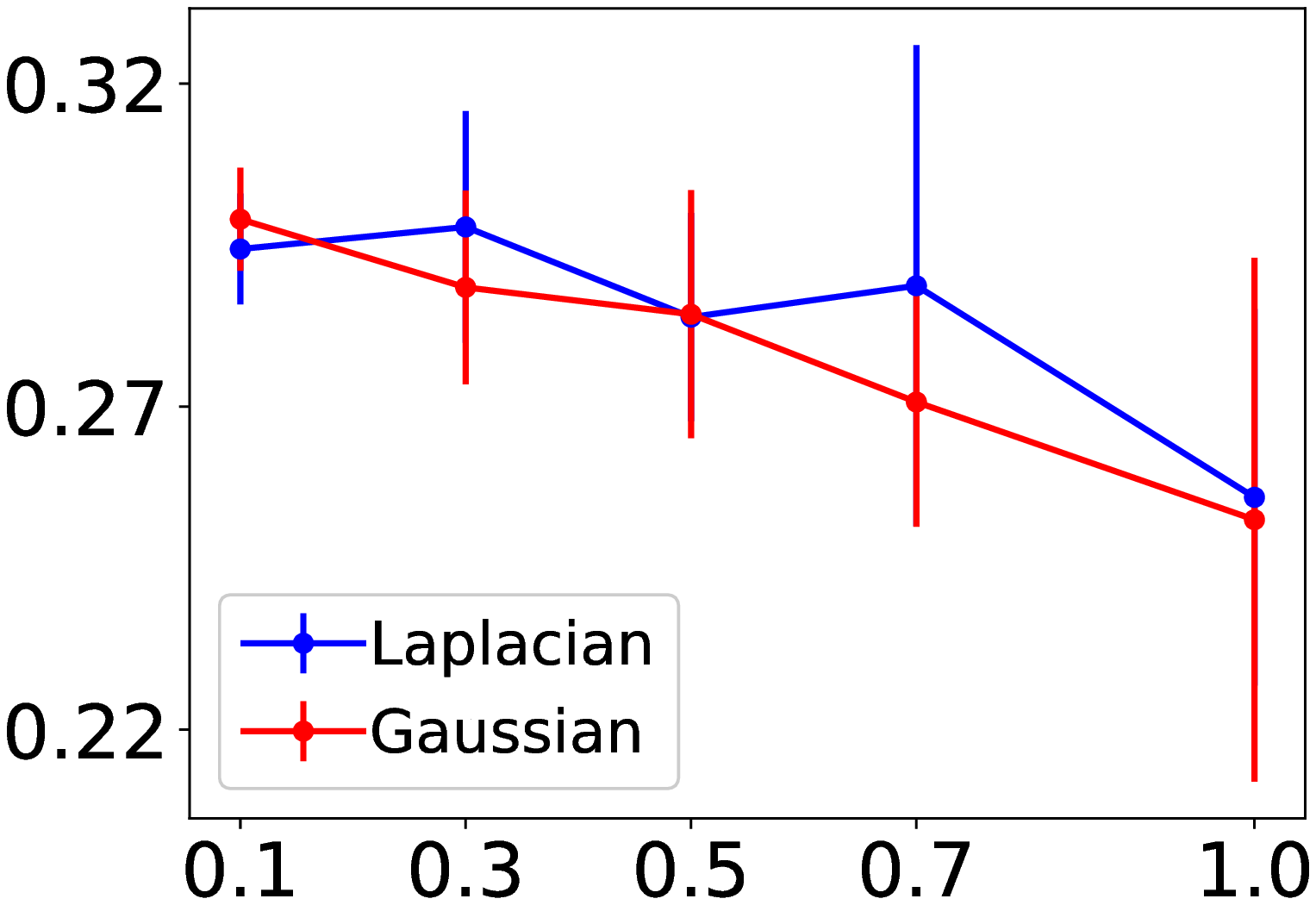}&
    \rotatebox{90}{\hspace{7ex}  \large{$\theta_0$ error} } &  
	 \includegraphics*[width=0.43\linewidth]{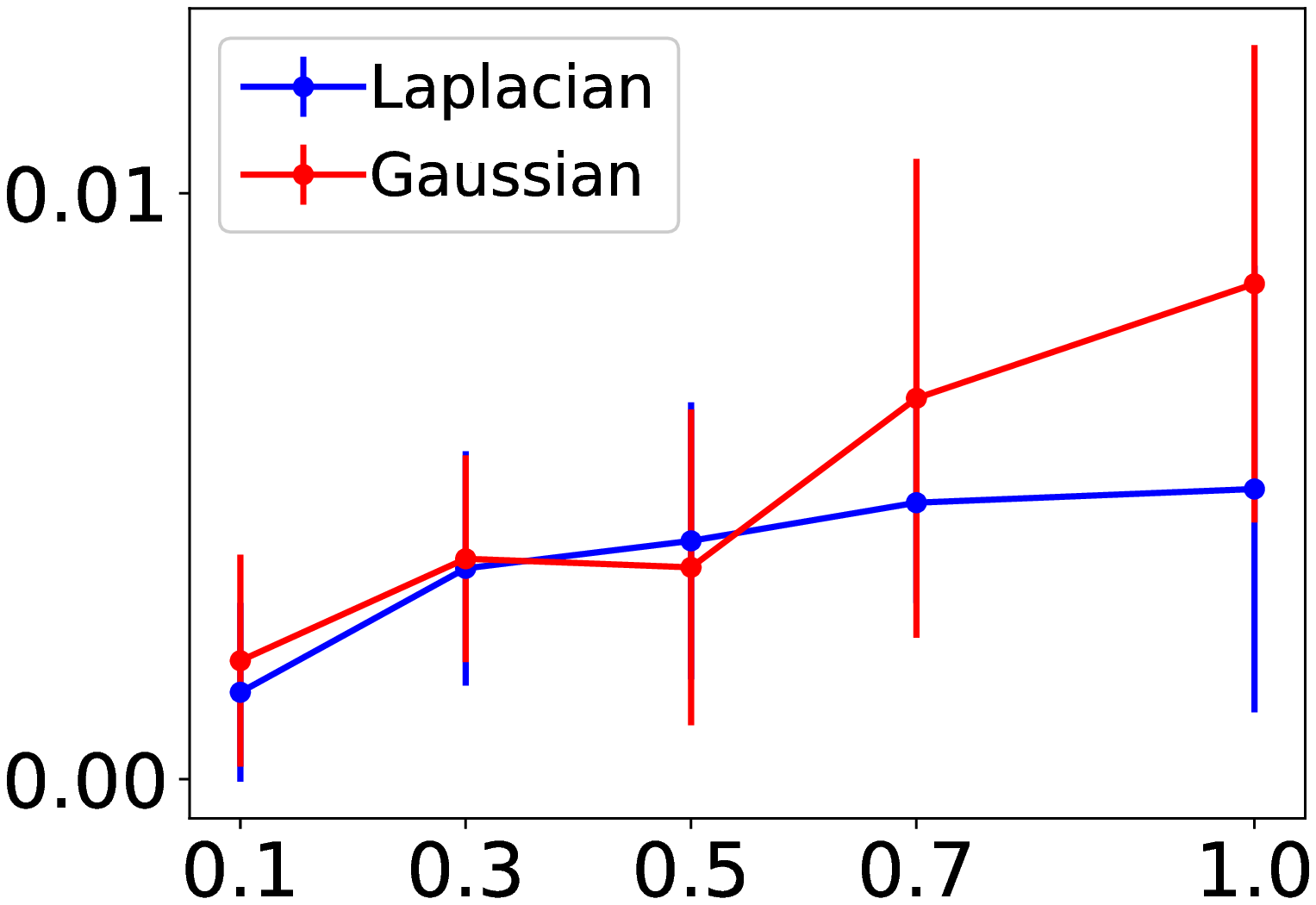}\\	 
	  \rotatebox{90}{\hspace{7ex}  \large{$\theta_1$ error} } &  
    \includegraphics*[width=0.43\linewidth]{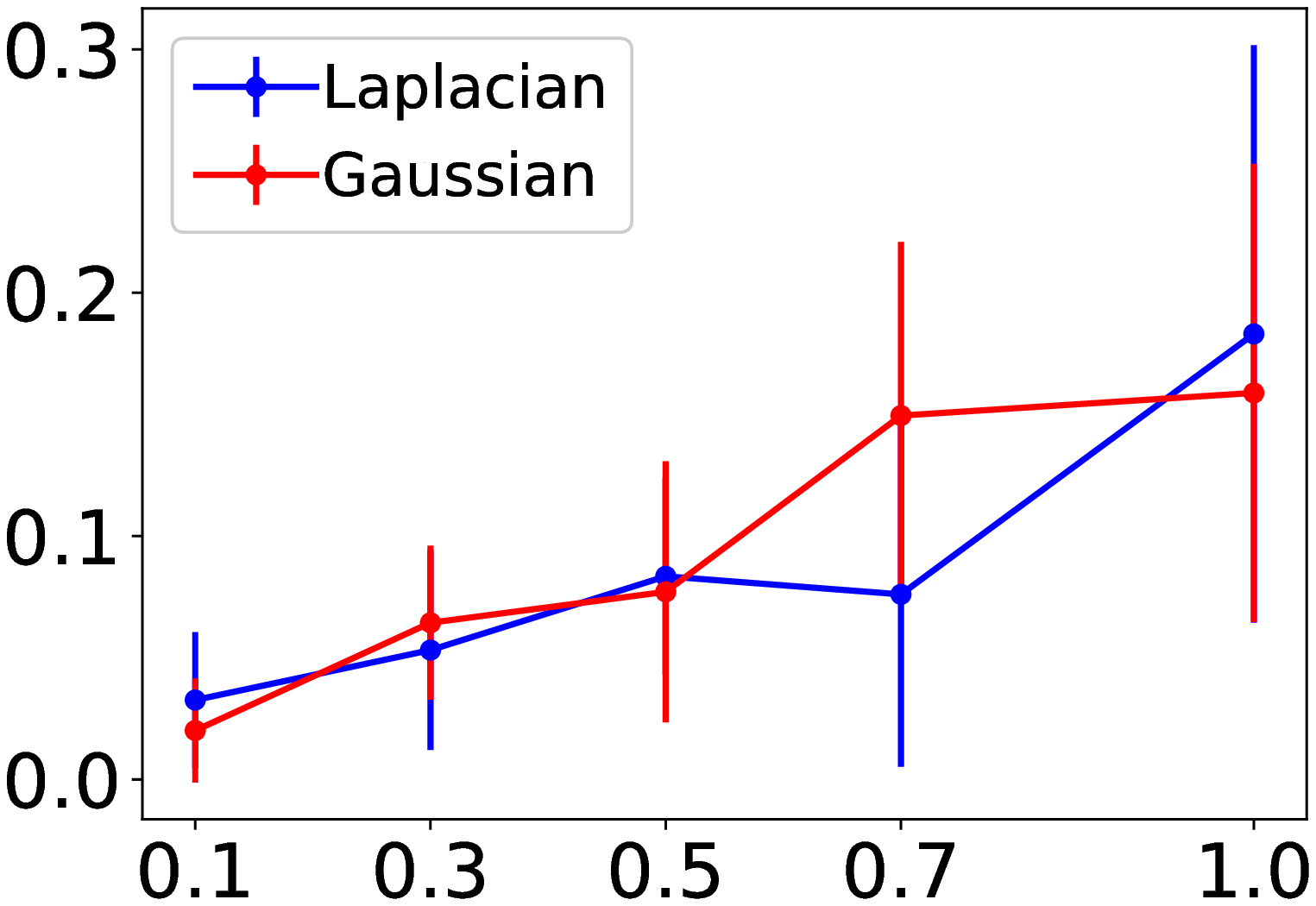}&
    \rotatebox{90}{\hspace{7ex}  \large{$\theta_1$ error} } &  
	 \includegraphics*[width=0.43\linewidth]{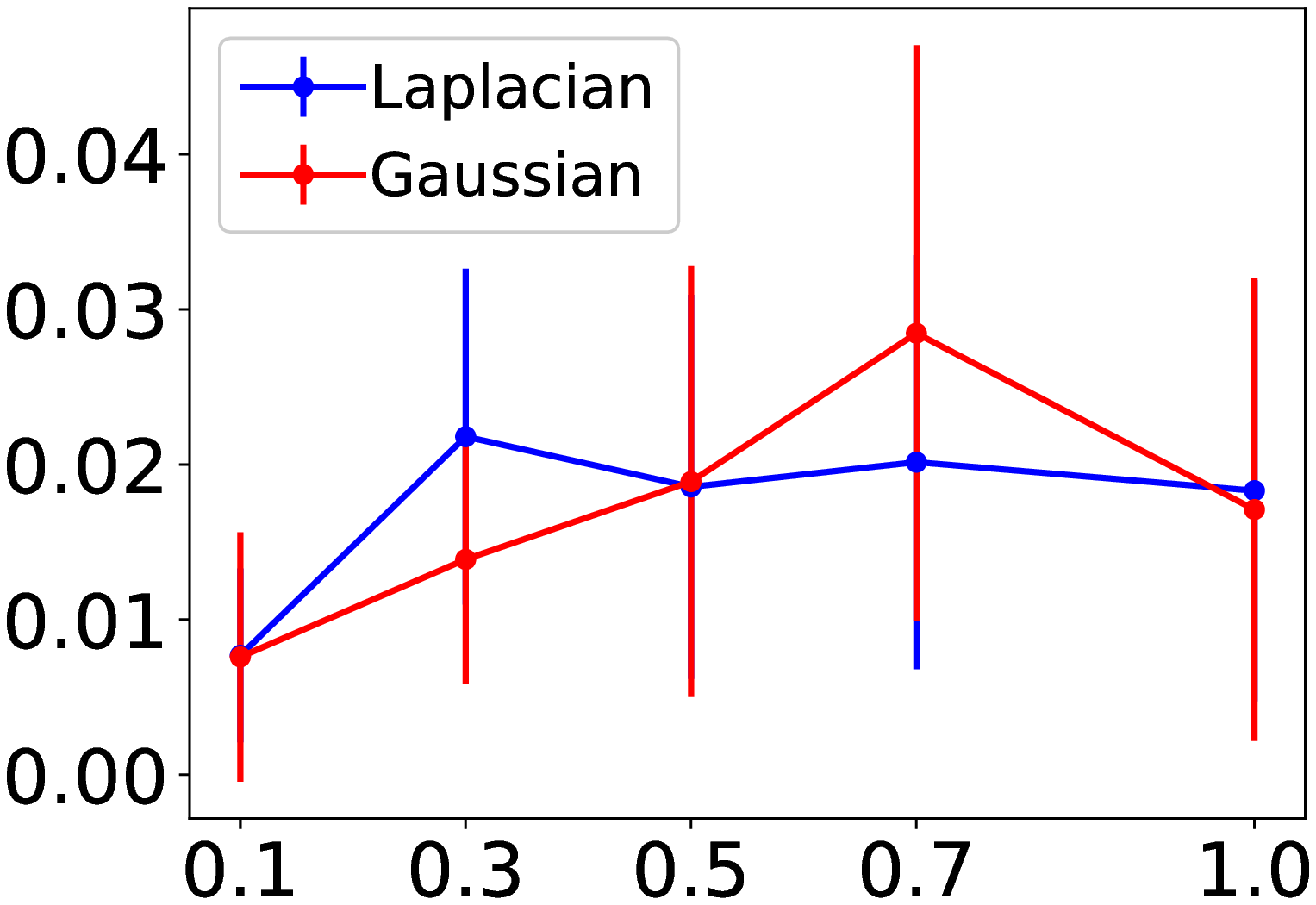}\\	 
	  \rotatebox{90}{\hspace{7ex}  \large{$\theta_2$ error} } &  
    \includegraphics*[width=0.43\linewidth]{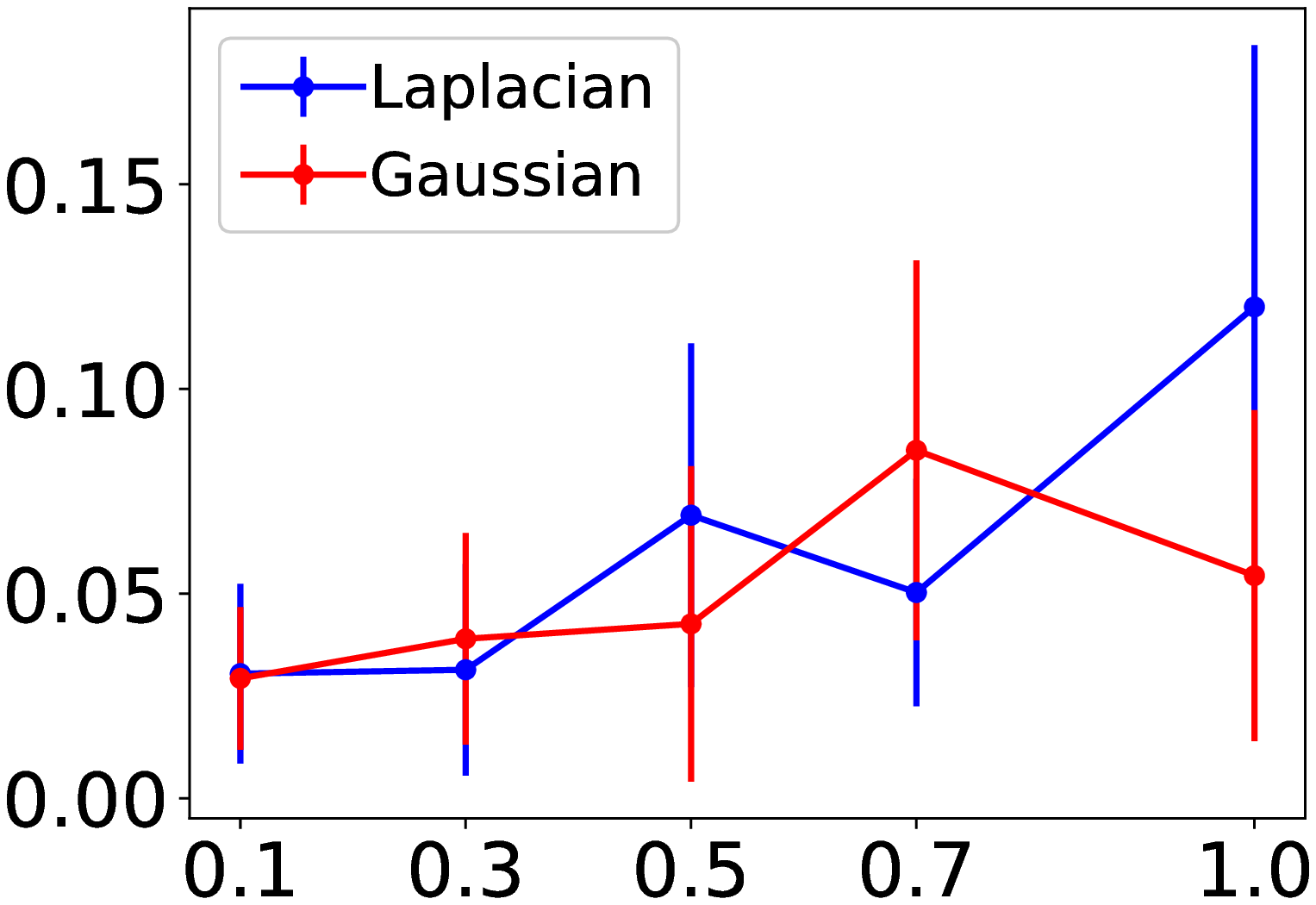}&
    \rotatebox{90}{\hspace{7ex}  \large{$\theta_2$ error} } &  
	 \includegraphics*[width=0.43\linewidth]{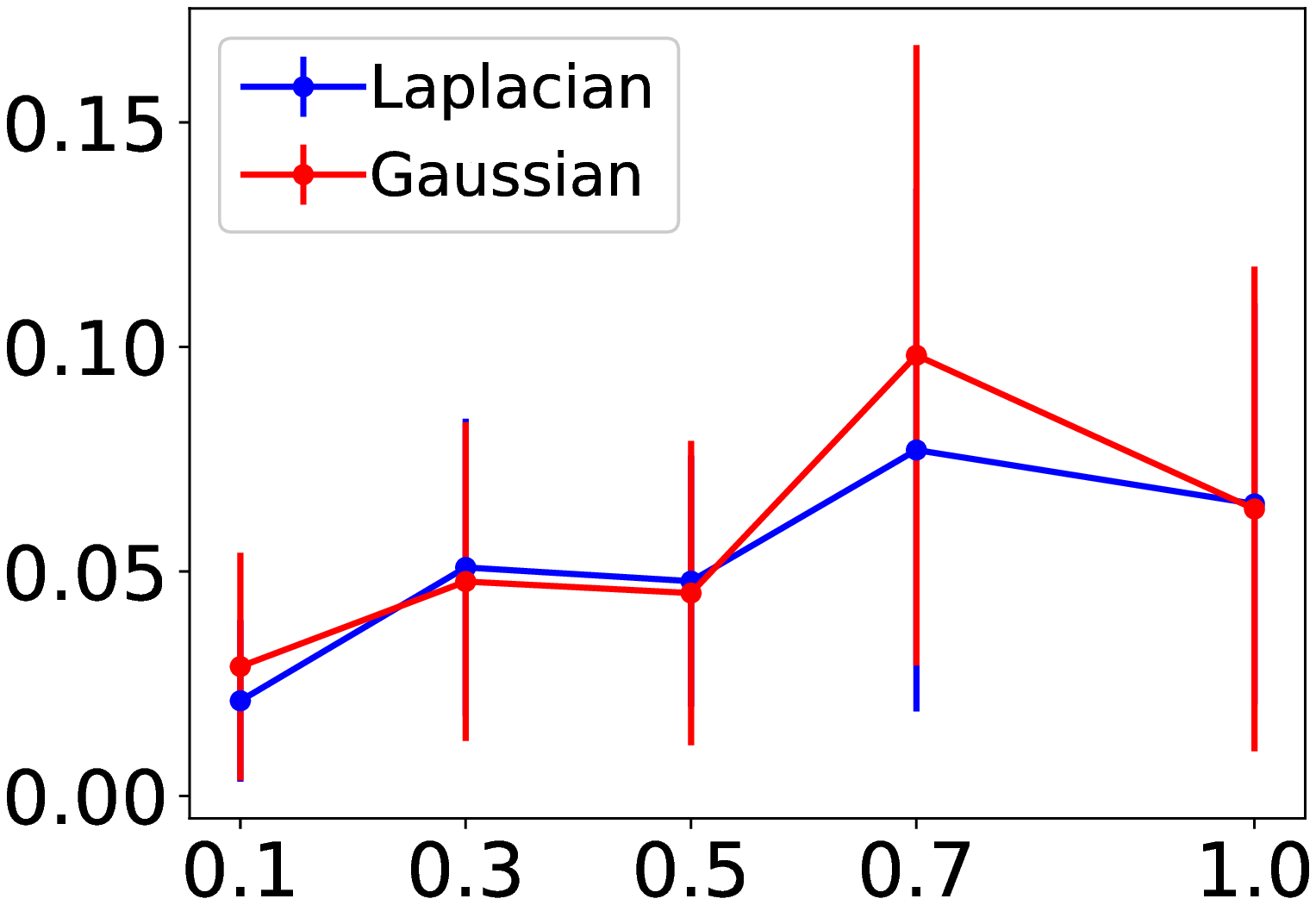}\\	 	 
     & \large{noise variance} & & \large{noise variance} \\[-1ex]
  \end{tabular}
  \caption{We change the amount and type of noise in the observations, and report the prediction and parameter errors on the FitzHugh--Nagumo (first column) and R\" ossler (second column) models.}

  \label{f:noise}
\end{figure*}

\begin{figure*}[!h]
  \centering
  \psfrag{i}[t][t]{it}
  \psfrag{idx}[t][t]{index of the hash fcn}
  \begin{tabular}{@{}c@{}c@{\hspace{.2ex}}c@{\hspace{.2ex}}c@{\hspace{.2ex}}c@{}}
  	& $\sigma_{\btheta}^2=1$ & $\sigma_{\btheta}^2=5$ & $\sigma_{\btheta}^2=10$ & $\sigma_{\btheta}^2=20$\\
	  \hspace{3ex}\rotatebox{90}{\hspace{2ex}\raisebox{3ex}[0pt][0pt]{\makebox[0pt][c]		{\hspace{-30ex}\makebox[21.5em][c]{\dotfill \small{FitzHugh--Nagumo}\dotfill}}} \small{state error}} &              
    \includegraphics*[width=0.235\linewidth,height=0.165\linewidth]{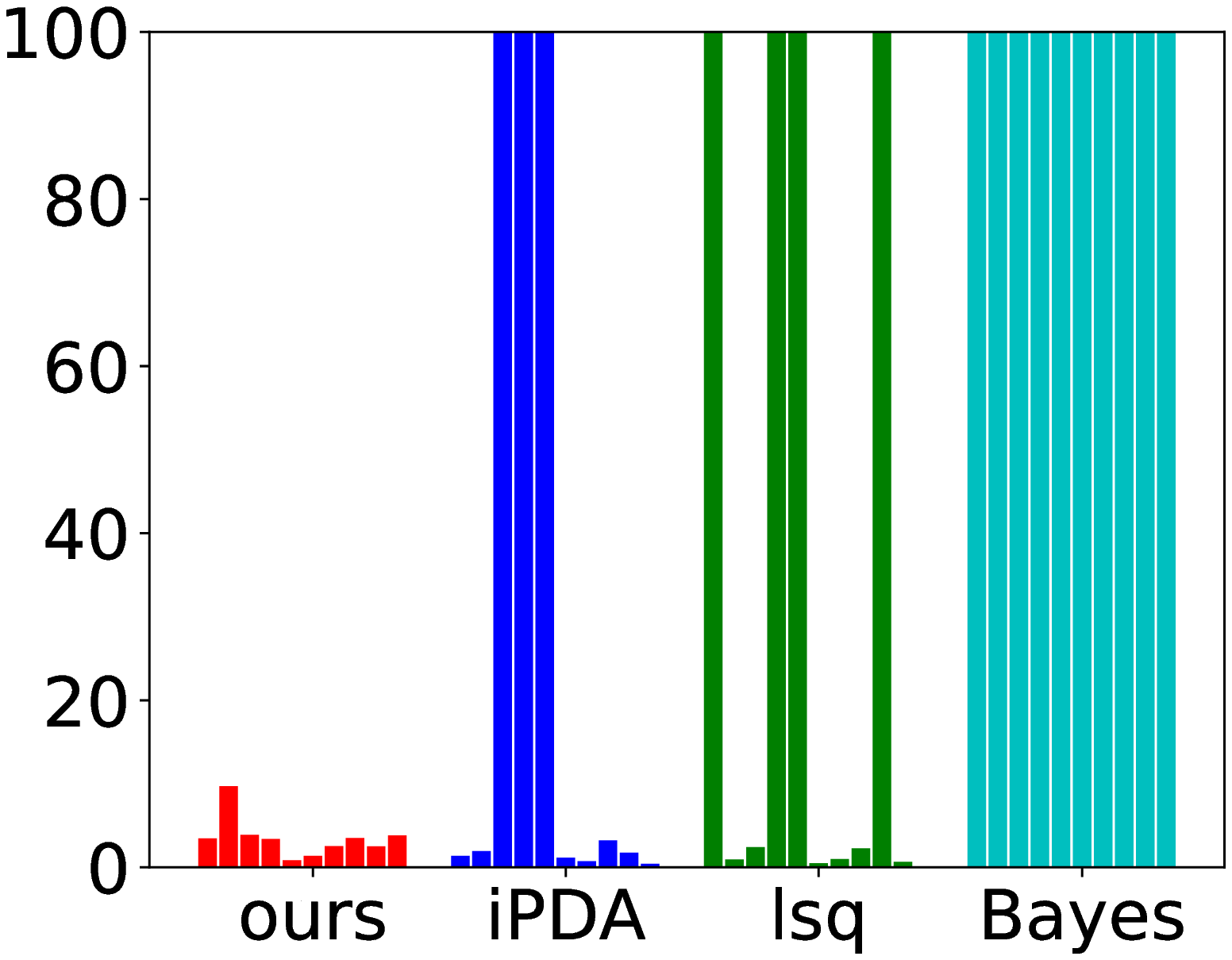}&
    \includegraphics*[width=0.235\linewidth,height=0.165\linewidth]{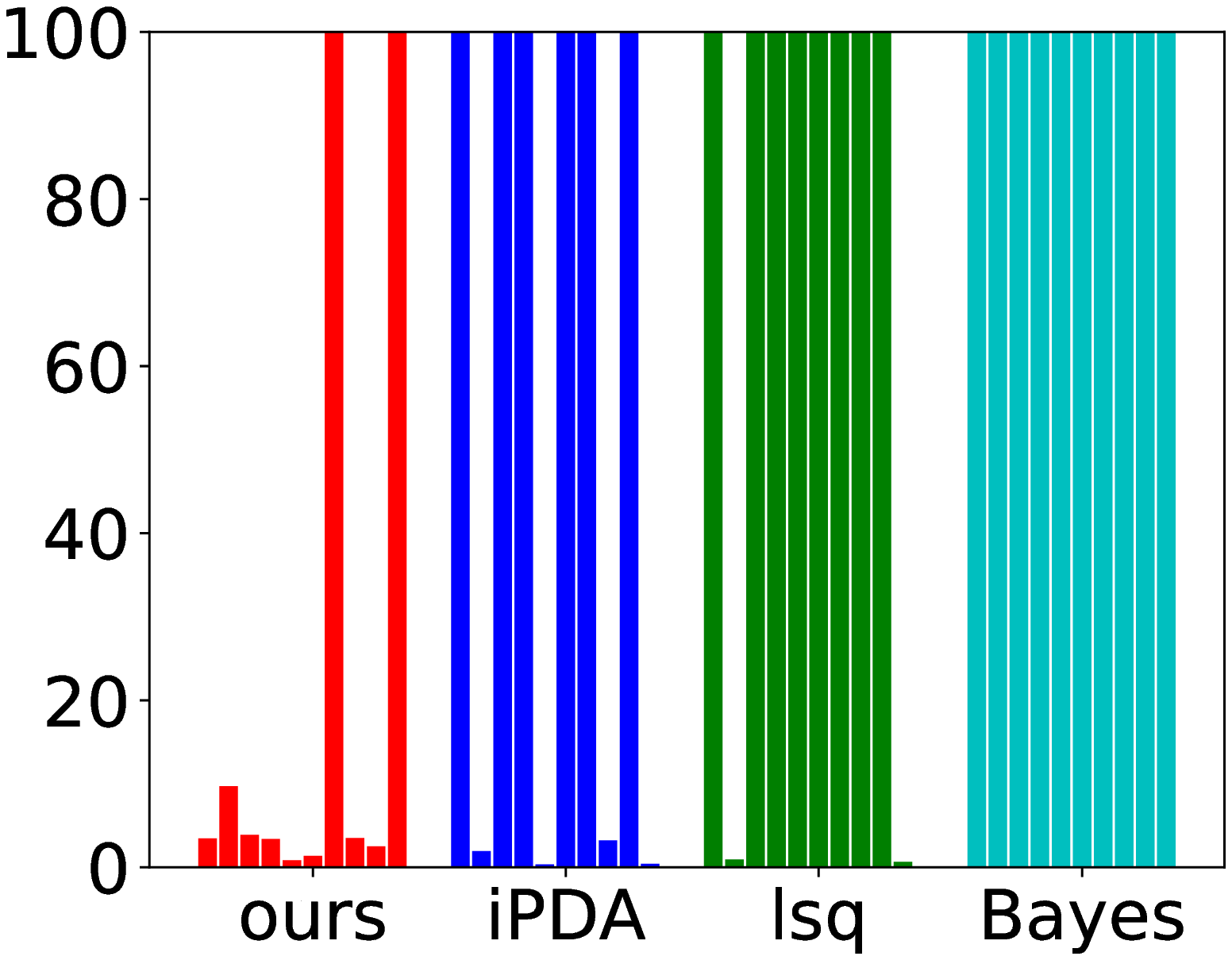}&
    \includegraphics*[width=0.235\linewidth,height=0.165\linewidth]{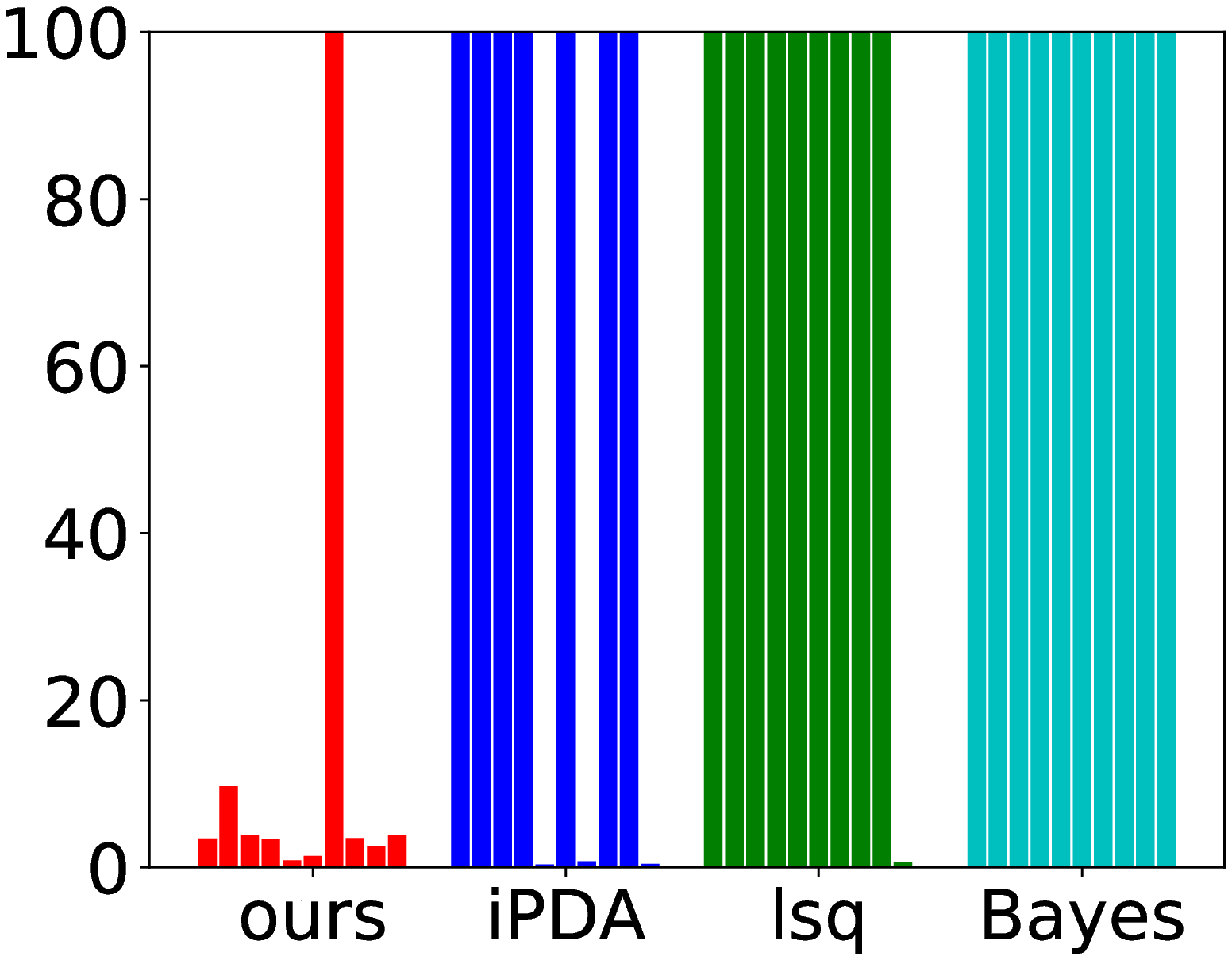}&
    \includegraphics*[width=0.235\linewidth,height=0.165\linewidth]{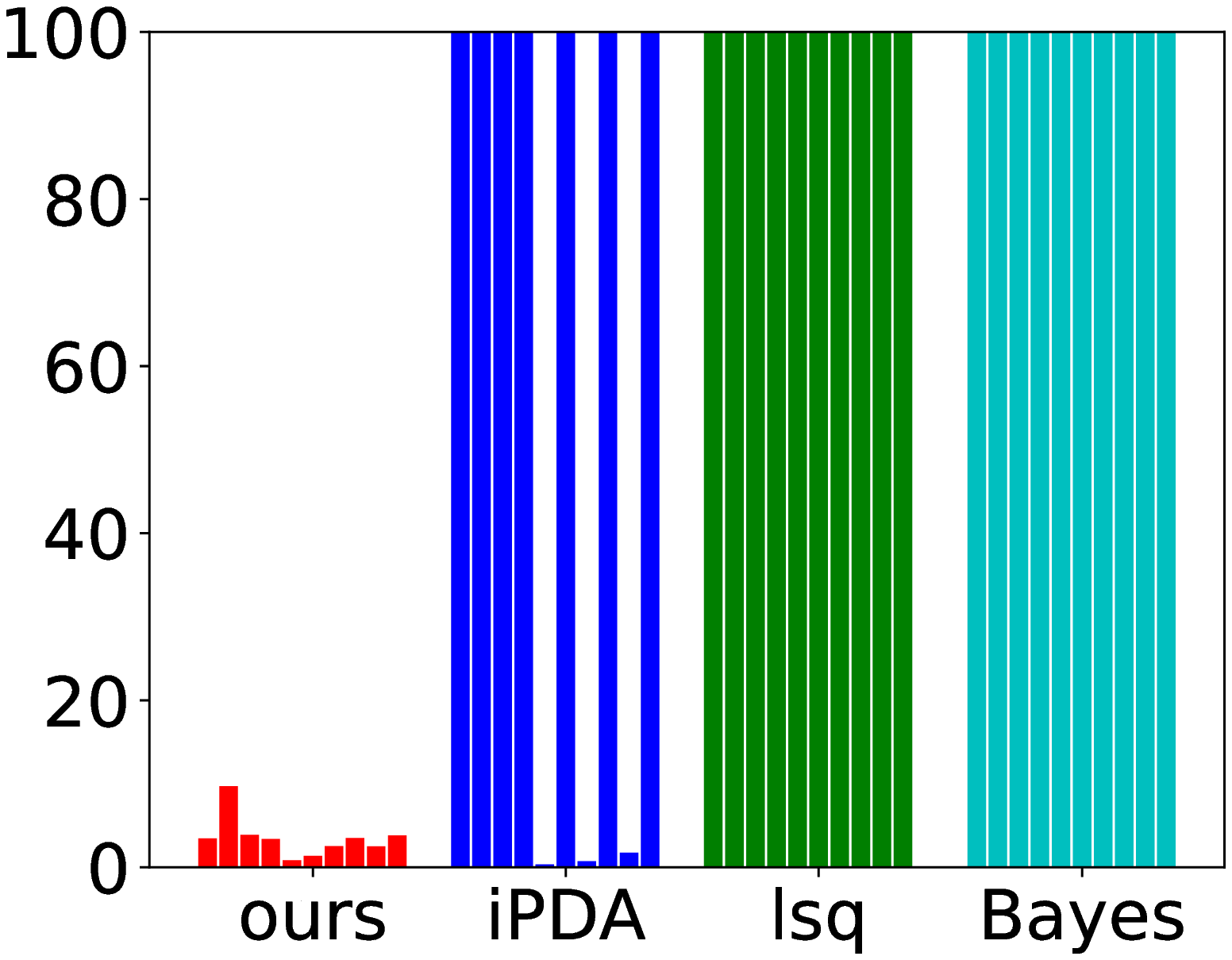}\\
     \hspace{3ex}\rotatebox{90}{\hspace{2.5ex} \small{$\theta_0$ error} } &       
    \hspace{1ex}\includegraphics*[width=0.225\linewidth,height=0.165\linewidth]{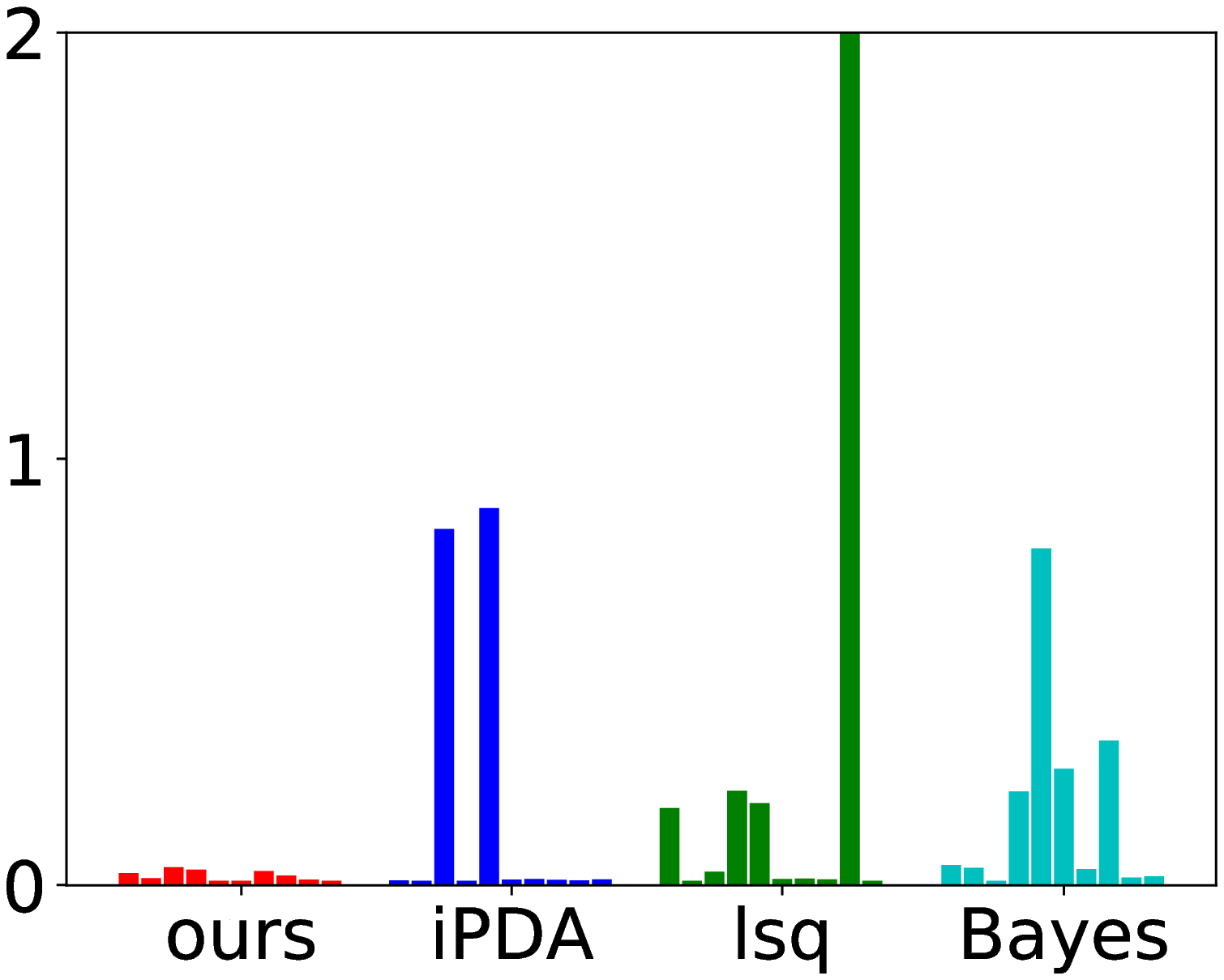}&
    \hspace{1ex}\includegraphics*[width=0.225\linewidth,height=0.165\linewidth]{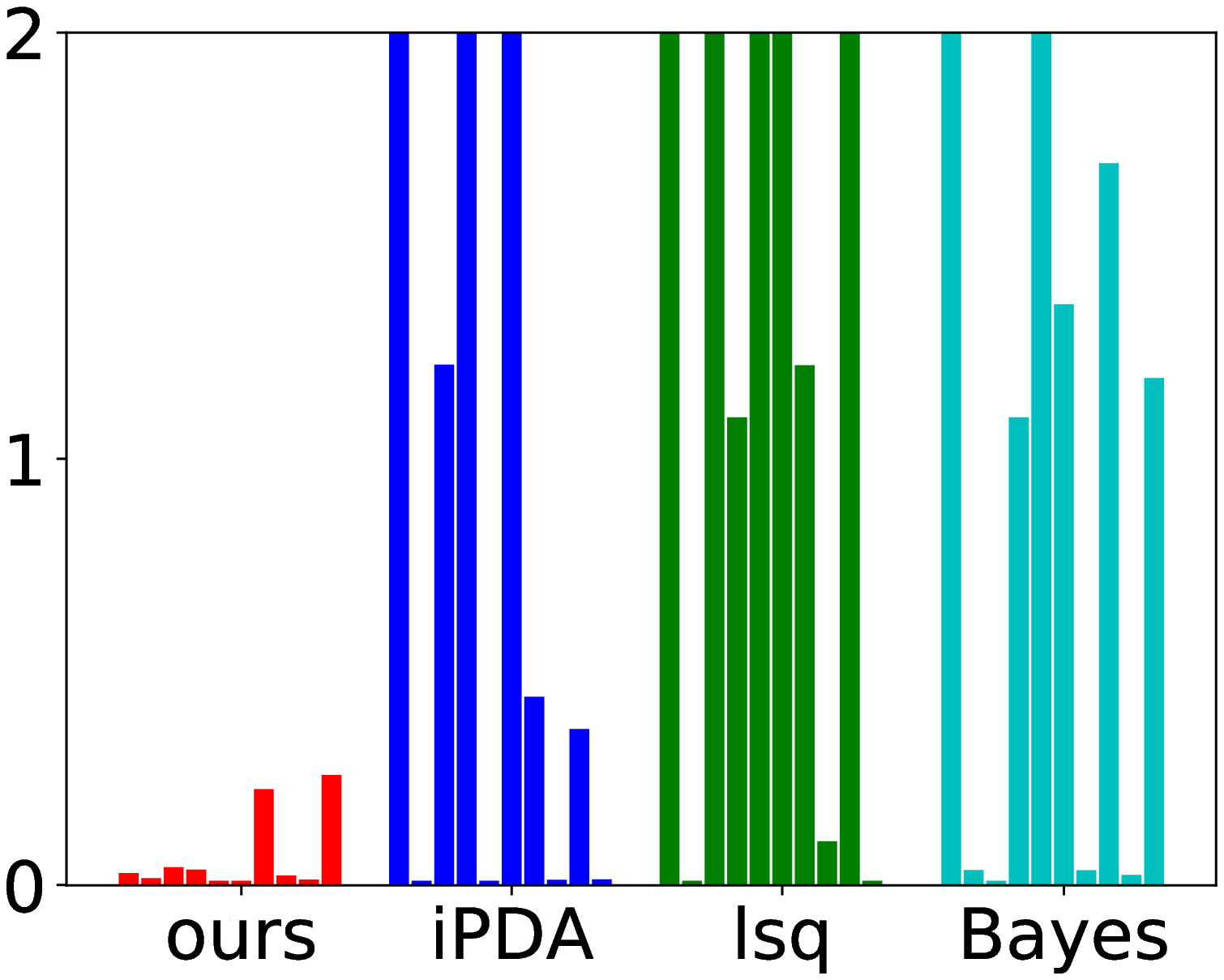}&
    \hspace{1ex}\includegraphics*[width=0.225\linewidth,height=0.165\linewidth]{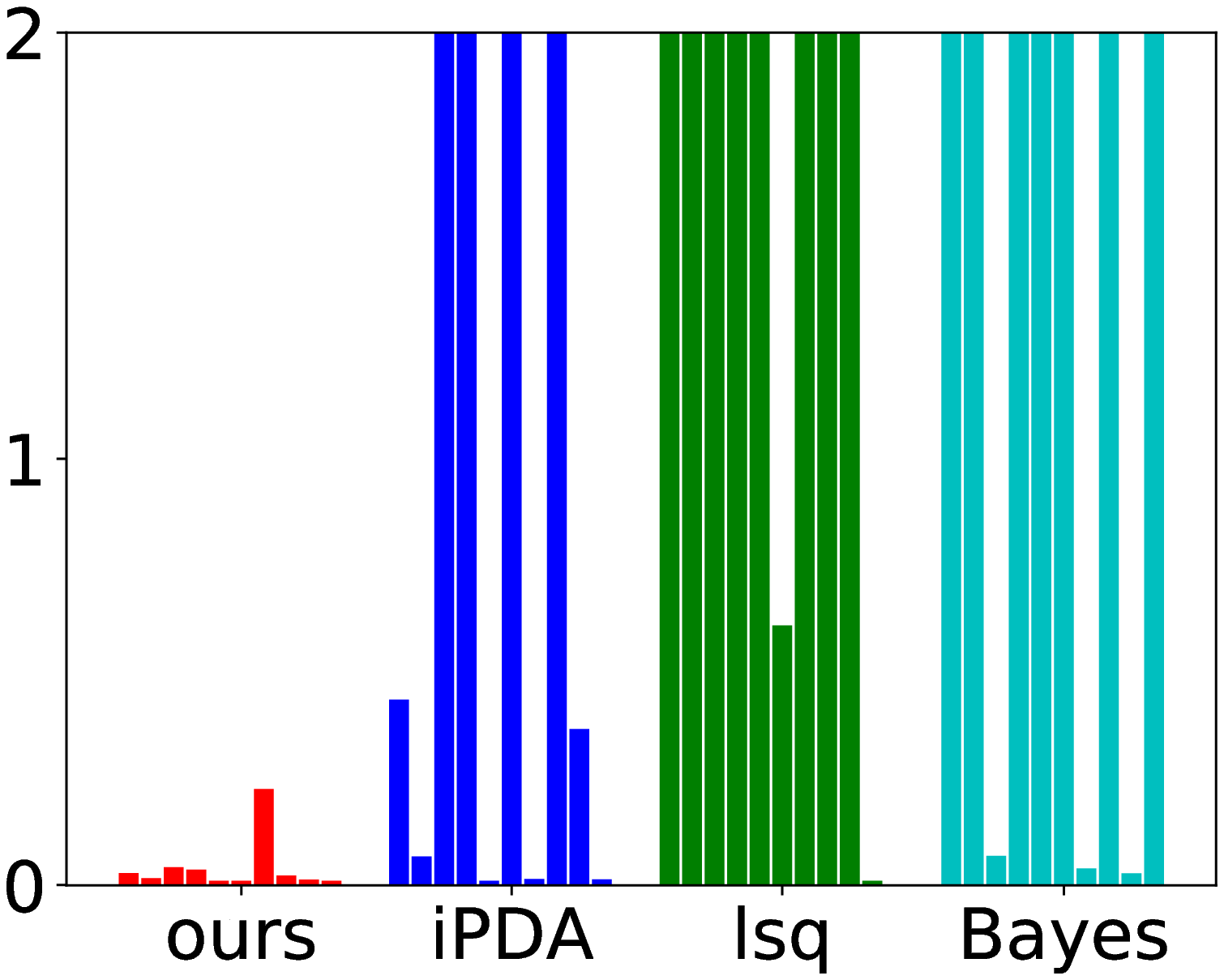}&
    \hspace{1ex}\includegraphics*[width=0.225\linewidth,height=0.165\linewidth]{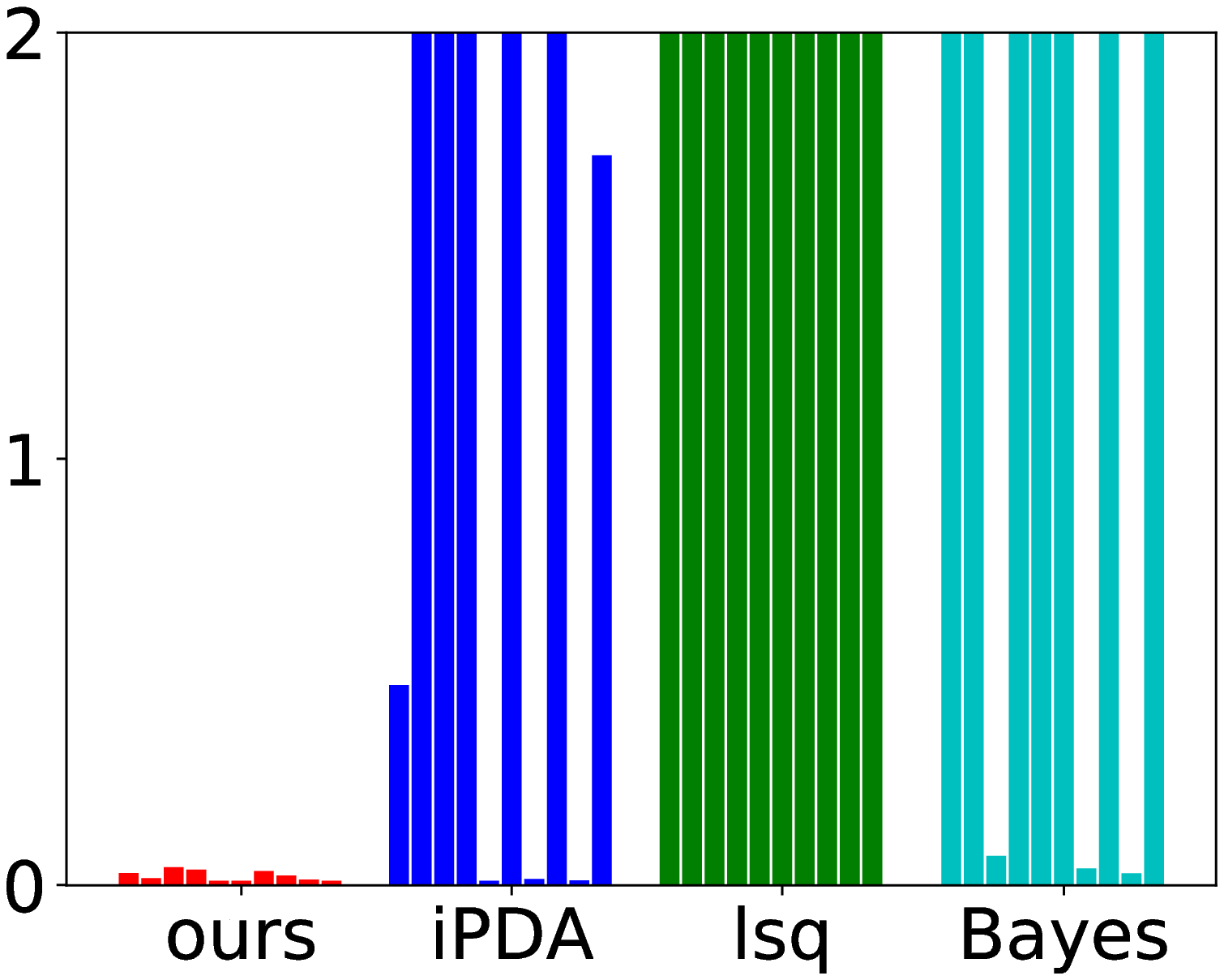}\\
     \hspace{3ex}\rotatebox{90}{\hspace{2.5ex} \small{$\theta_1$ error} } &       
    \hspace{1ex}\includegraphics*[width=0.225\linewidth,height=0.165\linewidth]{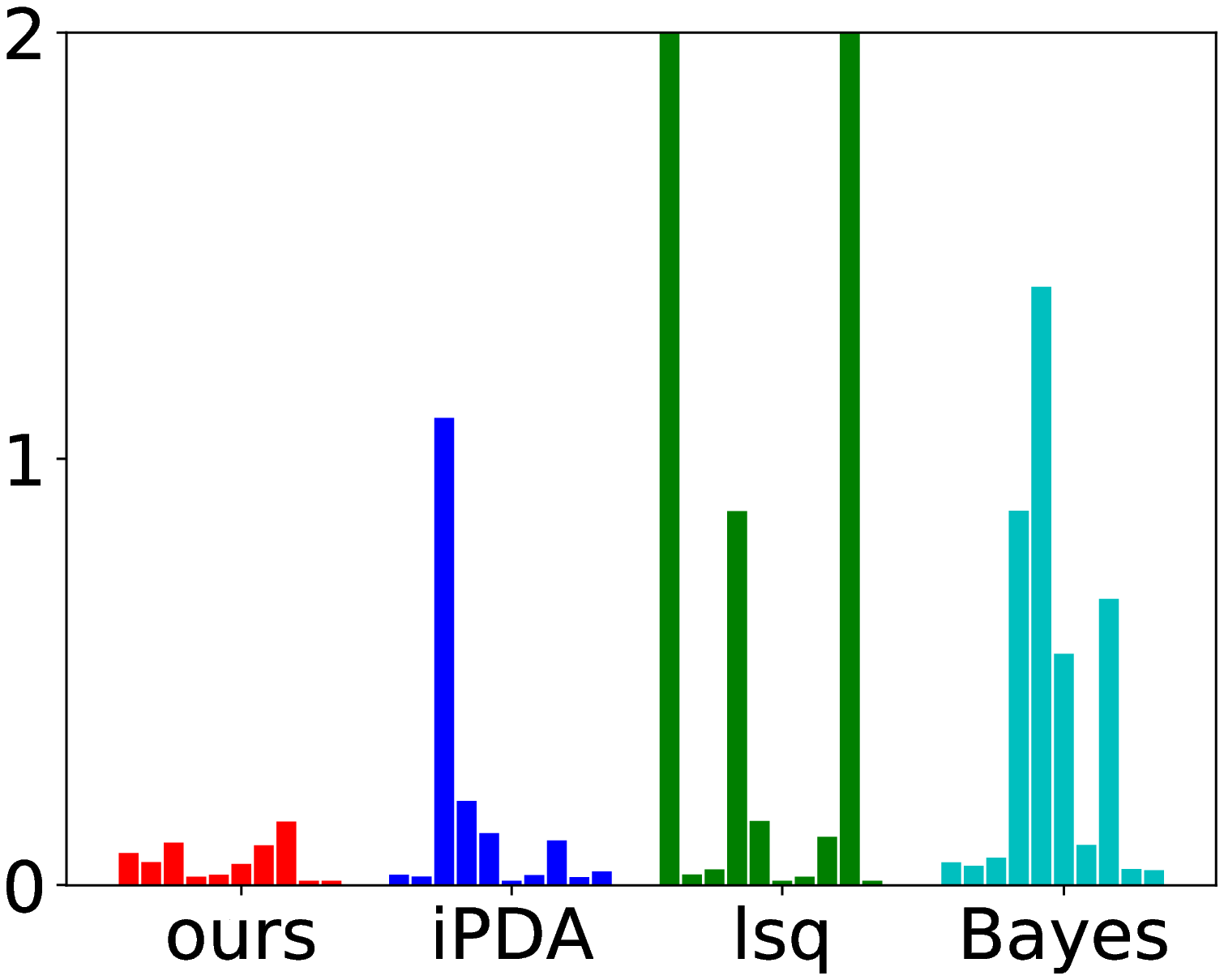}&
    \hspace{1ex}\includegraphics*[width=0.225\linewidth,height=0.165\linewidth]{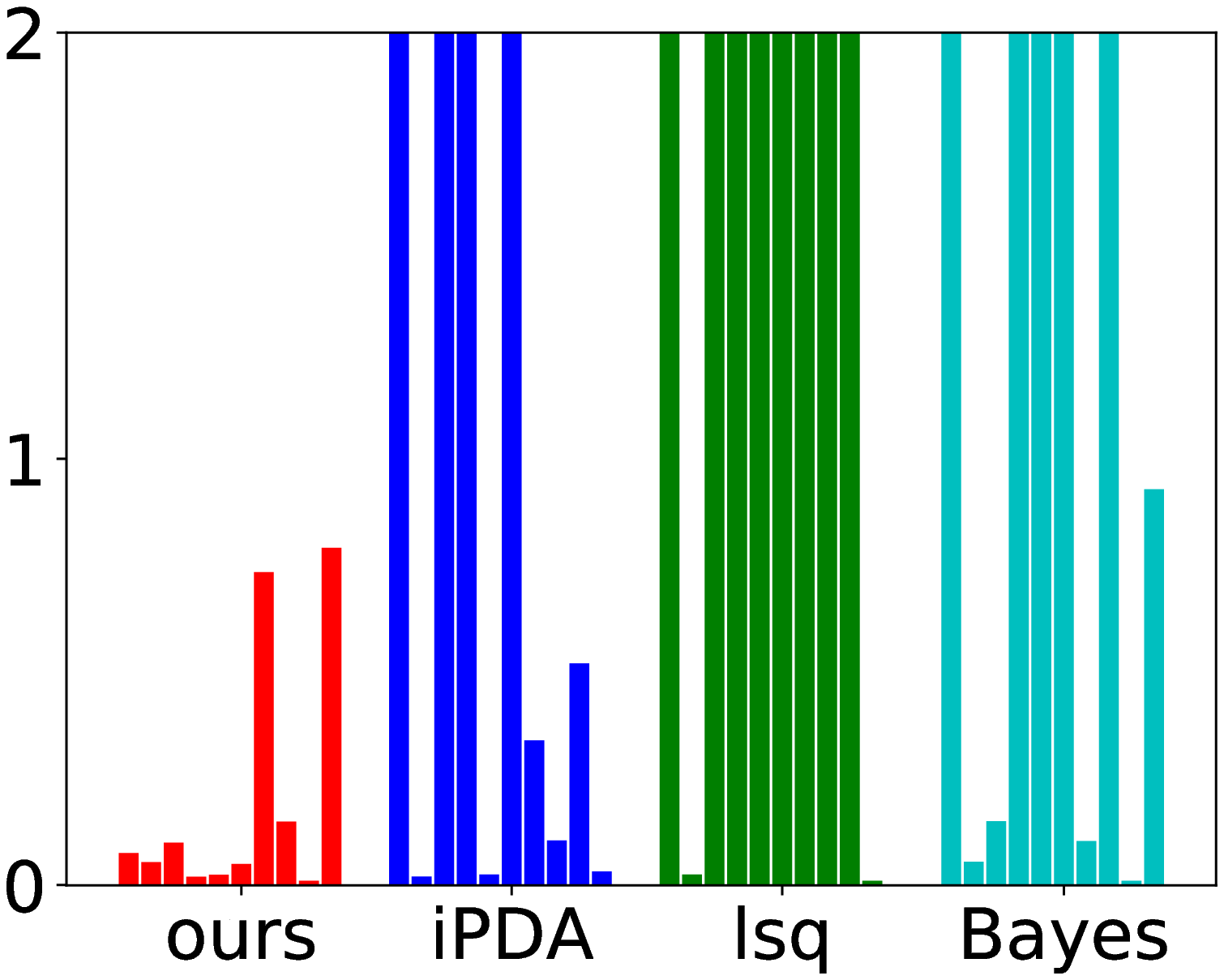}&
    \hspace{1ex}\includegraphics*[width=0.225\linewidth,height=0.165\linewidth]{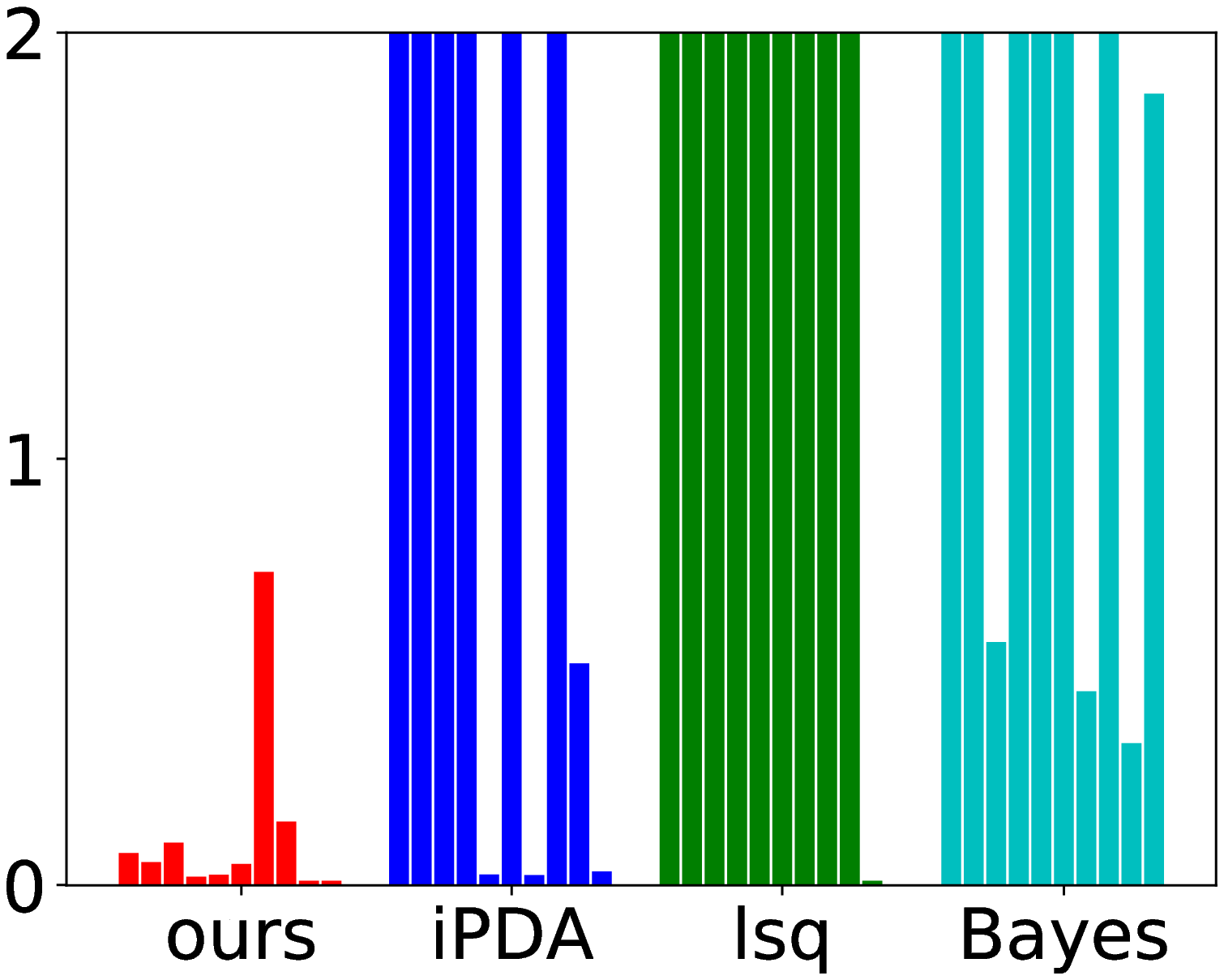}&
    \hspace{1ex}\includegraphics*[width=0.225\linewidth,height=0.165\linewidth]{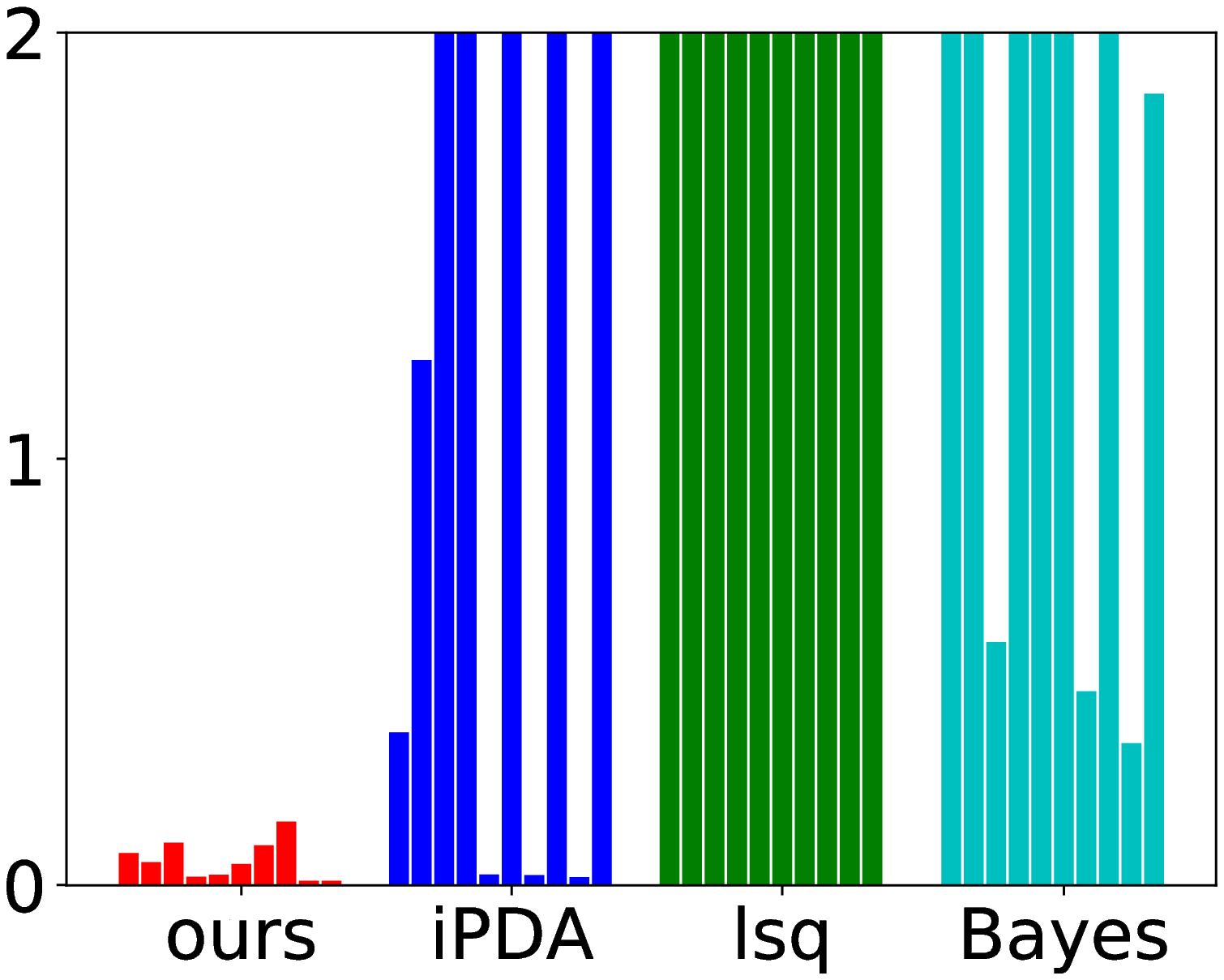}\\
     \hspace{3ex}\rotatebox{90}{\hspace{2.5ex} \small{$\theta_2$ error} } &       
    \hspace{1ex}\includegraphics*[width=0.225\linewidth,height=0.165\linewidth]{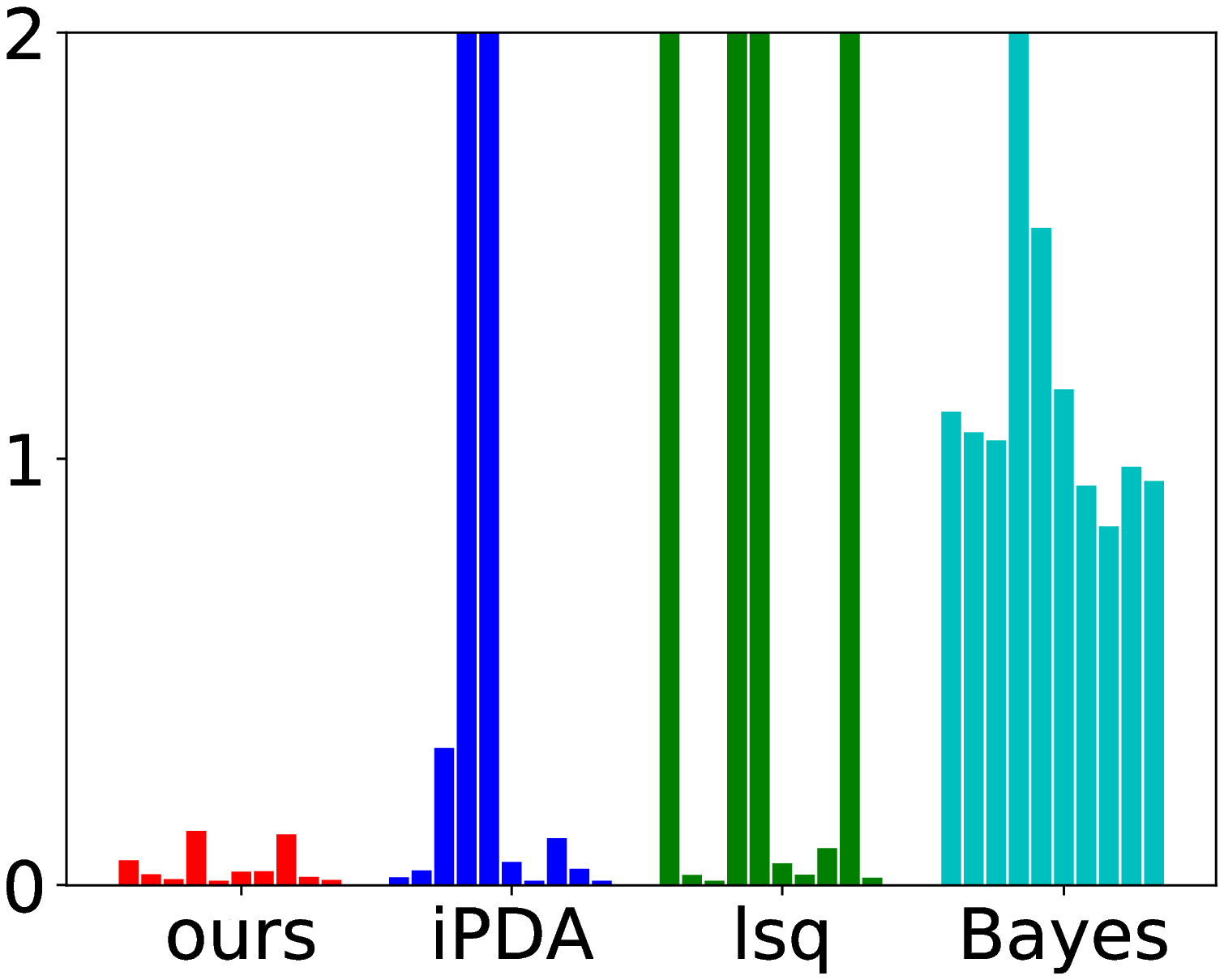}&
    \hspace{1ex}\includegraphics*[width=0.225\linewidth,height=0.165\linewidth]{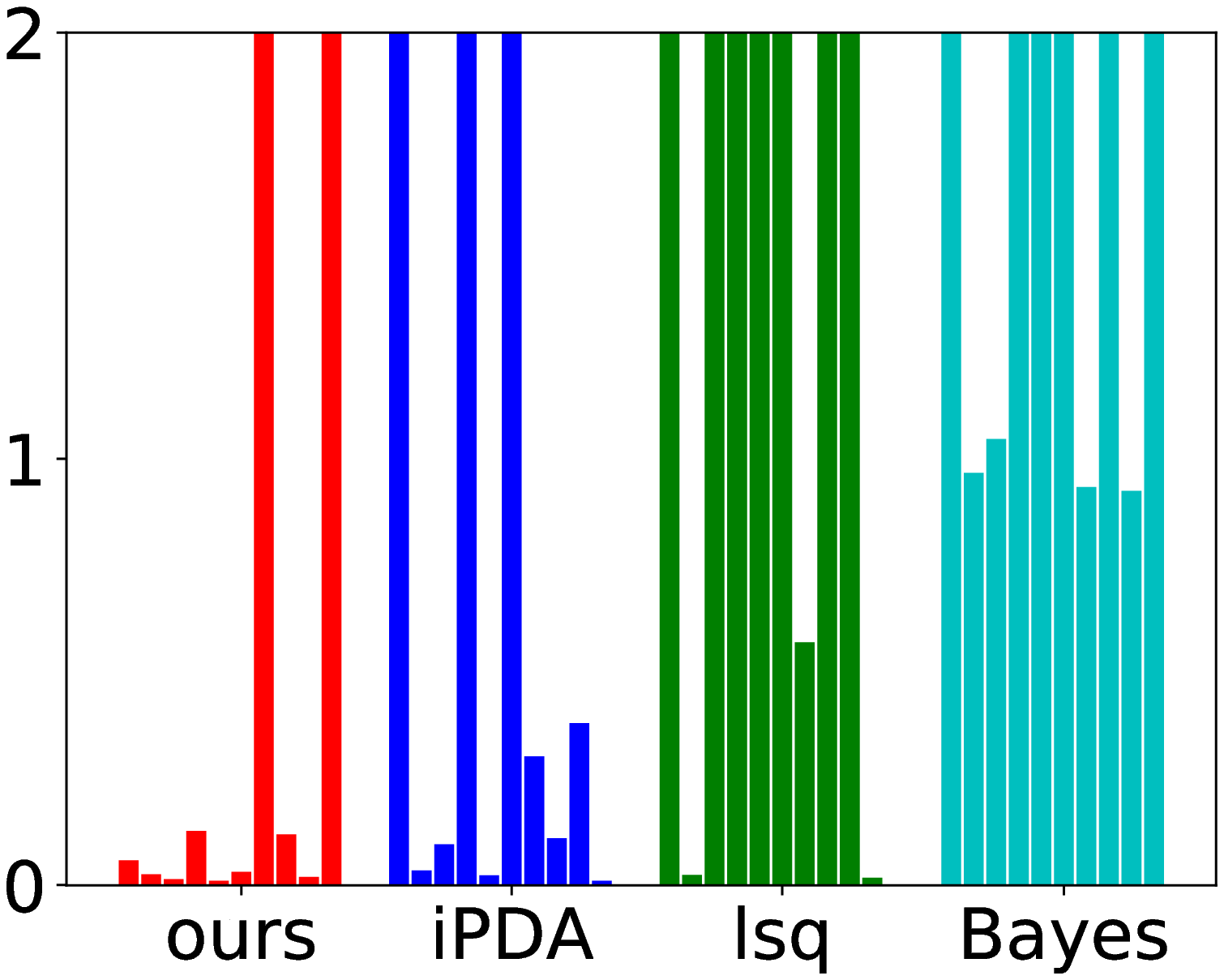}&
    \hspace{1ex}\includegraphics*[width=0.225\linewidth,height=0.165\linewidth]{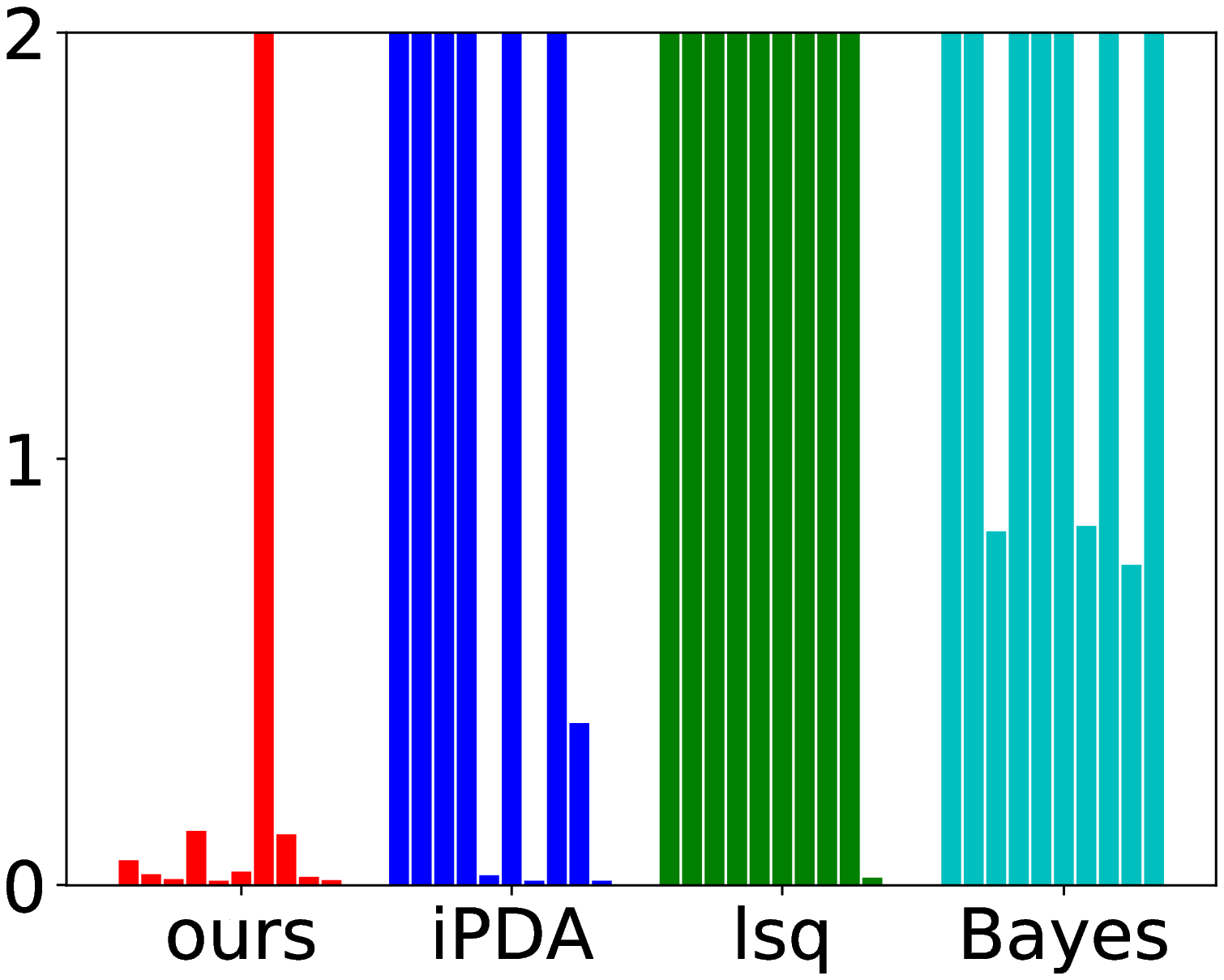}&
    \hspace{1ex}\includegraphics*[width=0.225\linewidth,height=0.165\linewidth]{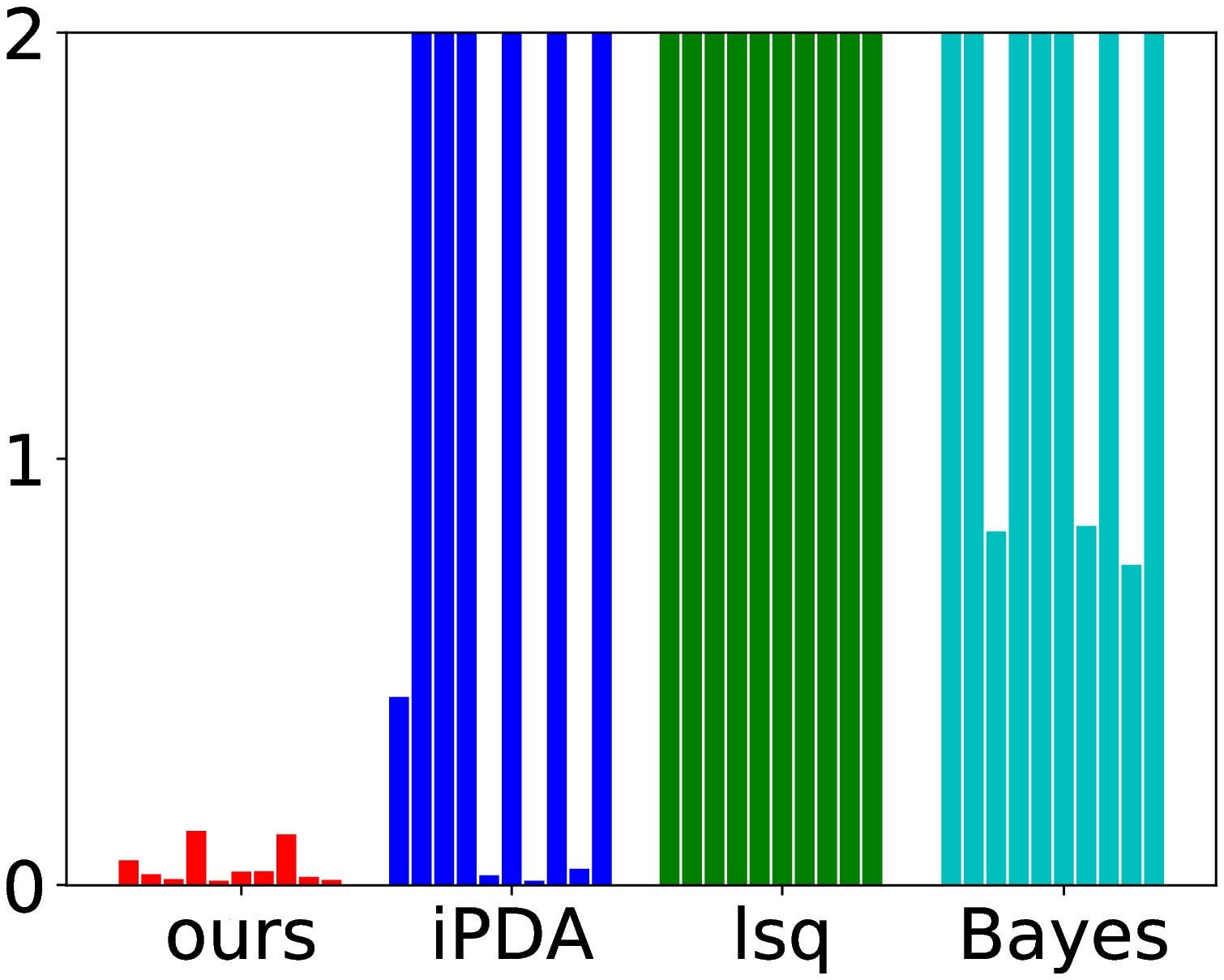}
  \end{tabular}
  \caption{Similar to Fig.\ 4 of the paper, but on the FitzHugh--Nagumo model. Our method significantly outperforms other methods.}

  \label{f:compSUPP}
\end{figure*}

\begin{figure*}[!t]
  \centering
  \begin{tabular}{@{}c@{\hspace{1ex}}c@{\hspace{0ex}}c@{\hspace{0ex}}c@{}c@{}}
    &\large{$\sigma^2=.5$} & \large{$\sigma^2=1$} &  \large{$\sigma^2=1.5$}\\
    \rotatebox{90}{\hspace{5ex}  \large{pred.\ error} } &
    \includegraphics*[width=0.31\linewidth]{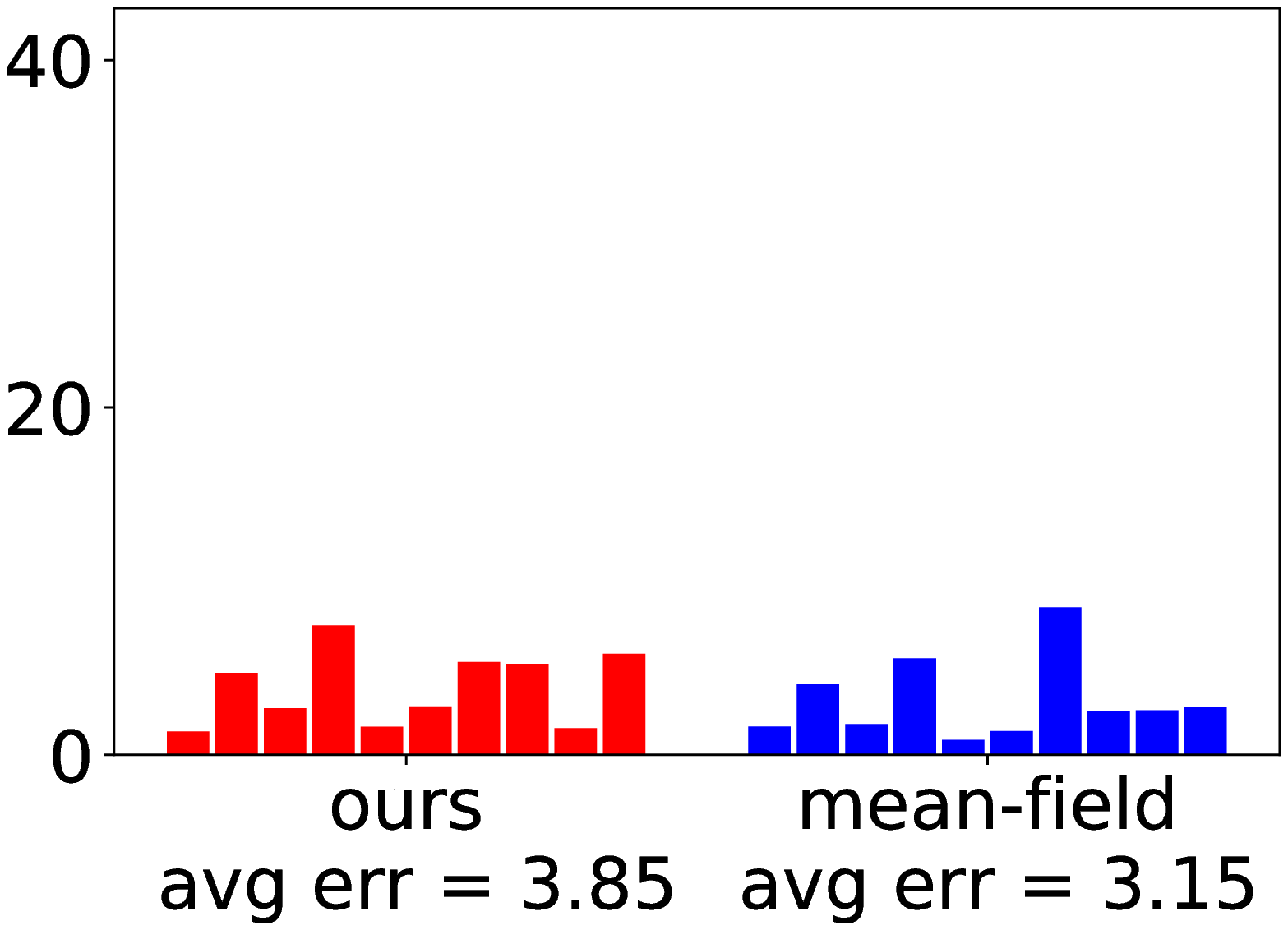}& 
	\includegraphics*[width=0.31\linewidth]{predator/compad3/objstd1.eps}&
	 \includegraphics*[width=0.31\linewidth]{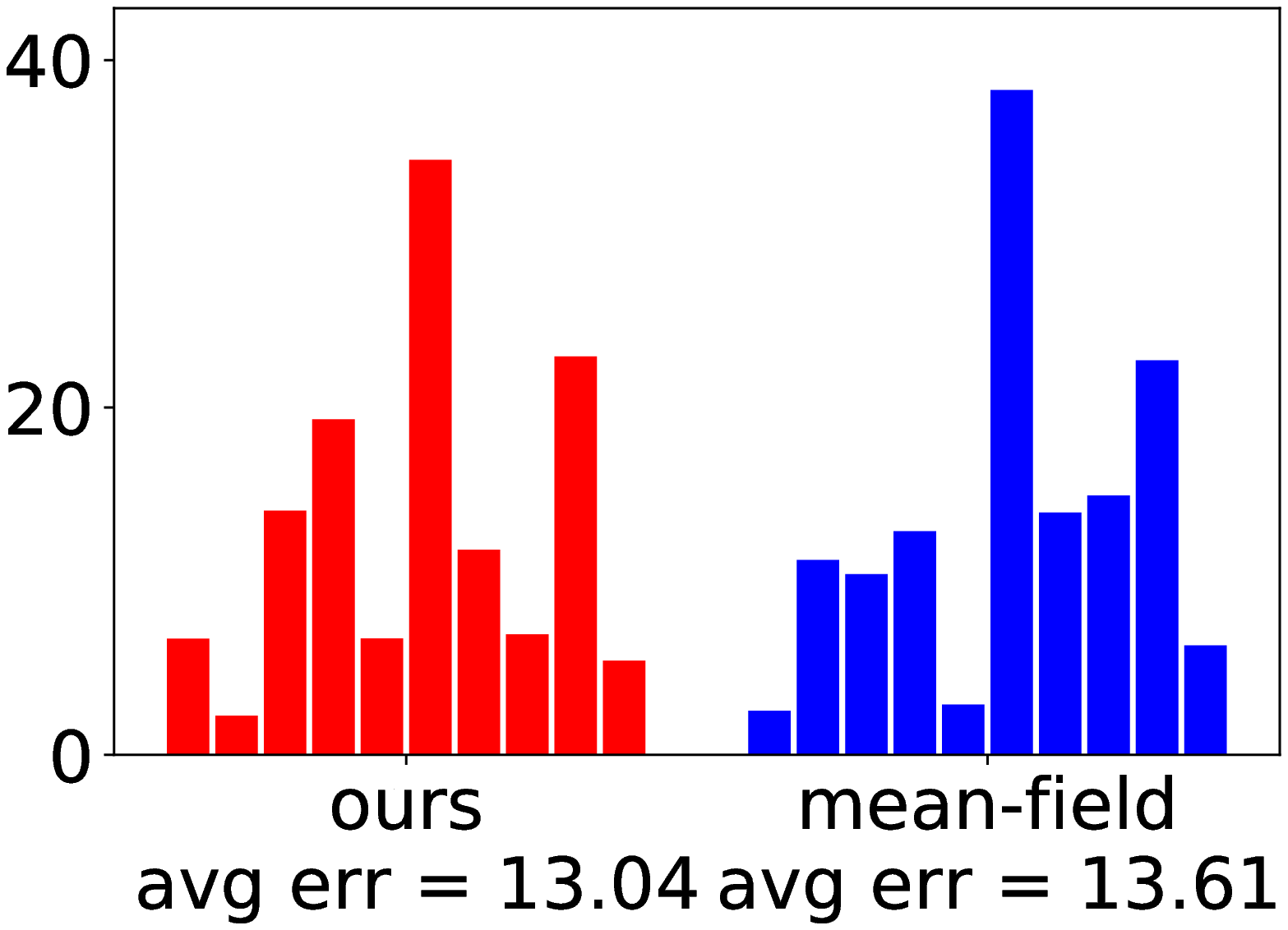}\\	 
    \rotatebox{90}{\hspace{7ex}  \large{$\theta_0$ error} } &  
    \includegraphics*[width=0.31\linewidth]{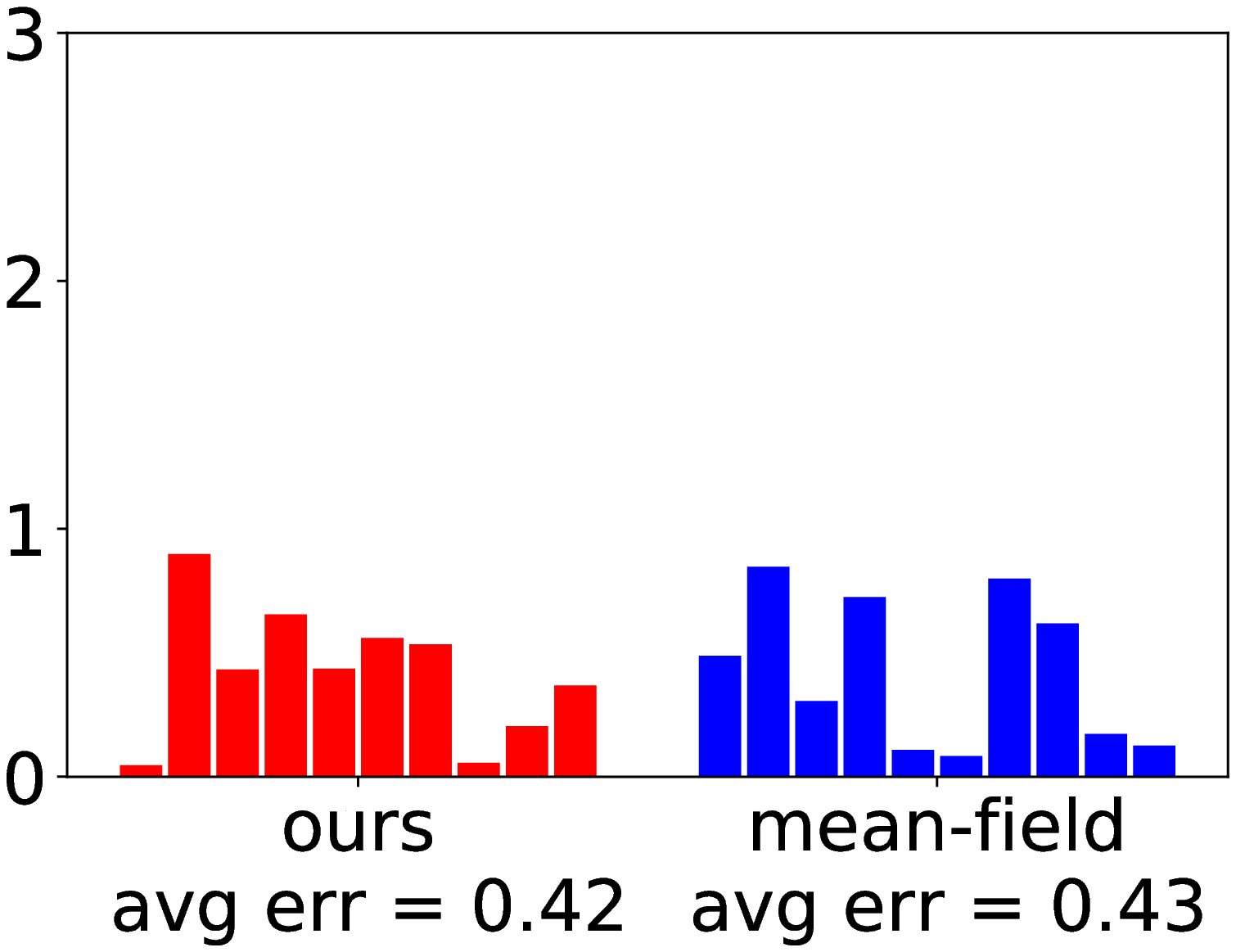}&
	 \includegraphics*[width=0.31\linewidth]{predator/compad3/param0std1.eps}&
	 \includegraphics*[width=0.31\linewidth]{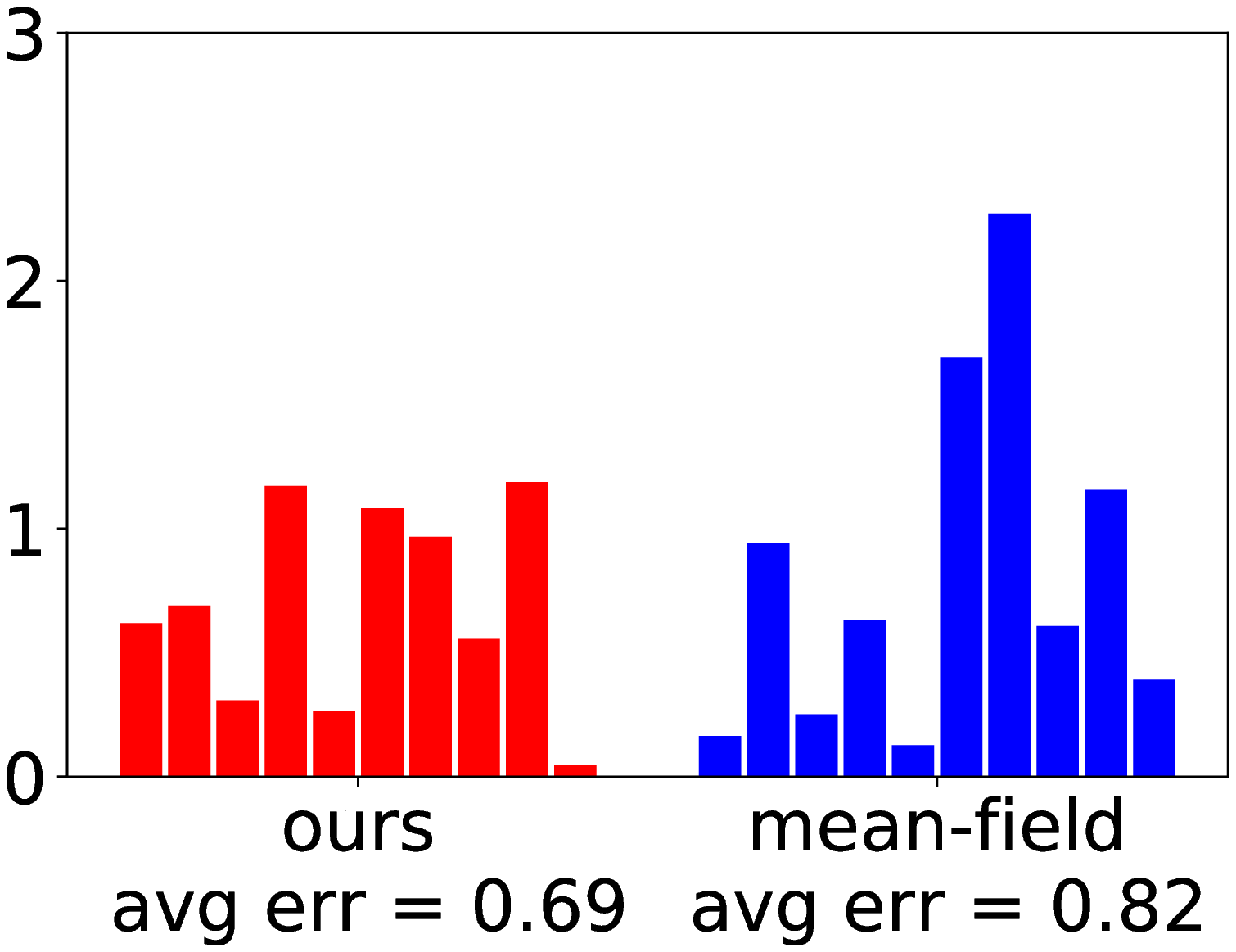}\\	 
	  \rotatebox{90}{\hspace{7ex}  \large{$\theta_1$ error} } &  
    \includegraphics*[width=0.31\linewidth]{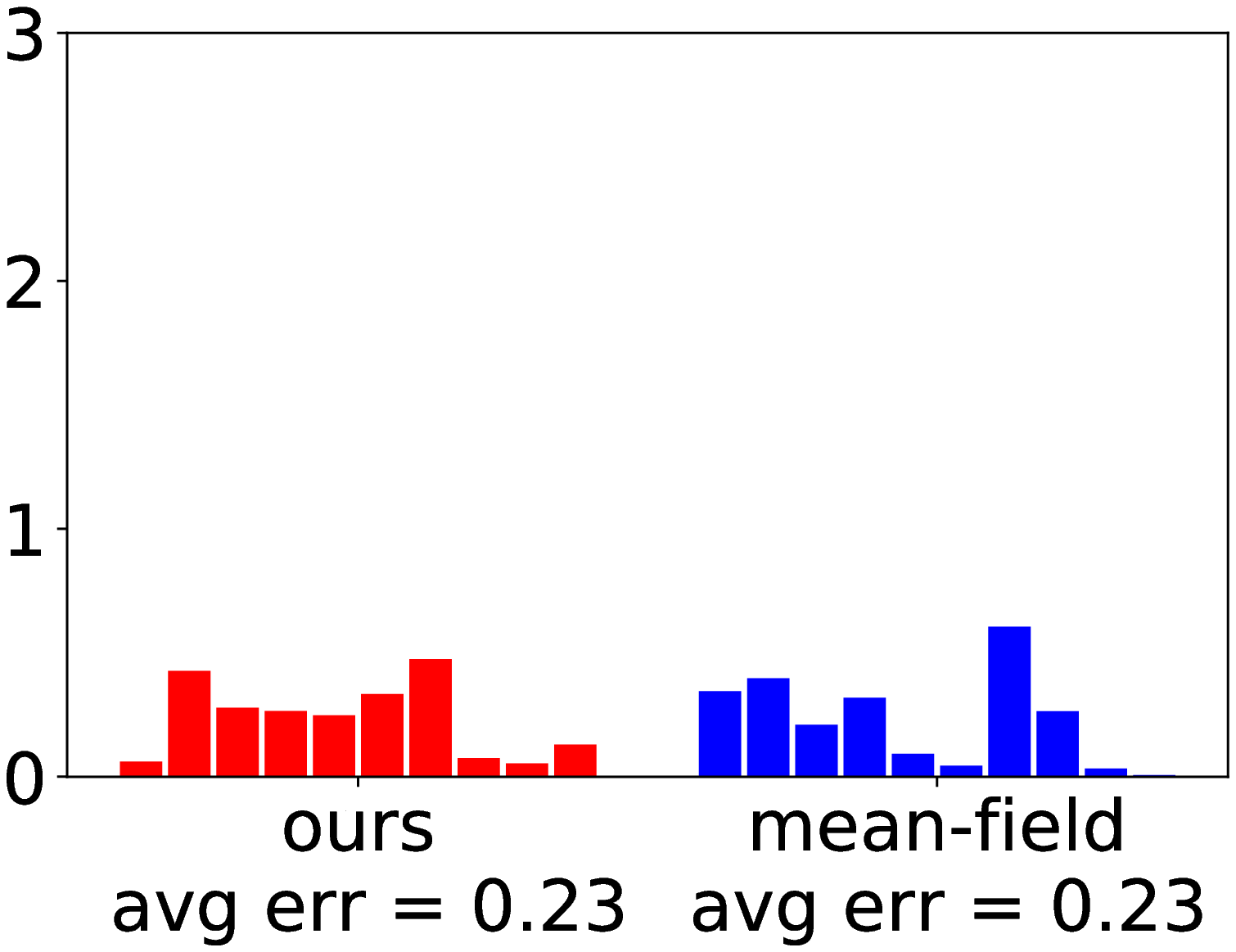}&
	 \includegraphics*[width=0.31\linewidth]{predator/compad3/param1std1.eps}&
	 \includegraphics*[width=0.31\linewidth]{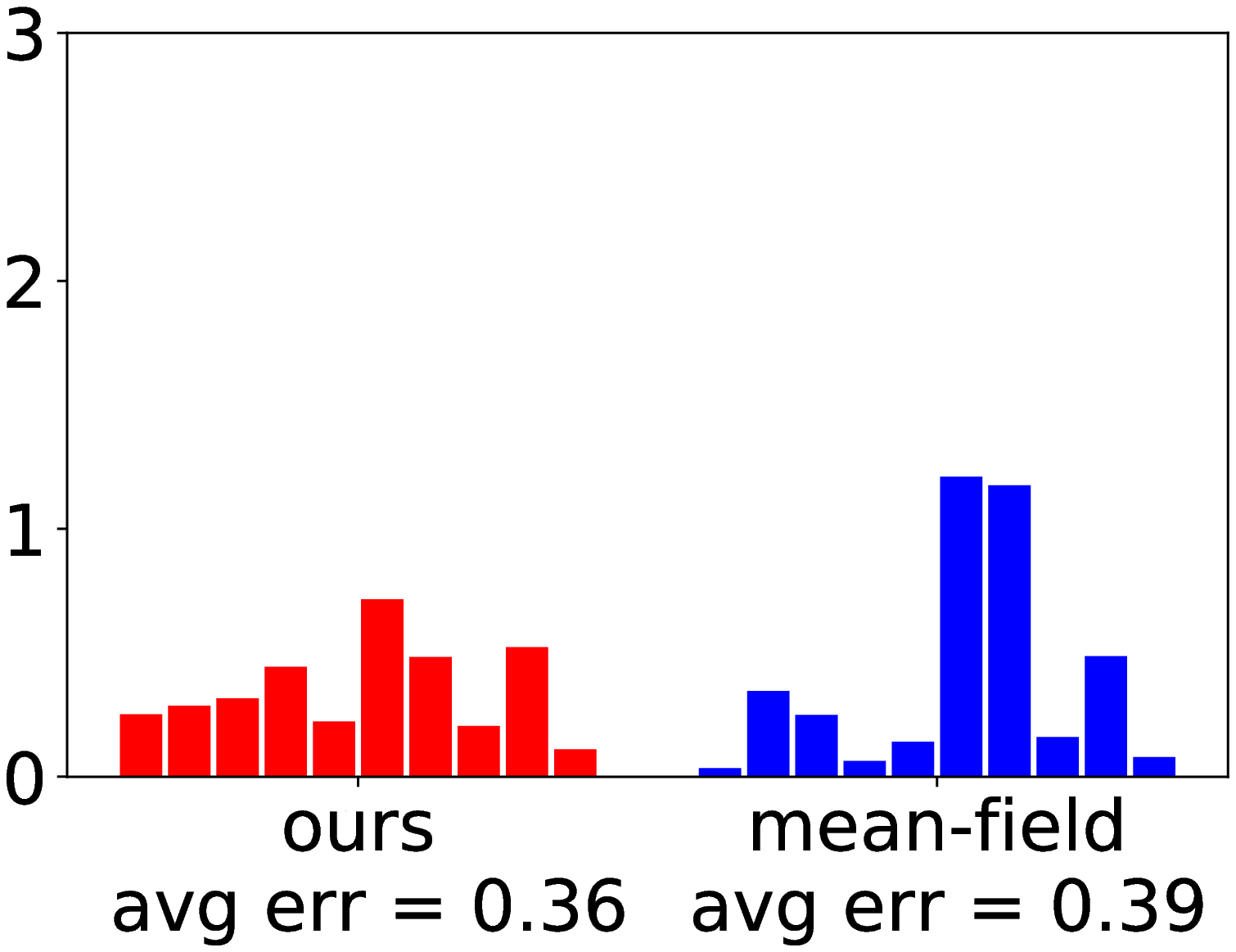}\\	 
	  \rotatebox{90}{\hspace{7ex}  \large{$\theta_2$ error} } &
    \includegraphics*[width=0.31\linewidth]{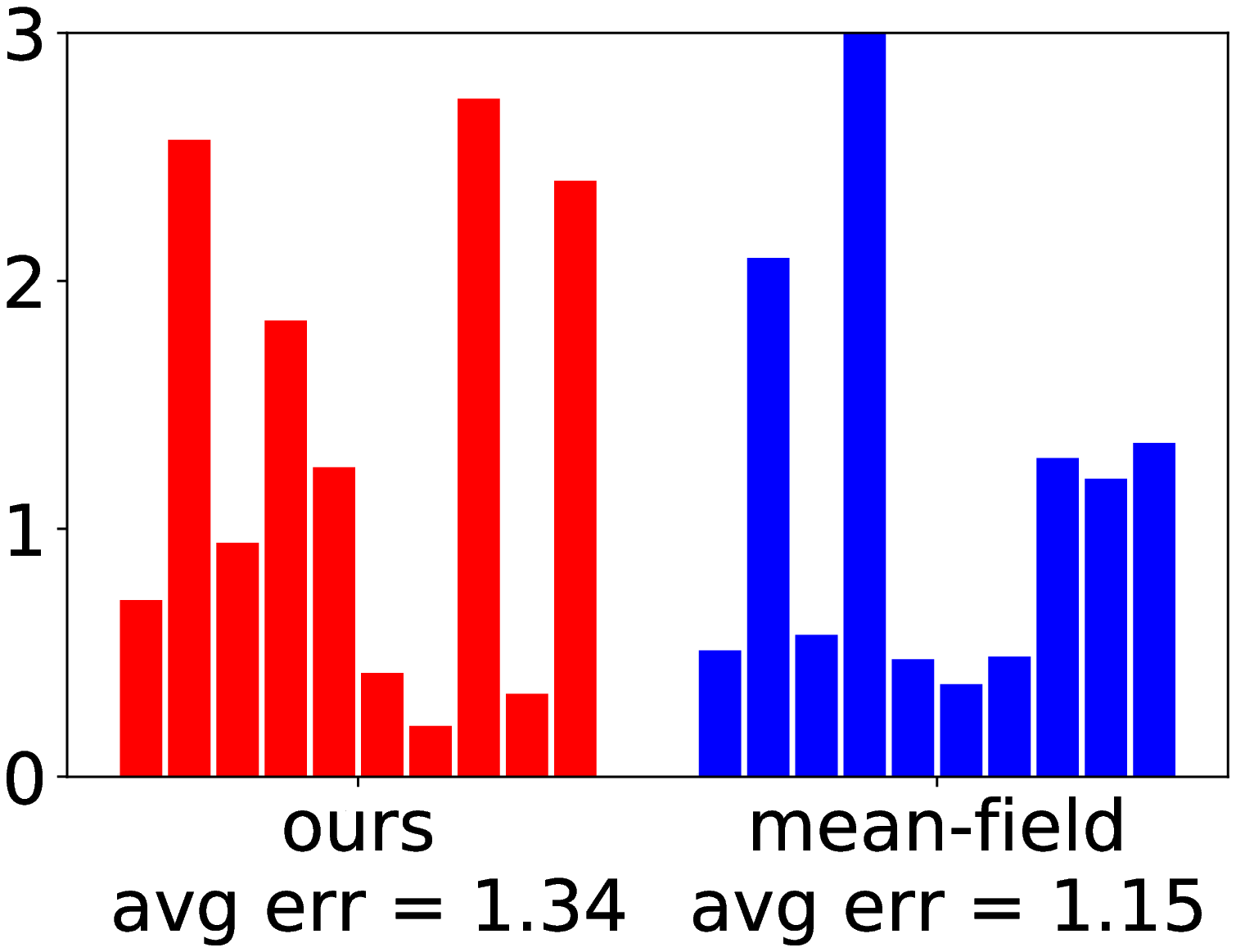}&
	 \includegraphics*[width=0.31\linewidth]{predator/compad3/param2std1.eps}&
	 \includegraphics*[width=0.31\linewidth]{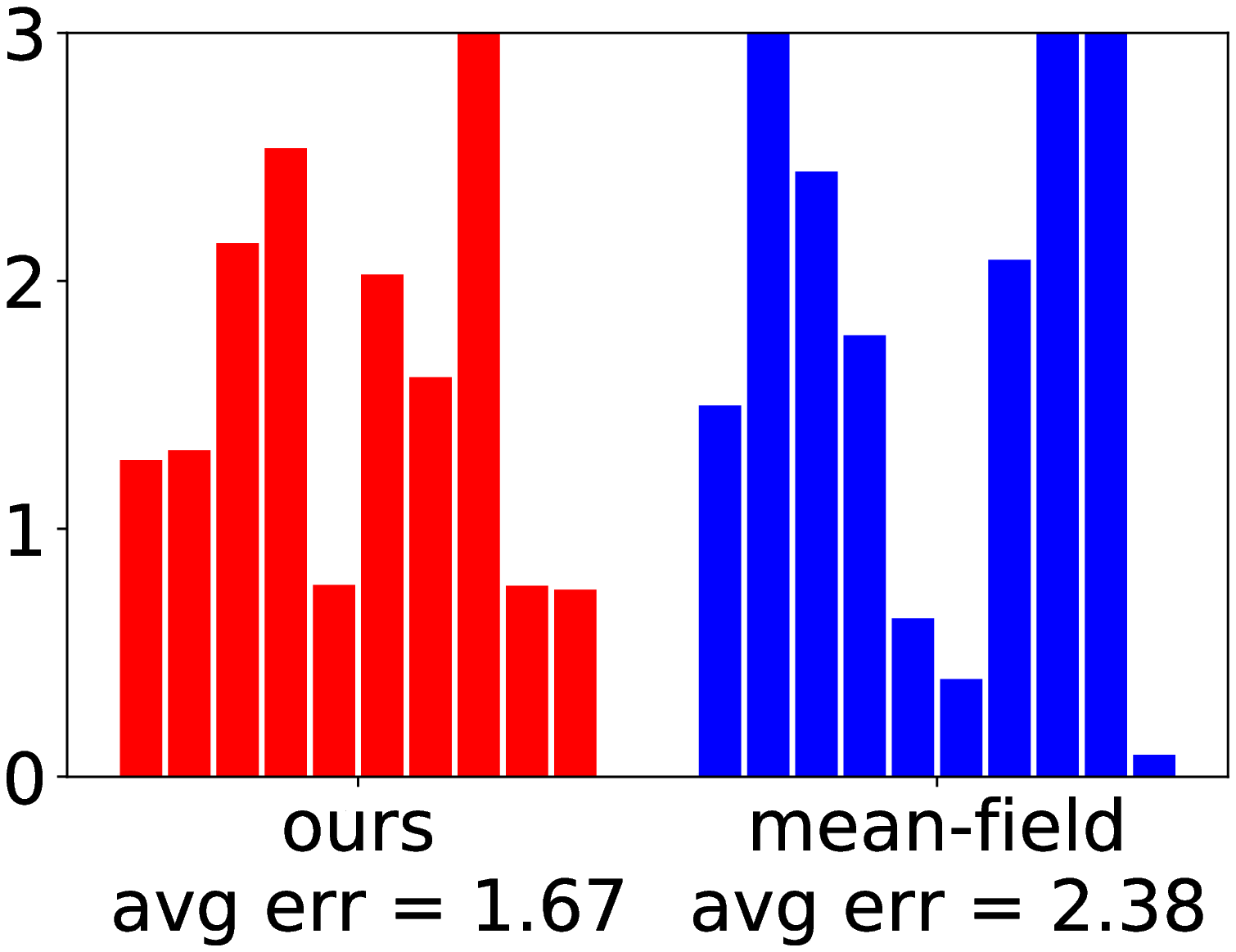}\\	 	 
	 	  \rotatebox{90}{\hspace{7ex}  \large{$\theta_3$ error} } &  
    \includegraphics*[width=0.31\linewidth]{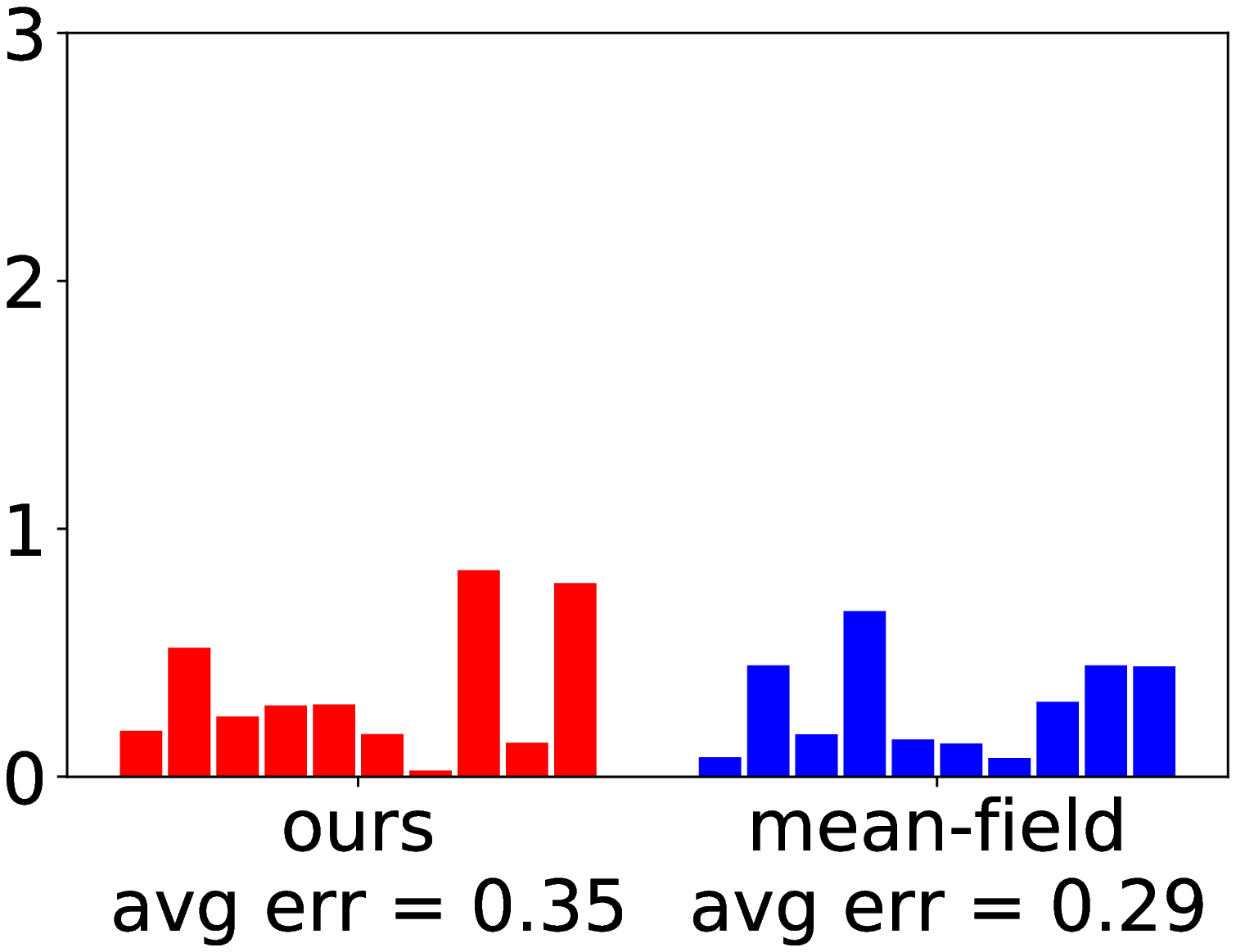}&
	 \includegraphics*[width=0.31\linewidth]{predator/compad3/param3std1.eps}&
	 \includegraphics*[width=0.31\linewidth]{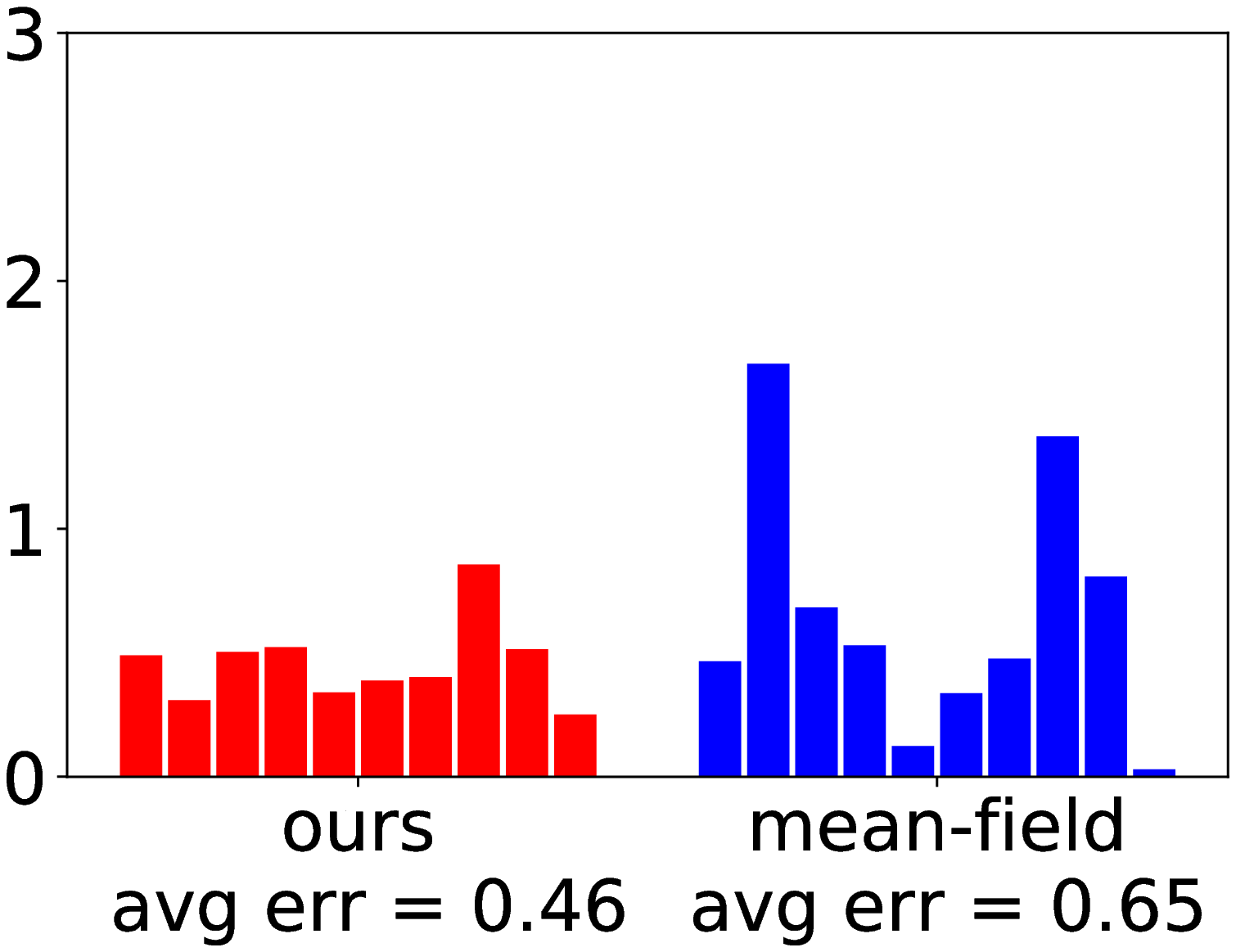}
  \end{tabular}
  \caption{Comparison with the mean-field method. Similar to the first row of Fig.\ 5 in the paper, but for a set of noise variances: $\sigma^2=.5, 1,$ and $1.5$. Our method is more robust with respect to  noise and performs better.}

  \label{f:compSUPP1}
\end{figure*}

\begin{figure*}[!t]
  \centering
  \begin{tabular}{@{}c@{\hspace{1ex}}c@{\hspace{0ex}}c@{\hspace{2ex}}c@{}c@{}}
  	& \multicolumn{2}{ c }{\dotfill $T=20$ \dotfill} & \multicolumn{2}{ c }{\dotfill $T=10\,000$ \dotfill} \\
    &\small{$\sigma^2=.1$} & \small{$\sigma^2=1.5$} &  \small{$\sigma^2=.1$} & \small{$\sigma^2=1.5$}\\
    \rotatebox{90}{\hspace{2ex}  \small{avg.\ est.\ error} } &
    \includegraphics*[width=0.23\linewidth]{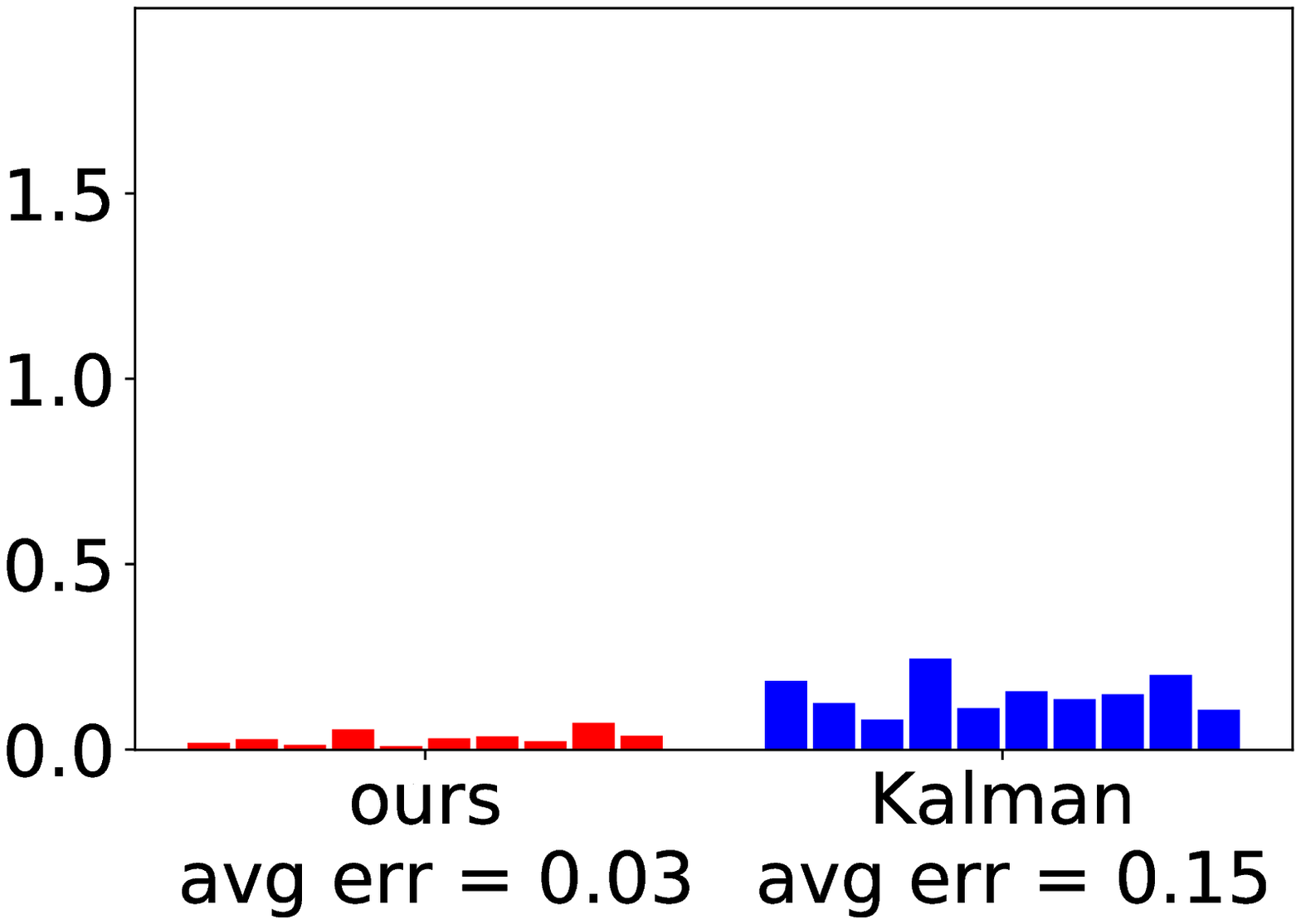}& 
    \includegraphics*[width=0.23\linewidth]{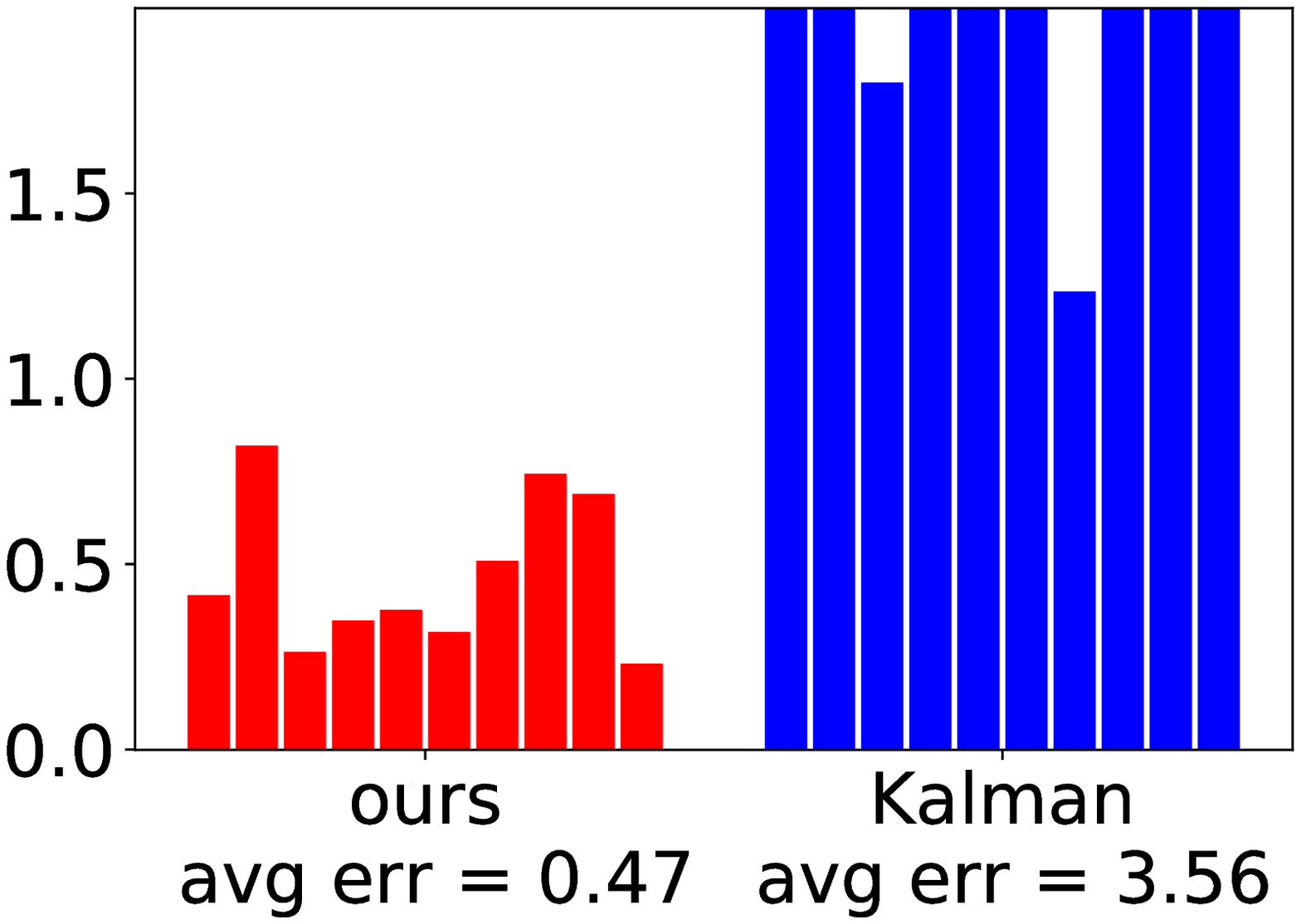}& 
	    \includegraphics*[width=0.23\linewidth]{kalman/predatorobj10000std01.eps}& 
    \includegraphics*[width=0.23\linewidth]{kalman/predatorobj10000std15.eps}\\
    \rotatebox{90}{\hspace{4ex}  \small{$\theta_0$ error} } &  
    \includegraphics*[width=0.23\linewidth]{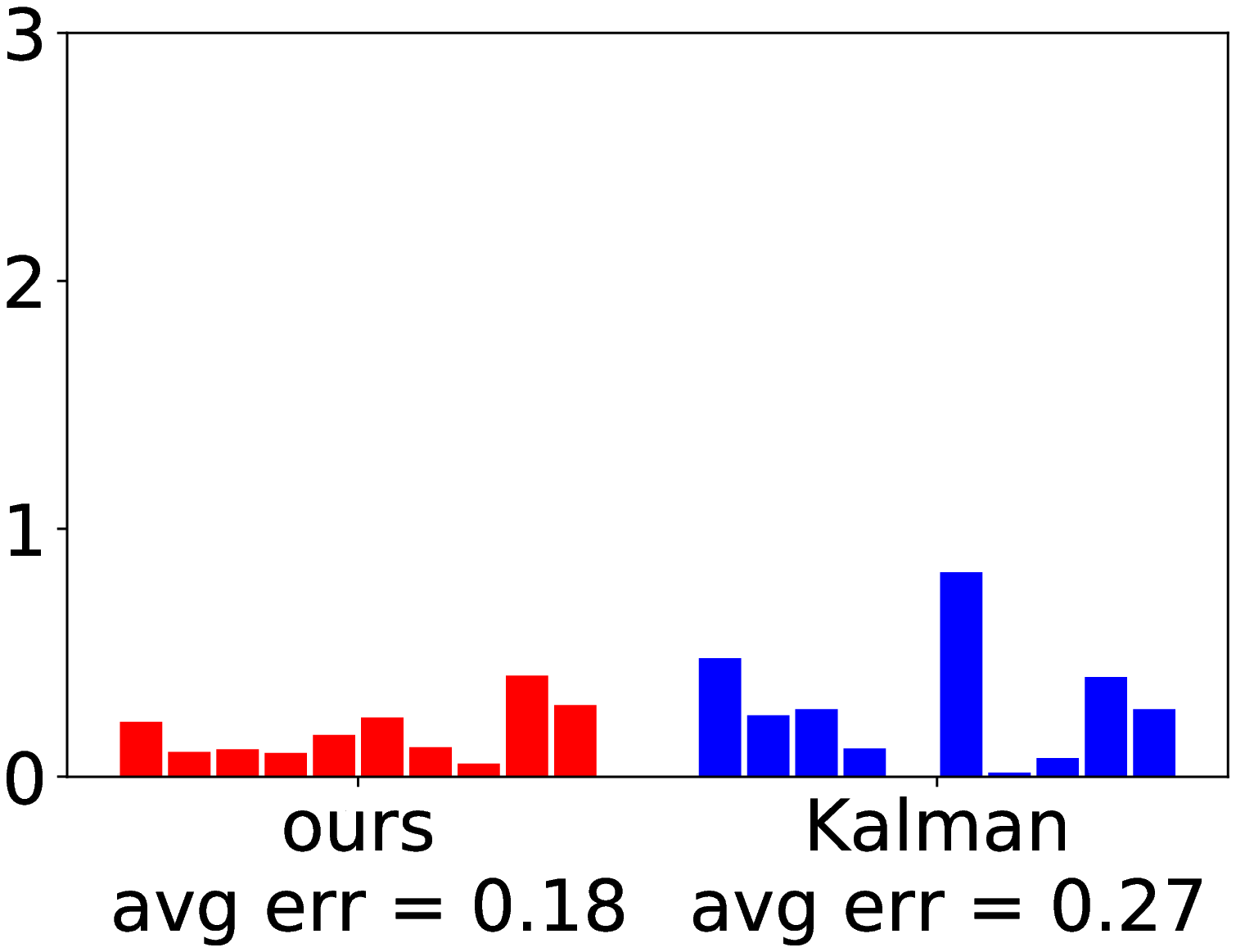}&
    \includegraphics*[width=0.23\linewidth]{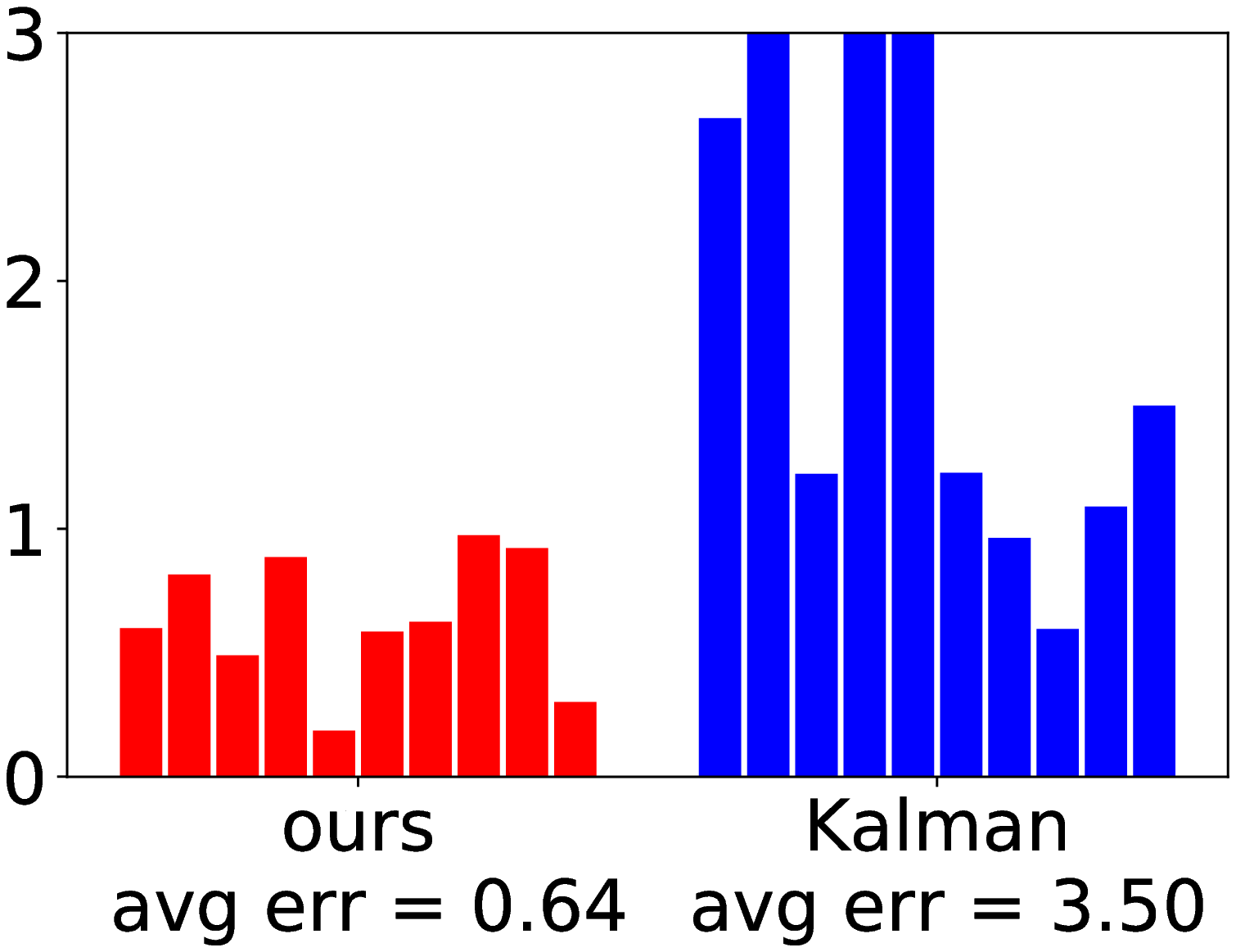}&
    \includegraphics*[width=0.23\linewidth]{kalman/predatorobjparam010000std01.eps}&
    \includegraphics*[width=0.23\linewidth]{kalman/predatorobjparam010000std15.eps}\\
	  \rotatebox{90}{\hspace{4ex}  \small{$\theta_1$ error} } &  
    \includegraphics*[width=0.23\linewidth]{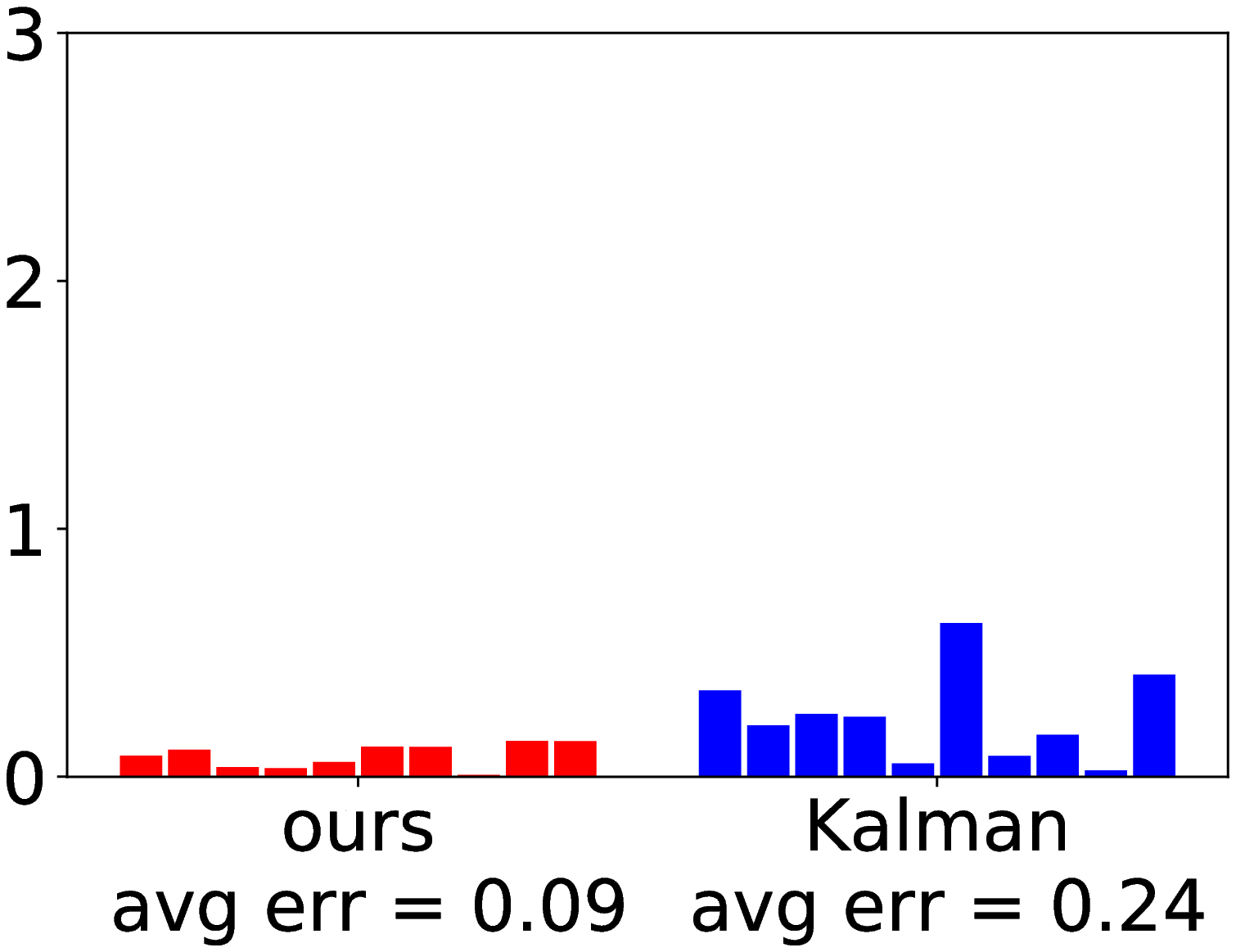}&
    \includegraphics*[width=0.23\linewidth]{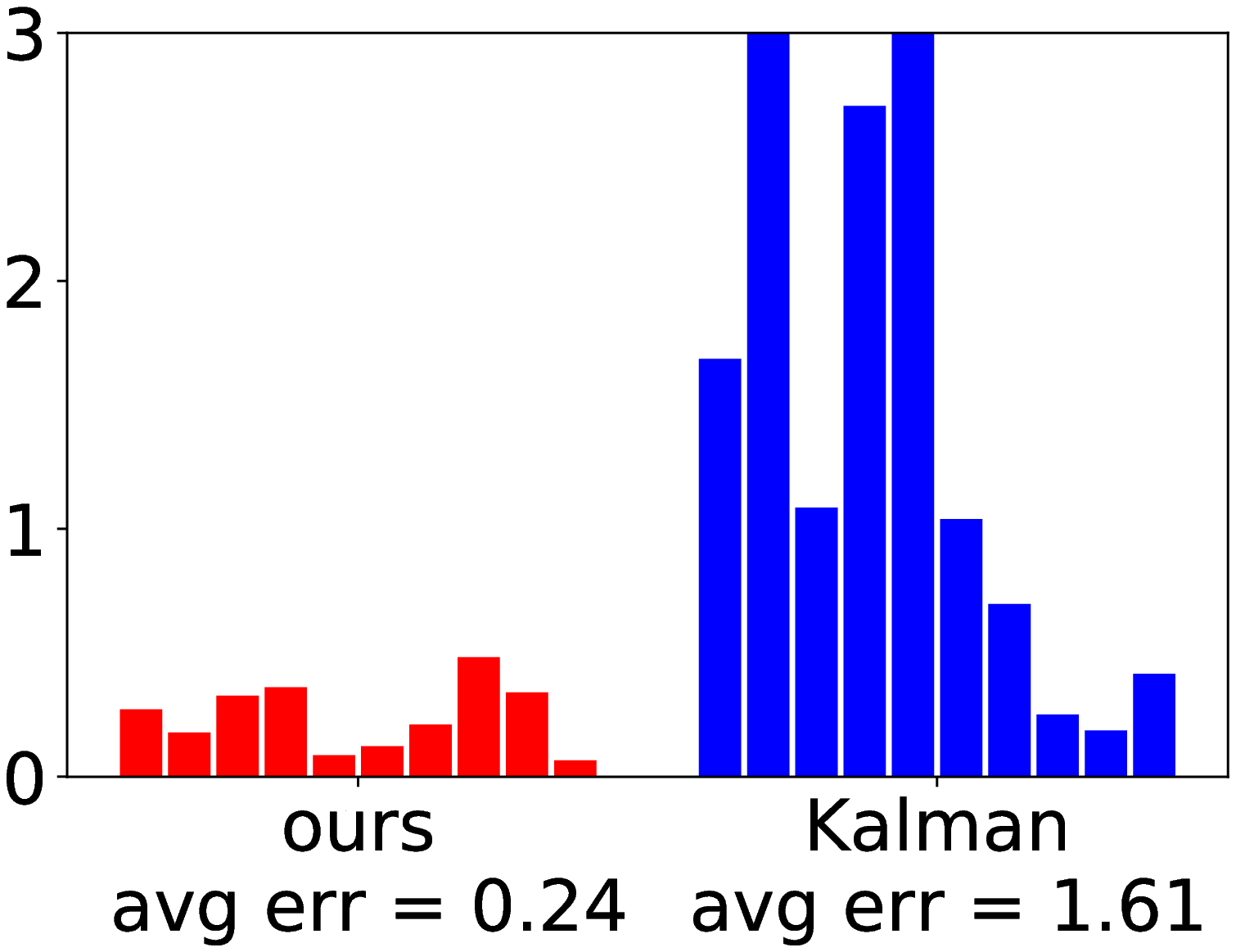}&
    \includegraphics*[width=0.23\linewidth]{kalman/predatorobjparam110000std01.eps}&
    \includegraphics*[width=0.23\linewidth]{kalman/predatorobjparam110000std15.eps}\\
	  \rotatebox{90}{\hspace{4ex}  \small{$\theta_2$ error} } &
    \includegraphics*[width=0.23\linewidth]{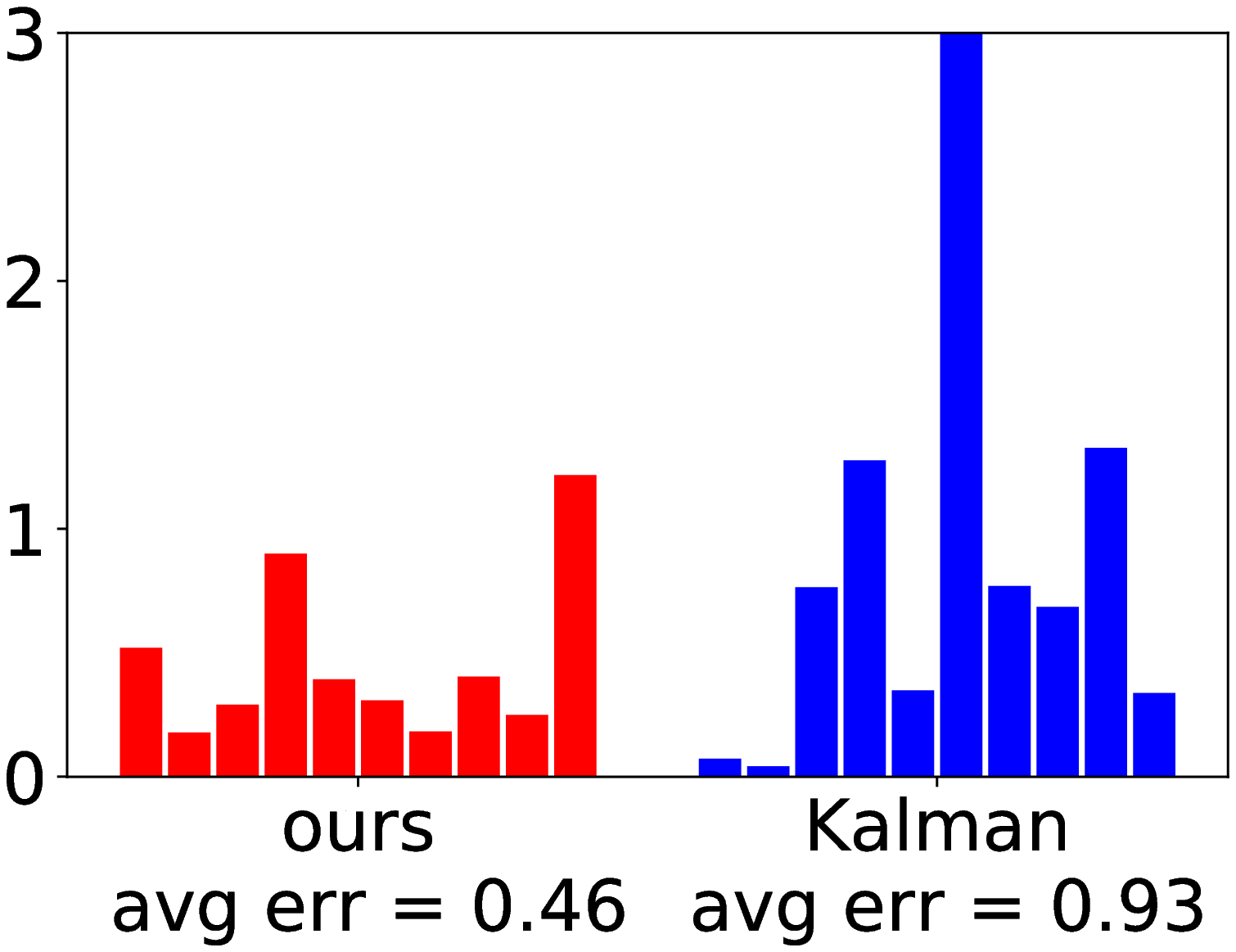}&
    \includegraphics*[width=0.23\linewidth]{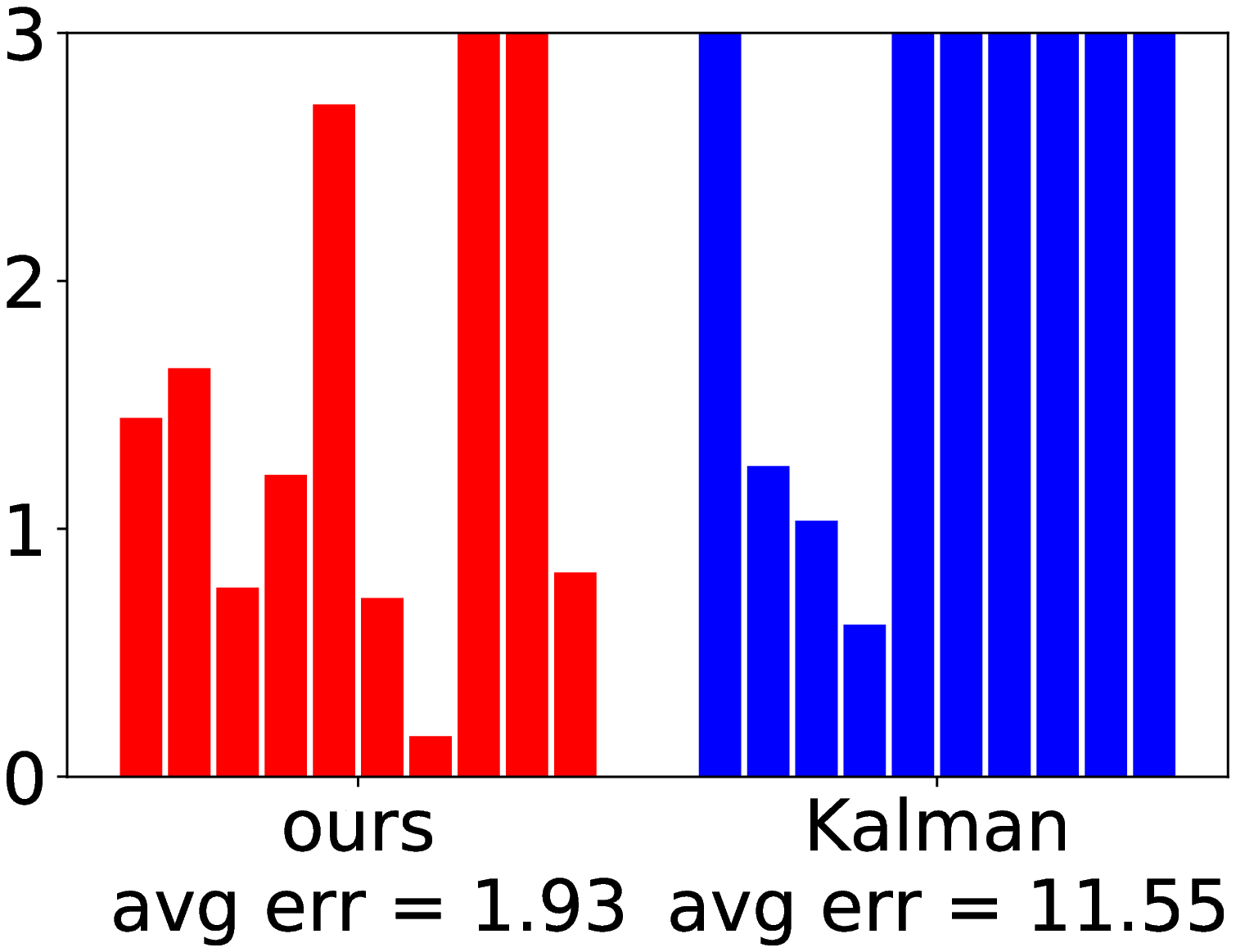}&
    \includegraphics*[width=0.23\linewidth]{kalman/predatorobjparam210000std01.eps}&
    \includegraphics*[width=0.23\linewidth]{kalman/predatorobjparam210000std15.eps}\\	 
    \rotatebox{90}{\hspace{4ex}  \small{$\theta_3$ error} } &  
	 \includegraphics*[width=0.23\linewidth]{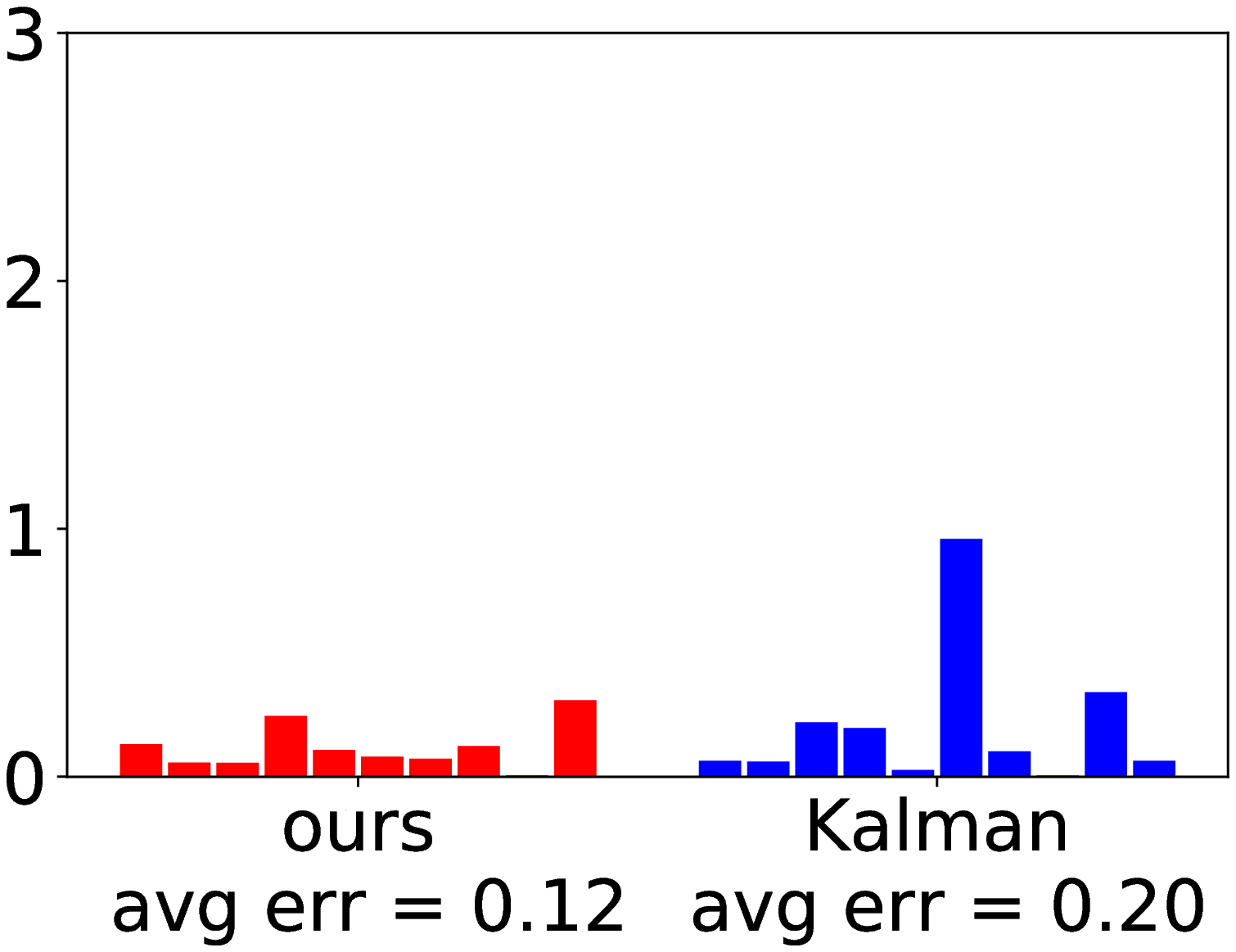}&
    \includegraphics*[width=0.23\linewidth]{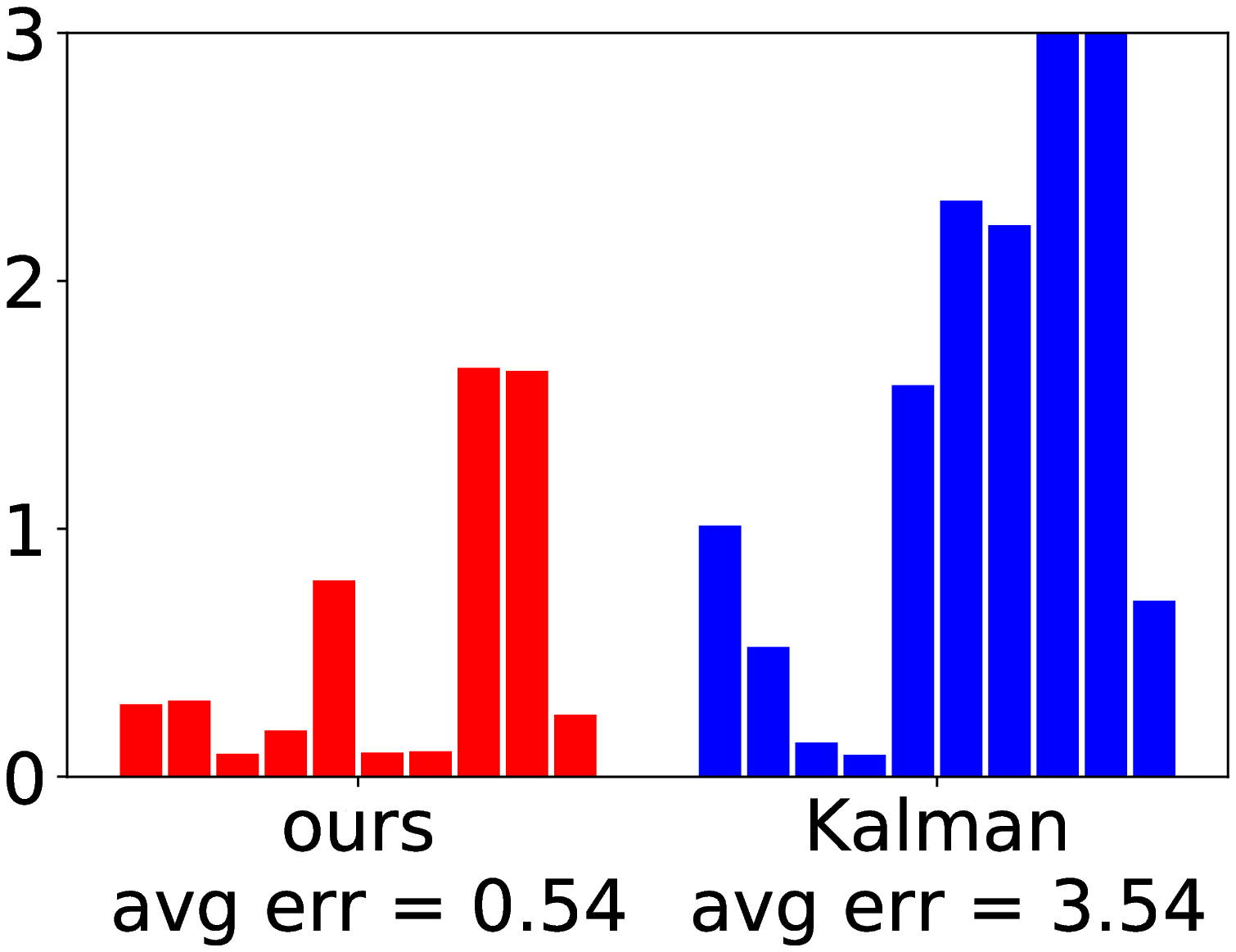}&
    \includegraphics*[width=0.23\linewidth]{kalman/predatorobjparam310000std01.eps}&
    \includegraphics*[width=0.23\linewidth]{kalman/predatorobjparam310000std15.eps}
  \end{tabular}
  \caption{Comparison with  EKF. Similar to the second and third rows of Fig.\ 5 of the paper, but includes both $T=20$ and $T=10\,000$ observations.}

  \label{f:kalman}
\end{figure*}


\clearpage

{\small
\bibliographystyle{icml2019}
\bibliography{lib1}
}

\end{document}